\documentclass[11pt,a4paper]{article}
\usepackage[T1]{fontenc}
\usepackage[htt]{hyphenat}
\usepackage{newtxtext}   
\usepackage{breakcites}
\usepackage{authblk}
\usepackage[top=3cm, bottom=3cm, left=3cm, right=3cm]{geometry}

\usepackage{algorithm}
\usepackage{algpseudocode}
\usepackage{graphicx}
\usepackage{dsfont}
\usepackage{setspace}
\usepackage{enumitem}
\usepackage[dvipsnames]{xcolor}
\usepackage{graphicx}
\usepackage{epstopdf}
\usepackage{csquotes}
\usepackage{mathtools}
\usepackage{amssymb, amsmath,amsthm, amsfonts}
\usepackage{ifthen}

\usepackage{caption}

\usepackage{verbatim}
\usepackage{pifont}
\usepackage{bm}  % for maths in bold for vectors.
\usepackage{stackengine}
\usepackage{bbm}

\allowdisplaybreaks
\usepackage{accents}

\usepackage{xpatch}
\makeatletter
\xpatchcmd{\@thm}{\thm@headpunct{.}}{\thm@headpunct{}}{}{}
\makeatother

\newtheorem{definition}{Definition}

\newtheorem{theorem}{Theorem}
\newtheorem{lemma}{Lemma}
\newtheorem{remark}{Remark}
\newtheorem{corollary}{Corollary}
\newtheorem{proposition}{Proposition}

\usepackage{graphicx}
\usepackage{subcaption}
\usepackage{enumitem}
\usepackage{hyperref}
\usepackage{url}
\usepackage{wrapfig,framed}
\usepackage{setspace}
\usepackage{graphicx}
\usepackage{epstopdf}
\usepackage{csquotes}
\usepackage{mathtools}
\usepackage{amssymb, amsmath,amsthm, amsfonts}
\usepackage{ifthen}
\usepackage{tikz}
\usepackage[round]{natbib}
\usepackage{verbatim}
\usepackage{pifont}
\usepackage{bm}  % for maths in bold for vectors.
\usepackage{stackengine}
\usepackage{bbm}

\allowdisplaybreaks
\usepackage{accents}

\usepackage{xpatch}
\makeatletter
\xpatchcmd{\@thm}{\thm@headpunct{.}}{\thm@headpunct{}}{}{}
\makeatother

\newcommand{\cF}{\mathcal{F}}
\newcommand{\cG}{\mathcal{G}}
\newcommand{\cH}{\mathcal{H}}

\newcommand{\cL}{\mathcal{L}}
\newcommand{\cM}{\mathcal{M}}
\newcommand{\mM}{\mathfrak{M}}

\newcommand{\cP}{\mathcal{P}}

\newcommand{\cX}{\mathcal{X}}
\newcommand{\cY}{\mathcal{Y}}
\newcommand{\cZ}{\mathcal{Z}}

\newcommand{\EE}{\mathbb{E}}

\newcommand{\NN}{\mathbb{N}}

\newcommand{\PP}{\mathbb{P}}

\newcommand{\RR}{\mathbb{R}}

\newcommand{\sD}{\mathsf{D}}
\newcommand*{\dd}{\, \mathrm{d}}
\newcommand{\argmax}{\mathop{\mathrm{argmax}}}
\newcommand\numberthis{\addtocounter{equation}{1}\tag{\theequation}}

\newcommand{\GMMD}{\mathsf{GMMD}}

\newcommand{\GW}{\mathsf{GW}}
\newcommand{\UGW}{\mathsf{UGW}}
\newcommand{\MMD}{\mathsf{MMD}}

\definecolor{darkblue}{rgb}{0.0,0.0,0.66}  
\definecolor{darkred}{rgb}{100,0.0,0.0} 
\hypersetup{colorlinks=true,linkcolor=darkred,citecolor=darkblue}

\newcommand{\zg}[1]{{\color{blue!30!purple}[ZG: #1]}}
\newcommand{\youssef}[1]{{\color{red}[YM: #1]}}

\makeatletter
\renewcommand\AB@affilsepx{, \protect\Affilfont}
\makeatother

\title{Cycle Consistent Probability Divergences \\Across Different Spaces}

\author[a]{\bfseries Zhengxin Zhang}
\author[b]{\bfseries Youssef Mroueh}
\author[a]{\bfseries Ziv Goldfeld}
\author[c]{\bfseries Bharath K. Sriperumbudur}
\affil[a]{Cornell University}
\affil[b]{IBM Research AI}
\affil[c]{Pennsylvania State University}

\begin{document}

\maketitle
\begin{abstract}
\iffalse
Discrepancy measures between probability distributions are at the core of statistical inference and machine learning. In many applications, distributions of interest are supported on different spaces, and yet a meaningful correspondence between data points is desired. We propose in this paper a novel unbalanced Monge formulation for matching up to isometries, distributions supported on different spaces. Our formulation follows from a principled relaxation of the Gromov-Haussdroff distance between metric spaces, and relies on two cycle-consistent pushforward maps between the two spaces. Our discrepancy explains the success of cycle-consistent GAN and we show that the latter is a particular case of our framework. We further kernelize this divergence and restrict the mappings to parametric function classes. We coin the resulting kernelized divergence \emph{the Generalized Maximum Mean Discrepancy} ($\mathsf{GMMD}$). We study the statistical proprieties of  $\mathsf{GMMD}$, and present numerical experiments to validate our theory. 
\fi

Discrepancy measures between probability distributions are at the core of statistical inference and machine learning. In many applications, distributions of interest are supported on different spaces, and yet a meaningful correspondence between data points is desired. Motivated to explicitly encode consistent bidirectional maps into the discrepancy measure, this work proposes a novel unbalanced Monge optimal transport formulation for matching, up to isometries, distributions on different spaces. Our formulation arises as a principled relaxation of the Gromov-Haussdroff distance between metric spaces, and employs two cycle-consistent maps that push forward each distribution onto the other. We study structural properties of the proposed discrepancy and, in particular, show that it captures the popular cycle-consistent generative adversarial network (GAN) framework as a special case, thereby providing the theory to explain it. Motivated by computational efficiency, we then kernelize the discrepancy and restrict the mappings to parametric function classes. The resulting kernelized version is coined the \emph{generalized maximum mean discrepancy} (GMMD). Convergence rates for empirical estimation of GMMD are studied and experiments to support our theory are provided.
\end{abstract}

\section{Introduction}\label{SEC:intro}

Discrepancy measures between probability distributions are ubiquitous in machine learning. In practice, distributions of interests are often supported on different spaces and the goal is not only to quantify discrepancy, but also to obtain a meaningful and consistent correspondence between data points. %In many applications, the distributions of interest are supported on different spaces, and a correspondence between data points is desired.
Such problems arise, e.g., in natural language
processing for unsupervised matching across different languages or ontologies \citep{alvarez2018gromov,grave2019unsupervised,pmlr-v108-alvarez-melis20a}, shape matching \citep{Bronstein1168,memoli2011gromov,xu2019scalable}, 
%quantum chemistry  \cite{gilmer2017neural}, 
heterogenous domain adaptation \citep{yan2018semi}, %deep metric alignment \cite{ezuz2017gwcnn}, graph classification \cite{titouan2019optimal}, 
generative modeling \citep{bunne2019learning}, and many~more. %\zg{this list seems too long, can omit a couple of less frequent applications} 

%Popular discrepancies between distributions on incompatible spaces include the Gromov-Wasserstein distance (GW) \citep{memoli2011gromov} and its sliced \citep{vayer2019sliced} and unbalanced \citep{sejourne2020unbalanced} variants. 

Among the most popular discrepancies between distributions on incompatible spaces is the Gromov-Wasserstein distance (GW) \citep{memoli2011gromov} (see \citet{sejourne2020unbalanced} for an unbalanced variant). Computationally, the GW distance amounts to a quadratic assignment problem %(computing a coupling matrix in the discrete case) 
that is NP hard \citep{Commander2005}. %For discrete distributions this corresponds to computing a coupling matrix that capture only the correspondence between the samples at hand 
%instance, when distributions are discrete, the GW distance is computed  distributions %In practice, these distances are computed on discrete distributions and capture only the correspondence between the samples at hand via a coupling matrix.
To alleviate this impasse, \citet{pmlr-v48-peyre16} proposed an entropic regularization to the GW problem, and derived an algorithm with cubic $O(n^3)$ complexity in the number of samples. More recently, \citet{vayer2019sliced} proposed slicing the GW distance, which further reduces the computational complexity to $O(n\log n)$. Despite these algorithmic advances, a common issue with GW-based discrepancies is their lack of generalization to new data points. These approaches only quantify the distance without generating a map that captures the correspondence. This requires recomputing the distance whenever one wants to account for new data points, thereby incurring an additional cost. %one cannot leverage distance computation on the training set to infer its value on test samples, which necessitates recomputing on the coupling on the whole (train and test) data the coupling and incur an additional cost. 
This shortcoming %along with the demanding computational complexity of the classic GW distance 
motivate us to explore computationally friendly discrepancies between distributions on different spaces that explicitly encode consistent, bidirectional measure preserving mappings that capture the correspondence. 
%Given new samples from the distributions, one would need to recompute the coupling and incur an additional cost. This lack of generalization to new points and the challenging computational complexity of GW, motivated us to ask the question whether we could formulate the GW matching problem via continuous `Monge' transport mappings instead of couplings. 
This is similar in spirit to the recent interest in learning Monge optimal transport (OT) maps \citep{perrot2016mapping,makkuva2020optimal,paty2020regularity,flamary2019concentration}.\\

% For instance \citep{pmlr-v48-peyre16} proposed an entropic regularization to the GW problem, and an algorithm with cubic scaling ($O(N^3)$) in the number of samples. More recently, \citep{vayer2019sliced} proposed slicing for GW and reduced the computational complexity to $O(N\log N)$. 

% Given new samples from the distributions, one would need to recompute the coupling and incur an additional cost. This lack of generalization to new points and the challenging computational complexity of GW, motivated us to ask the question whether we could formulate the GW matching problem via continuous `Monge' transport mappings instead of couplings. This is similar in spirit to the recent interest in optimal transport in  learning Monge maps corresponding to the Wasserstein distance \citep{perrot2016mapping,makkuva2020optimal,paty2020regularity,flamary2019concentration}.
%formulate the GW matching problem via continuous `Monge' transport mappings instead of couplings. This is similar in spirit to the recent interest in optimal transport in  learning Monge maps corresponding to the Wasserstein distance \citep{perrot2016mapping,makkuva2020optimal,paty2020regularity,flamary2019concentration}.

Specifically, we propose a novel unbalanced divergence between probability measures supported on different spaces that explicitly employs two cycle-consistent maps that (approximately) push each distribution onto another. Cycle-consistency here is in the context of cycle generative adversarial networks \citep{zhu2017unpaired,kim2017learning} (GANs), which requires that the two pushforward maps are roughly inverses one of another. Note that \citet{memoli2021distance} recently proposed a quasi-metric called the Gromov-Monge distance that employs a single push forward map. A key advantage of our approach is its cycle consistency that provides a mathematical framework for the popular cycle GAN and enables a principled study thereof. %explains the empirical success of cycle GAN and roots it in optimal transport \zg{I'm not sure about the 2nd half of this sentence}.  

The main contributions of this paper are:
\begin{itemize}[leftmargin=*,align=parleft]
    \item We introduce and study in Section \ref{sec:UGM} the unbalanced bidirectional Gromov-Monge (UBGM) divergence, drawing connections between UBGM and the popular cycle GAN framework. 
    \item Motivated by computational efficiency, we kernelize UBGM in Section \ref{SEC:GMMD} and restrict mappings to parametric function classes, such as neural networks (NNs). We call the resulting divergence the generalized maximum mean discrepancy (GMMD). We then derive convergence rates for two-sample empirical estimation of GMMD.% \textcolor{red}{(B: like in the abstract, should be rephrase empirical convergence rates?)}.   
    \item We present numerical results in Section \ref{sec:exp} that support our theory and demonstrate the computational efficiency, and generalization capability of the proposed framework for matching across different spaces.
\end{itemize}

\section{Background and Preliminaries}\label{SUBSEC:notation}

%\zg{Notation and setting up Gromov-Haussdroff}

\subsection{Notations}

Let $(\cX,d_\cX)$ be a compact metric space. The diameter of a set $A\subseteq\cX$ is $\mathsf{diam}(A):=\sup_{x,x'\in A} d_\cX(x,x')$. We use $B(x,r)$ to denote the open ball of radius $r>0$ centered at $x\in\cX$. For $\epsilon>0$, a set $\cX_\epsilon$ is called an $\epsilon$-cover of $\cX$ if for any $x\in\cX$, $\inf_{x'\in\cX_\epsilon} d_\cX(x,x')<\epsilon$. The $\epsilon$-covering number of $\cX$ is $N(\cX,d_\cX,\epsilon):= \inf\{|\cX_\epsilon|:\,\cX_\epsilon\mbox{ is an }\epsilon \mbox{-cover of }\cX\}$. When $\cX$ is a subset of $\RR^d$, we always use the metric induce by the Euclidean norm, denoted as $\|\cdot\|$. For a metric space $(\cX,d_\cX)$, the diameter of $\cX$ is defined as $\sup_{x,x'\in\cX}d_\cX(x,x')$.

%The set of bounded continuous functions on $\cX$ is denoted by $C_b(\cX)$ and is equipped with the infinity norm $\|\cdot\|_\infty$ \zg{do we use $C_b$ notation anywhere?? if we only need the sup-norm, it can defined for general $L^p$ spaces (up next)}. 

For $1\leq p \leq \infty$, the $L^p$ space over $\cX$ is denoted by $L^p(\cX)$ with $\| \cdot \|_{p}$ designating the norm. The Lipschitz constant of a function $f:\cX\to\cY$ is $\|f\|_{\mathsf{Lip}}=\sup_{x,x'\in\cX}\frac{d_\cY(f(x),f(x'))}{d_\cX(x,x')}$, with $\mathsf{Lip}_L(\cX,\cY)=\{f:\|f\|_{\mathsf{Lip}}\leq L\}$ denoting the Lipschitz ball of radius $L>0$. A mapping $f:\cX\to\cY$ between metric spaces is called an isometry if $d_\cX(x,x') = d_\cY\big(f(x),f(x')\big)$, for all $x,x'\in\cX$, i.e., $f$ preserves the metric structure. For a class of mappings from $\cX$ to $\cY$, define the sup-metric on $\cF$ as $d_\cF\big(f_1,f_2):=\sup_{x\in \cX} d_\cY\big(f_1(x),f_2(x)\big)$.

%\zg{we refer to $L^p$ spaces/norm later on; should we have the notation?}. The $\alpha$-H\"older norm of $f:\cX\rightarrow\cY$ is $\|f\|_\alpha:=\sup_{x,x'\in\cX} \frac{d_\cY(f(x),f(x'))}{d_\cX(x,x')^\alpha}$ \zg{change Holder norm notation? Sometimes $\|\cdot\|_{C^\alpha}$ is used}. In particular, for $L>0$, we define the Lipschitz ball $\mathsf{Lip}_L(\cX,\cY)=\{f:\|f\|_1\leq L\}$ \zg{adapt according to notation of Holder norm}. 

The probability space on which all random variables are defined is denoted by  $(\Omega,\mathcal{A},\PP)$ (assumed to be sufficiently rich), with $\EE$ designating the corresponding expectation. The class of Borel probability measures over $\cX$ is denoted by $\cP(\cX)$. For $n \in \mathbb{N}$, $P^{\otimes n}$ denotes the $n$-fold product measure of $P$. Given a measurable $f:\cX\to\cY$ and $P\in\cP(\cX)$, the pushforward of $P$ through $f$ is $f_\sharp  P(B) := P(f^{-1}(B))$, for any Borel set $B$. Clearly $f_\sharp P\in\cP(\cY)$. A sequence of probability measures $P_n\in\cP(\cX)$ converges to a probability measure $P$, denoted as $P_n\stackrel{w}{\to}P$, if for any bounded continuous function $\phi(x)$ on $\cX$, $\lim_n \int\phi(x)\dd P_n(x) = \int \phi(x)\dd P(x)$ We also write $a\lesssim_x b$ when $a\leq C_x b$, where $C_x$ is a constant depending only on $x$, and write $a\lesssim b$ if the omitted constant is universal. Also denote $n\wedge m=\min\{n,m\}$.

\subsection{Gromov-Haussdroff Distance}% Between Metric Spaces}

To motivate our proposed discrepancy measure, we start by recalling the Gromov-Hausdorff distance between metric spaces, which is defined as
%As a starting point of our proposed approach in next section, we give a preliminary introduction to the Gromov-Hausdorff distance. Throughout we adopt the notations from \cite{memoli2011gromov}. 
%The Gromov-Haussdroff distance between $(\cX,d_\cX)$ and $(\cY,d_\cY)$ in then defined as %metric spaces $d_{\rm{GH}}$ is defined as follows:
\begin{equation}
d_{\sf{GH}}(\cX,\cY):=\inf_{\cZ,\phi_{\cX},\phi_{\cY}}d^{\cZ}_{\sf{H}}\big(\phi_{\cX}(\cX), \phi_{\cY}(\cY)\big),%\label{eq:GH}
\end{equation}
where the infimum is on an ambient metric space $(\cZ,d_{\cZ})$ and isometric embeddings $\phi_{\cX}:\cX \to \cZ$ and  $\phi_{\cY}:\cY \to \cZ$, with $d^{\cZ}_{\sf{H}}$ as the Hausdroff distance on~$\cZ$.

Evidently, the formulation above is not computable. Nevertheless, $d_{\sf{GH}}$ has several equivalent~forms \citep{memoli2011gromov} that, with appropriate relaxations, can be computed efficiently. Two such forms are as follows.

% that admit powerful relaxations and that are computationally friendly. We will present in what follows two of these reformulations. 
%From this definition of $d_{\rm{GH}}$ in Equation \eqref{eq:GH}, we see that this formulation is not computable. Interestingly, $d_{\rm{GH}}$ admits many reformulations \cite{memoli2011gromov}  that admit powerful relaxations and that are computationally friendly. We will present in what follows two of these reformulations. 

\paragraph{Correspondence set reformulation.} Following \citet{memoli2021distance}, a correspondence set between $\cX$ and $\cY$ is a set $R\subset \cX \times \cY$ whose projections to $\cX$ and $\cY$ define surjections on $R$. The set of all such correspondences is denoted by $\mathcal{R}(\cX,\cY)$. The Gromov-Hausdroff distance can be reformulated as
\begin{equation}
    d_{\sf{GH}}(\cX,\cY)= \frac{1}{2}\inf_{R \in \mathcal{R}(\cX,\cY)} \sup_{(x,y),(x',y')\in R} \mspace{-8mu}\Gamma_{\cX,\cY}(x,y,x',y'),\label{eq:dGHCorrespondence}
\end{equation}
where $\Gamma_{\cX,\cY}(x,y,x',y'):= \big|d_\cX(x,x')- d_\cY(y,y')\big|$ is the pointwise distortion between $(x,x')\in\cX^2$ and $(y,y')\in\cY^2$. Note that \eqref{eq:dGHCorrespondence} can be written in the following compact form as an $L^{\infty}$ norm:
\begin{equation}
    d_{\sf{GH}}(\cX,\cY)= \frac{1}{2}\inf_{R \in \mathcal{R}(\cX,\cY)}\left\|\Gamma_{\cX,\cY} \right\|_{L^{\infty}(R\times R)}.\label{eq:dGHC}
\end{equation}

\paragraph{Two mappings reformulation.} Another important reformulation of the Gromov-Hausdroff distance in terms of mappings between spaces is:
  \begin{equation}
     d_{\sf{GH}}\mspace{-1mu}(\cX,\cY)\mspace{-3mu}=\mspace{-3mu}\frac{1}{2}\mspace{-2mu} \inf_{\substack{f: \cX\to \cY \\g :\cY \to \cX} }\mspace{-5mu}\max\mspace{-3mu} \big\{\mspace{-2mu}\Delta^{\infty}_{\cX}\mspace{-1mu}(f),\Delta^{\infty}_{\cY}\mspace{-1mu}(g), \Delta^{\infty}_{\cX,\cY}(f,g)\mspace{-2mu}\big\},
     \label{eq:dGH2}
 \end{equation}
where the distortions $\mspace{-1.5mu}\Delta^{\infty}_{\cX}$,$\mspace{2.5mu}\Delta^{\infty}_{\cY}$, and $\Delta^{\infty}_{\cX,\cY}$ are given~by\footnote{The superscript $\infty$ indicates that $\Delta^{\infty}_{\cX}$, $\Delta^{\infty}_{\cY}$, and~$\Delta^{\infty}_{\cX,\cY}$ can be written as $L^{\infty}$ norms of the appropriate pointwise distortions (e.g., $\Gamma_{\cX,\cY}\big(x,f(x),x',f(x')\big)$ for $\Delta^{\infty}_{\cX}$).}
 \begin{align*}
 \Delta^{\infty}_{\cX}(f)&:= \sup_{x,x' \in \cX} \big|d_{\cX}(x,x')-d_{\cY}\big(f(x),f(x')\big)\big| \\
% &= \left|\left| \Gamma_{\cX,\cY}(x,f(x),x',f(x'))\right|\right|_{L^{\infty}(\cX\times \cX)}. 
 \Delta^{\infty}_{\cY}(g)&:= \sup_{y,y' \in \cY} \big|d_{\cX}\big(g(y),g(y')\big)-d_{\cY}(y,y')\big|\\
% &= \left|\left|\Gamma_{\cX,\cY}(g(y), y,g(y'),y') \right|\right|_{L^{\infty}(\cY\times \cY)}
 \Delta^{\infty}_{\cX,\cY}(f,g)&:= \sup_{x \in \cX,y \in \cY} \big|d_{\cX}\big(x,g(y)\big)-d_{\cY}\big(f(x),y\big)\big|.
 %&=\left|\left|\Gamma_{\cX,\cY}(x, g(y),f(x),y)\right|\right|_{L^{\infty}(\cX\times \cY) }
\end{align*}
Formulation \eqref{eq:dGH2} thus measures distance by searching for low distortion maps between the two metric spaces, such that the so-called \emph{cycle consistency} property holds, i.e., the maps are approximate inverses of one another. More specifically following \citet{memoli2005theoretical}, if $d_{\sf{GH}}(\cX,\cY)\leq \epsilon$ then there exists $f:\cX\to\cY$ and $g:\cY\to\cX$ such that: (1) the induced metric distortions are small, i.e., $\Delta^{\infty}_{\cX}(f)\leq 2\epsilon$ and $\Delta^{\infty}_{\cY}(g)\leq 2\epsilon$; and (2) these functions are \emph{almost} inverses one of another, in the sense that $d_{\cX}\big(x,g(f(x))\big)\leq 2\epsilon$ and $d_{\cY}\big(f(g(y)),y\big)\leq 2\epsilon$ (which follows from $\big|d_{\cX}(x,g(y))-d_{\cY}(f(x),y)\big|\leq 2\epsilon$ by taking $y=f(x)$ and $x=g(y)$, respectively). Thus, $\Delta^{\infty}_{\cX,\cY}$ ensures cycle consistency of the maps.

% is to remind the reader that the quantities $\Delta^{\infty}_{\cX},\Delta^{\infty}_{\cY}, \Delta^{\infty}_{\cX,\cY} $ can be seen as $L^{\infty}$ norms of $\Gamma_{\cX,\cY}(x,f(x),x',f(x'))$, $\Gamma_{\cX,\cY}(g(y), y,g(y'),y')$ and $\Gamma_{\cX,\cY}(x, g(y),f(x),y)$ in $L^{\infty}(\cX\times \cX)$, $L^{\infty}(\cY\times \cY)$, and $L^{\infty}(\cX \times \cY)$ respectively.
%Formulation \eqref{eq:dGH2} is intuitive: $\Delta_{\cX}^{\infty}$ and $\Delta_{\cY}^{\infty}$ measure whether  low metric distortion maps $f: \cX\to \cY$ and $g:\cY\to \cX$ exist between the two metric spaces, and $\Delta^{\infty}_{\cX,\cY}$ ensures that these two mappings satisfy a  \emph{cycle consistency} propriety.

%Following \cite{memoli2005theoretical}, we make this intution more formal. If $d_{GH}(\cX,\cY)\leq \varepsilon$ then we have: 1) $f$ and $g$ are  low metric distortion maps from $\cX$ to $\cY$ and vice versa respectively (since  $\Delta^{\infty}_{\cX}\leq 2\varepsilon$ and $\Delta^{\infty}_{\cY}\leq 2\varepsilon$) and  2) $f$ and $g$ are \emph{almost} inverses one of another: note that $\Delta^{\infty}_{\cX,\cY} \leq 2\varepsilon$ implies that $|d_{\cX}(x,g(y))-d_{\cY}(f(x),y)|\leq 2\varepsilon$; For $y=f(x)$ we obtain $d_{\cX}(x,g(f(x)))\leq 2\varepsilon$ and for $x=g(y)$, $d_{\cY}(f(g(y)),y)\leq 2\varepsilon.$  It is now clear that $\Delta^{\infty}_{\cX,\cY}$ ensures the so-called cycle consistency between the two maps (i.e almost inverse propriety). 
 
 %explain the meaning of each distortion what it means ...

%\section{Distance Construction from   Metric Spaces to  Metric Measure Spaces }

\section{Probability Divergences Across Metric Measure Spaces}\label{sec:UGM}

We are now ready to present the proposed discrepancy between probability measures on different spaces. In conjunction with the preceding discussion, such a discrepancy can equivalently be viewed as a distance between metric measure (mm) spaces \citep{memoli2011gromov}. 

%Paralleling the definition of distances on metric spaces we start by defining metric measure spaces, a central mathematical object for the theory \cite{memoli2011gromov}:  

\begin{definition}[Metric measure spaces]
A metric measure space is a triple $(\cX,d_{\cX},P)$ where $(\cX,d_{\cX})$ is a compact metric space and $P\in\cP(\cX)$ has full support. %whose support is exactly $\cX$. Two metric measure spaces $(\cX,P)$ and $(\cY,Q)$ are called equivalent if and only if there is an invertible isometry $f:\cX\rightarrow\cY$ such that $f_\sharp P=Q$. We also define the set $\mM$ that consists of the all equivalence classes of pairs $(\cX,P)$. For simplicity we will write $P$ for $(\cX,P)$ whenever there is no ambiguity.
\end{definition} 
%In what follows, we refer to metric measure spaces as mm spaces for short as in \cite{memoli2011gromov}.

\subsection{The Unbalanced Bidirectional Gromov-Monge Divergence}%Proposed Approach: Relaxation of \eqref{eq:dGH2}} 

Formulation \eqref{eq:dGH2} of $d_{\sf{GH}}$ is computationally appealing and has lead to many algorithmic relaxations that approximate the Gromov-Hausdroff distance \citep{memoli2004comparing,memoli2005theoretical,Bronstein2006}. As a stepping stone towards our unbalanced formulation, we first adapt this formulation to an extended metric between two mm spaces $(\cX,d_{\cX},P)$ and $(\cY,d_{\cY},Q)$. To that end, we first relax the distortion terms and then restrict the mappings to be measure persevering, as described next.

Let $(X,X')\sim P^{\otimes 2}$ be independent of $(Y,Y')\sim Q^{\otimes 2}$. For $1\leq p<\infty$, set
\begin{align*}\mspace{-3mu}
&\Delta^{(p)}_{\cX}(f;P)\mspace{-3mu}:=\mspace{-3mu}\Big(\mspace{-1mu}\EE\mspace{-2mu}\Big[\big|d_{\cX}(X,X')-d_{\cY}\big(f(X),f(X')\big)\big|^p\Big]\Big)^{\mspace{-3mu}\frac{1}{p}}\\
% &= \left|\left| \Gamma_{\cX,\cY}(x,f(x),x',f(x'))\right|\right|_{L^{\infty}(\cX\times \cX)}. 
 &\Delta^{(p)}_{\cY}(g;Q)\mspace{-3mu}:=\mspace{-3mu}\Big(\mspace{-1mu}\EE\mspace{-2mu}\Big[\big|d_{\cX}\big(g(Y),g(Y')\big)-d_{\cY}(Y,Y')\big|^p\Big]\Big)^{\mspace{-3mu}\frac{1}{p}}\\
%  \underset{ \cY\times \cY}{\int} |d_{\cX}(g(y),g(y')))-d_{\cY}(y,y')|\dd Q^{\otimes2}(y,y'),\\
% &= \left|\left|\Gamma_{\cX,\cY}(g(y), y,g(y'),y') \right|\right|_{L^{\infty}(\cY\times \cY)}
& \Delta^{(p)}_{\cX,\cY}(f\mspace{-2mu},\mspace{-2mu}g;\mspace{-2mu}P,\mspace{-2mu}Q)\mspace{-3mu}:=\mspace{-3mu}\Big(\mspace{-1mu}\EE\mspace{-2mu}\Big[\big|d_{\cX}\big(\mspace{-1.5mu}X\mspace{-1mu},\mspace{-2mu}g(Y)\big)\mspace{-3mu}-\mspace{-3mu}d_{\cY}\big(\mspace{-1.5mu}f(X),\mspace{-3mu}Y\mspace{-1.5mu}\big)\big|^p\mspace{-1.5mu}\Big]\mspace{-1.5mu}\Big)^{\mspace{-3mu}\frac{1}{p}}\mspace{-6mu},
% =\left|\left|\Gamma_{\cX,\cY}(x, g(y),f(x),y)\right|\right|_{L^{\infty}(\cX\times \cY) }
%\Delta^{1}_{\cX,\cY}(f,g;P,Q)=\underset{x\sim P, y \sim Q}{\mathbb{E}} |d_{\cX}(x,g(y))-d_{\cY}(f(x),y)|. 
 \end{align*}
as the $L^p$ relaxation of the distortion terms from \eqref{eq:dGH2}. Restricting the mappings in \eqref{eq:dGH2} to be `Monge' measure preserving, i.e., $f_{\sharp }P=Q$ and $g_{\sharp }Q=P$, gives rise to $p$th order \emph{bidirectional Gromov-Monge} (BGM) distance, as defined next. 

\begin{definition}[Bidirectional Gromov-Monge distance]\label{def:metricD}
Fix $1\leq p<\infty$. The $p$th order BGM distance between $(\cX,d_{\cX},P)$ and $(\cY,d_{\cY},Q)$ is 
\[\mathsf{D}_p(P,Q):= \inf_{\substack{f: \cX \to \cY,\, f_{\sharp }P=Q\\ g:\cY\to \cX,\, g_{\sharp }Q=P  }} 
\Delta_p(f,g;P,Q) , \]
where $\Delta_p(f,g;P,Q):= \Delta^{(p)}_{\cX}(f;P)+ \Delta^{(p)}_{\cY}(g;Q)+ \Delta^{(p)}_{\cX,\cY}(f,g;P,Q)$. If the set of measure preserving maps is empty, then we set $\mathsf{D}_p(P,Q)= \infty$. 
\end{definition}

This definition follows a reasoning similar to Formulation \eqref{eq:dGH2} of $d_{\mathsf{GH}}$ above: We are looking for measure preserving maps $f,g$, that are low metric distortion (minimizing $\smash{\Delta^{(p)}_{\cX},\Delta^{(p)}_{\cY}}$) and satisfy a cycle consistency propriety (minimizing $\smash{\Delta^{(p)}_{\cX\times \cY}}$, i.e., almost inverses one of another).

The BGM distance defines an extended metric between equivalence classes of mm spaces.

\begin{definition}[Equivalence classes]\label{DEF:equivalence}
Two mm spaces $(\cX,d_{\cX},P)$ and $(\cY,d_{\cY},Q)$ are called equivalent if and only if (iff) there is an invertible isometry $f:\cX\rightarrow\cY$ such that $f_\sharp P=Q$ (and hence $f^{-1}_\sharp Q=P$). The set of all such equivalence classes is denoted by $\mM$.%We also define the set $\mM$ that consists of the all equivalence classes of triples $(\cX,d_{\cX},P)$.
\end{definition}

\begin{proposition}[BGM distance metrizes $\mM$]\label{pro:Dmetric}
For any $1\leq p<\infty$, $\mathsf{D}_p$ defines an extended metric on $\mM$.%on $\mM$.
\end{proposition}

\begin{remark}[Finiteness of BGM] Let $(\cX,d_{\cX},P)$ be a mm space, with $\cX$ uncountable and $P$ atomless. If $C_{\infty}$ is the subcategory of mm spaces isomorphic to $(\cX,d_{\cX},P)$ \citep{memoli2021distance}, then BGM is a finite metric on $C_{\infty}$. This setting is interesting for matching isomorphic mm spaces (of same dimension), in image, shape and text matching applications.     
%As shown in \citet{dudley2002real}, any two mm spaces with the same cardinality have bidirectional Monge maps if the underlying metric spaces are Polish. As a concrete example, $\mathsf{D}_p$ between two compact Euclidean mm spaces (potentially of different dimension) with same cardinality is finite.
%\zg{add; can you give a concrete example for Gaussians? What is a Monge map from 1D to 2D Gaussians, both isotropic??} \youssef{we should make this remark with respect to $C_{\infty}$ class saying, it defines a metric there for same dimension. It is not obvious for different dimensions what happens}\zg{the remark is generally fine, but a big vague. Zhangxin, do you `mm spaces where the supports are equicardinal compact subsets of the Euclidean space (possibly of different dimension)?'} \youssef{We need atomless we can not talk about cardinality}}
\end{remark}

The proof of Proposition \ref{pro:Dmetric} is given in Appendix \ref{app:Dmetric_proof}. Positivity, symmetry, and the triangle inequality all follow from elementary calculations. The main challenge is in showing that if $\mathsf{D}_p$ nullifies then the considered spaces are isometrically isomorphic. To that end, we use the following lemma, that may be of independent interest (see Appendix \ref{app:existence_proof} for the proof).

%The following is needed to prove that $D$ is a metric in  a sense that will be made precise:
\begin{lemma}[Existence of isometries]
Fix $P,Q\in\cP(\cX)$ and let $\cF$ and $\cG$ be arbitrary function classes such that $\inf_{f\in\cF,\,g\in\cG} \Delta_p(f,g;P,Q) =0$. Then there exist minimizing sequences $(f_n)_{n\in\NN}\subset\cF$ and $(g_n)_{n\in\NN}\subset\cG$ that converge almost surely (a.s.) to isometries $f$ and $g$, respectively, such that $f=g^{-1}$. %\textcolor{red}{(B: are we also ensuring these maps to be measure preserving?)}
\label{lem:existence}
\end{lemma}

\begin{remark}[Convergent sequences]
By definition of isometry, the limits $f,g$ are metric preserving, thus continuous and Lipschitz 1. We don't know if they still live in $\cF,\cG$ unless we impose compactness conditions later on. Notice that the convergence stated in the lemma is guaranteed solely by the first two terms of $\Delta_p(f,g;P,Q)$. If we drop the third term, convergent sequences still exists, but the limits are not guaranteed to be inverse of each other.
\end{remark}

While $\mathsf{D}_p$ is a valid extended metric on $\mathfrak{M}$, its evaluation is computationally challenging as it is unclear how to optimize over bidirectional Monge maps. %in several aspects. First, due to the stringent requirement of bidirectional measure preserving maps, finiteness of BGM can be guaranteed only in rather restricted settings (see \citep[Section 3.2]{memoli2021distance} for a related discussion). Even when it is finite, its evaluation is computationally challenging as it is unclear how to optimize over Monge maps. %Note that the metric $D$ in Definition \ref{def:metricD} is challenging computationally for two main reasons; the first stems from the optimization on Monge maps that we  discuss in Section \ref{SEC:GMMD}; the second goes to the point-wise measure preserving constraint. 
Following the unbalanced OT framework  \citep{chizat2018unbalanced,frogner2015learning}, we relax the measure preserving constraint using divergences\footnote{A divergence on $\cP(\cX)$ is a functional $\sD:\cP(\cX)\times\cP(\cX)\to \RR_{\geq 0}$ such that $\sD(P\|Q)=0$ iff $P=Q$.} $\sD_{\cX}$ and $\sD_{\cY}$ on measures on $\cX$ and $\cY$, respectively, and restrict the functions to pre-specified classes. This gives rise to the  unbalanced Gromov-Monge formulation.

\begin{definition}[Unbalanced bidirectional Gromov-Monge divergence]\label{def:UD}
Fix $1\leq p<\infty$, let $\sD_{\cX}$ and $\sD_{\cY}$ be divergences on $\cX$ and $\cY$, respectively. Further take $\cF$ as a class of mappings from $\cX$ to $\cY$, and $\cG$ a class of mappings from $\cY$ to $\cX$. The $p^{\text{th}}$ order UBGM divergence between $(\cX,d_{\cX},P)$ and $(\cY,d_{\cY},Q)$ is
\begin{align*}
\mathsf{UD}^{\cF,\cG}_p(P\|Q):=\inf_{\substack{f\in\cF\\ g\in\cG }} \Delta_p(f,g;P,Q) &+\lambda_x \sD_{\cX}(g_{\sharp }Q\|P)
+ \lambda_y \sD_{\cY}(f_{\sharp }P\|Q),
\end{align*}
where $\Delta_p(f,g;P,Q)$ is given in Definition \ref{def:metricD}, and $\lambda_x,\lambda_y>0$ are fixed regularization coefficients. %$\cF,\cG$ are mapping classes between $\cX$ and $\cY$. % = \Delta^1_{\cX}(f;P)+ \Delta^1_{\cY}(g;Q)+ \Delta^{(1)}_{\cX,\cY}(f,g;P,Q).\] $D_\cX$ and $D_\cY$ are two discrepancy measures on $\cX$ and $\cY$ respectively, i.e. $D_\cX(P,Q)=0\iff P=Q$, and the same for $D_\cY$.
\end{definition}

Evidently, $\mathsf{UD}^{\cF,\cG}_p$ no longer requires optimizing over Monge maps, which alleviates the computational difficulty associated with $\sD_p$ from Definition \ref{def:metricD}. The function classes $\cF$ and $\cG$ can also be chosen for computational convenience (e.g., NNs). The following proposition (see Appendix \ref{app:divergence_proof} for the proof) shows that $\mathsf{UD}^{\cF,\cG}_p$ is a continuous divergence.

\begin{proposition}[ $\mathsf{UD}^{\cF,\cG}_p$ is a divergence]\label{prop:divergence}
Suppose that $\sD_\cX$ and $\sD_\cY$ are weakly continuous in their arguments, and $\cF,\cG$ are rich enough so that they are dense around the isometric bijections if $\cX,\cY$ are isometric. Then
%\vspace{-4mm}
\begin{enumerate}[leftmargin=*,align=parleft]
    \item $\mathsf{UD}^{\cF,\cG}_p$ is an upper semi-continuous divergence on $\mathfrak{M}$, i.e., $\mathsf{UD}^{\cF,\cG}_p(P\|Q)\geq0$, for all $P\in\cP(\cX)$ and $Q\in\cP(\cY)$, with equality iff there exists isometries $f,g$ on the support of $P,Q$ respectively, such that $f_\sharp  P =Q$, $g_\sharp  Q=P$, and $f=g^{-1}$.
    \vspace{-2mm}
    \item If further $\cF,\cG$ are compact in the sup-metrics $d_\cF$ and $d_\cG$, then $\mathsf{UD}^{\cF,\cG}_p$ is continuous with respect to (w.r.t.) weak convergence.
\end{enumerate}  
\end{proposition}

\subsection{Cycle GAN as Unbalanced Gromov-Monge Divergence}\label{SUBSEC:cgan}

If $\sD_{\cX}$ and $\sD_{\cY}$ are integral probability metrics (IPM)\footnote{An IPM indexed by a function class $\cF$ is a pseudometric on $\cP(\cX)$ defined as $d_\cF(\mu,\nu):=\sup_{f\in\cF} \int_{\cX}f\dd(\mu-\nu)$.} \citep{zolotarev1984probability,muller1997integral} indexed by the function classes $\mathcal{F}_{\cX}$ and $\mathcal{F}_{\cY}$, respectively, then $\mathsf{UD}_1$ amounts to a minimax game between the maps $f,g$ and the witness functions $\psi$ and $\phi$ of the IPMs on $\mathcal{F}_{\cX}$ and $\mathcal{F}_{\cY}$ respectively, as follows:
\begin{align}
&\inf_{\substack{f: \cX \to \cY\\ g:\cY\to \cX }} \sup_{\substack{\psi \in \mathcal{F}_{\cX}\\\phi\in \mathcal{F}_{\cY}}} \Delta_1(f,g;P,Q)
+ \lambda_x \Big(\mathbb{E}\big[\psi\big(g(Y)\big)\big] -\mathbb{E}\big[\psi(X)\big]\Big)
+\lambda_y \Big(\mathbb{E}\big[\phi\big(f(X)\big)\big] -\mathbb{E}\big[\phi(Y)\big]\Big),
\label{eq:IPM_UD}
\end{align}
where $X\sim P$ and $Y\sim Q$.

Written in this form, we see the similarity to the cycle GAN formulation \citep{zhu2017unpaired,kim2017learning}:
\begin{align}
\inf_{\substack{f: \cX \to \cY, \\ g:\cY\to \cX }} &\sup_{\substack{\psi \in \mathcal{F}_{\cX},\\\phi \in \mathcal{F}_{\cY}}}\mspace{-4mu}\mathbb{E}\big[d_{\cX}\big(X, g\mspace{-2mu}\circ\mspace{-2mu} f(X)\big)\big]\mspace{-3mu}+\mspace{-3mu}\mathbb{E}\big[d_{\cY}\big(Y, f\mspace{-2mu}\circ\mspace{-2mu}g(Y)\big)\big]  
+\lambda_x \Big(\mathbb{E}\big[\psi\big(g(Y)\big)\big] -\mathbb{E}\big[\psi(X)\big]\Big)\nonumber\\
&\quad\qquad\qquad\qquad\ +\lambda_y \Big(\mathbb{E}\big[\phi\big(f(X)\big)\big] -\mathbb{E}\big[\phi(Y)\big]\Big).
\label{eq:cycleGAN}
\end{align}
The first two terms in \eqref{eq:cycleGAN} encourage $f$ and $g$ to be approximate inverses of one another, similar to the role of the distortion $\Delta_1(f,g;P,Q)$ in $\mathsf{UD}_1$ from \eqref{eq:IPM_UD}. While the original Cycle GAN formulation did not require $f$ and $g$ to be isometries, this constraint was introduced in followup works \citep{Hoshen_2018_CVPR}. We thus see that Cycle GAN, with a relaxed isometry requirement of $f$ and $g$, is a particular instantiation of the UBGM divergence. 

%corresponds This discussion demystifies cycle GAN and roots it in optimal transport as an instantiation of our Unbalanced Gromov Monge discrepancy where the requirement on mappings to be isometries is relaxed.

\iffalse
\paragraph{Cycle-GAN is an Instantiation of $UD$} 
some works imposed isometries , other just $f\circ g=id_{\cY}$ etc ..

Next we will use a simple kernelization \fi

%\zg{STOPPED HERE}

 \subsection{Relation to Past Works}%: Relaxations of Formulation \eqref{eq:dGHC} of $d_{GH}$}
%\youssef{shorten maybe this section a little}

We show briefly here the construction of two well-known discrepancies between mm spaces, namely the GW distance \citep{memoli2011gromov}, and the Gromov-Monge (GM) distance \citep{memoli2021distance}. The starting point of defining these two distances is the `correspondence set formulation'
of $d_{\mathsf{GH}}$ from \eqref{eq:dGHC}.  

%Let $(\cX,d_{\cX},P)$ and $(\cY,d_{\cY},Q)$ be two mm spaces. 
%\youssef{shorten below maybe and add details later. }
\paragraph{Gromov-Wasserstein distance.} In a nutshell, the GW distance is an $L^{p}, p\geq 1$ relaxation of the $L^{\infty}$ norm in formulation \eqref{eq:dGHC} of $d_{\mathsf{GH}}$, along with a Kantorovich relaxation of the correspondence set using couplings.  
\begin{definition}[Gromov-Wasserstein distance \citep{memoli2011gromov}] The GW distance between $(\cX,d_{\cX},P)$ and $(\cY,d_{\cY},Q)$ is
\begin{align*}
    \mathsf{GW}(P,Q):=\inf_{\pi\in\Pi(P,Q)}\|\Gamma_{\cX,\cY}\|_{L^p(\pi\otimes\pi)}
\end{align*}
where $\Gamma_{\cX,\cY}(x,y,x',y'):= \big|d_\cX(x,x')- d_\cY(y,y')\big|$, %\int_{\cX\times\cY\times\cX\times \cY}\mspace{-20mu}\big|d_{\cX}(x,x')-d_{\cY}(y,y')\big|^p \dd\pi^{\otimes 2}(x,y,x',y')
and $\Pi(P,Q)$ is the set of all couplings of $P,Q$. % i.e. probability measures on $\cX\times\cY$ that has marginals $P$ and $Q$ respectively.
\end{definition}

Another closely related formulation is the unbalanced GW distance from  \citet{sejourne2020unbalanced}. 
% For any entropy function $\phi$ \zg{why introduce this new term entropy function??'. Pls define the unbalanced distance via regular divergence, which we already set up. Either way, f-divergence is the standard term for what they use. }, denote the $\phi$-divergence as $\mathsf{D}_\phi(\cdot\|\cdot)$, and
For any divergence\footnote{\citet{sejourne2020unbalanced} used $\mathsf{f}$-divergences \citep{csiszar1967information} for $\sD_\cX$ and $\sD_\cY$, but we provide a general definition.} $\sD_\cX$ on $\cX$, define its two-fold extension $\mathsf{D}^{\otimes2}_\cX(P\|Q):=\mathsf{D}_\cX(P\otimes P\|Q\otimes Q)$.
\begin{definition}[Unbalanced Gromov-Wasserstein distance \citep{sejourne2020unbalanced}]
Let $\cM_+(\cX)$ be the set of all nonnegative Borel measures on $\cX$. The unbalanced GW distance between $(\cX,d_{\cX},P)$ and $(\cY,d_{\cY},Q)$~is  %\zg{no $\phi_i$ anymore...}
\begin{align*}\mspace{-3mu}
    &\mathsf{UGW}(P,Q):=
    \inf_{\pi\in \cM_+(\cX\times\cY)}  \|\Gamma_{\cX,\cY}\|_{L^1(\pi \otimes \pi)}\mspace{-2mu}+\mspace{-2mu} \mathsf{D}^{\otimes2}_{\cX}(\pi_1\|P)\mspace{-2mu} +\mspace{-2mu} \mathsf{D}^{\otimes2}_{\cY}(\pi_2\|Q)
\end{align*}
where $\pi_1,\pi_2$ are the marginals of $\pi$ on $\cX$ and $\cY$.%, respectively.
\end{definition}

%\zg{there's not much discussion going on here, but this should be a discussion section. We are just reciting definitions. It's not very clear what is the point. Connections are clearly there but we don't highlight/discuss them...}
% Comparing $\mathsf{GW}$ and $\mathsf{UGW}$ to UBGM, we see 
% Note that the unbalanced BGM

% The UBGM distance from Definition \ref{def:UD} relaxes the BGM distance in a similar fashion to how 

The unbalanced relaxation of the GW distance is similar to how our UBGM distance (Definition \ref{def:UD}) relaxes the BGM distance. A crucial difference is that both~the BGM distance and its unbalanced version explicitly encode bidirectional mappings, which are important in applications as they alleviate the need to recompute the coupling matrix given new datapoints. 

%\youssef{Mention here Unbalanced Gromov Wasserstein}
\paragraph{Gromov-Monge distance.} More recently, \citet{memoli2021distance} presented another extension of  $d_{\mathsf{GH}}$ to a discrepancy between mm spaces. Termed the GM distance, it considers an $L^p$  Monge relaxation of \eqref{eq:dGHC}, as opposed to the Kanotrovich-based approach of GW. Namely, instead of using couplings, the correspondence set now comprises Monge maps, i.e.,
% $\mathcal{T}(P,Q):=\{f:\cX\to \cY:\, f_{\sharp }P=Q\}$. Setting $\pi_{f}:=(\mathrm{id},f)_{\sharp }P$ for $f\in \mathcal{T}(P,Q)$, we have that $\pi_{f} \in \Pi(P,Q)$.  
all measurable maps $f:\cX\to \cY$ s.t.~$f_\sharp P=Q$. Also for arbitrary $f$, denote $\pi_{f}:=(\mathrm{id},f)_{\sharp }P$. Clearly for Monge maps $f$ that pushes $P$ to $Q$, $\pi_{f}\in \Pi(P,Q)$.
\begin{definition}[Gromov-Monge distance \citep{memoli2021distance}]
The GM distance between two mm spaces
$(\cX,d_{\cX},P)$ and $(\cY,d_{\cY},Q)$ is
\begin{align*}
    \mathsf{GM}(P,Q):=\inf_{f: \cX \to \cY,\, f_{\sharp }P=Q}  \|\Gamma_{\cX,\cY} \|_{L^p(\pi_f^{\otimes 2})}
\end{align*}
% \begin{align*}
%     \mathsf{GM}(P,Q):=\inf_{f\in\cT(P,Q)}  \|\Gamma_{\cX,\cY} \|_{L^p(\pi_f^{\otimes 2})}.
% \end{align*}
where  $\left\| \Gamma_{\cX,\cY} \right\|_{L^p(\pi_{f}\otimes \pi_f)}=\Delta^{(p)}_{\cX}(f;P).$
\label{def:GM}
\end{definition}

Comparing $\sD_p$ to $\sf{GM}$ above, we see that while the latter uses a single low metric distortion maps (with a cost of the form $\Delta^{(p)}_{\cX}(f;P)$), our BGM distance uses two such mappings that are approximately inverses (as enforced by $\Delta^{(p)}_{\cX,\cY}(f,g;P,Q)$). % ensure the cycle consistency between these maps. 
In a sense our definition is a symmetrized and cycle consistent version of $\mathsf{GM}$. %As discussed in Section \ref{SUBSEC:cgan}, the latter property is often desired in applications. %\zg{this is a nice disucssion and we should add something of that flavor to the end of the GW bit.} % The latter Having two cycle-consistent mappings is of big interest in applications, as already demonstrated in unveiling the link of our proposed Unbalanced Gromov Monge Divergence to cycle GAN. 

\section{Kernelization: Generalized Maximum Mean Discrepancy}\label{SEC:GMMD}

%\youssef{for Kernelization here we can cite for relaxing marginals in RKHS Bach paper \url{https://arxiv.org/pdf/2101.05380.pdf} and \url{https://arxiv.org/abs/2002.03179}}
%\zg{check consistency of fonts for divergences.}

Motivated by computational considerations, we now instantiate the divergences $\mathsf{D}_\cX$ and $\mathsf{D}_\cY$ in $\mathsf{UD}_p^{\cF,\cG}$ (see Definition \ref{def:UD}) as maximum mean discrepancies (MMDs) \citep{gretton2012kernel}. We coin the resulting kernelized divergence as the \emph{generalized maximum mean discrepancy} (GMMD). MMDs can be efficiently computed and offer flexibility in picking the proper kernel for each space. We start by reviewing preliminaries on MMDs (Section \ref{subsec:rkhs}), after which we present the kernelized UBGM distance (Section \ref{subsec:gmmd}), and explore its empirical convergence rates (Section~\ref{subsec:statistical}).

% We coin this resulting divergence \emph{Generalized Maximum Mean Discrepancy.} (GMMD)

% convergence guarantees for

% In this section, we consider a specific choice of $\mathsf{D}_\cX$ and $\mathsf{D}_\cY$. Namely we use maximum mean discrepancy (MMD) \cite{MMD} on both spaces. Kernelization using MMD in optimal transport  appeared in \cite{jagarlapudi2020statistical,vacher2021dimension} under smoothness assumptions on the underlying measures. In our case, using MMD offers significant computational convenience, as well as flexibility of picking the proper kernel for each space. For completeness we first in Section \ref{subsec:rkhs} give preliminaries on MMD, and then present in Section \ref{subsec:gmmd} our unbalanced Gromov Monge formulation for $D_{\cX},D_{\cY}$ being the MMD, and for pushforward maps restricted to mapping classes $\mathcal{F}$ and $\mathcal{G}$. We coin this resulting divergence \emph{Generalized Maximum Mean Discrepancy.} (GMMD). We  provide in Section \ref{subsec:statistical} statistical convergence guarantees for empirical GMMD, computed given samples from the two distribution. 

\subsection{Reproducing Kernel Hilbert Spaces}\label{subsec:rkhs}

We define reproducing kernel Hilbert spaces (RKHS) and the associated MMD. For a separable space $\cX$ and a continuous, positive definite, real-valued kernel $k_\cX:\cX\times\cX\rightarrow\RR$, let $\cH_\cX$ denote the corresponding RKHS, in which for any $f\in\cH_\cX$, we have $f(x)=\langle f(\cdot),k(x,\cdot) \rangle_{\cH_\cX}$, for any $x\in\cX$. See \cite{berlinet2011reproducing} for existence and uniqueness of $\cH_\cX$. There is a natural way to embed $\cP(\cX)$ into $(\cH_\cX,k_\cX)$, given by the kernel mean embedding 
\[
    \mu_\cX P(x):=\int_\cX k_\cX(x,y)\dd P(y)=\EE\big[k_\cX(x,Y)\big],
\]
where $Y\sim P$. This enables defining a discrepancy measure between probability distribution as the RKHS distance between their kernel mean embeddings.

%we may measure discrepancy between probability measures as the distance between their kernel mean embeddings in $\cH_X$. This gives rise to the $\mathsf{MMD}$.

\begin{definition}[Maximum mean discrepancy]
Let $\cH_\cX$ be an RKHS. The MMD between $P,Q\in\cP(\cX)$ is
\begin{align*}
    &\mathsf{MMD}_\cX(P,Q):= \|\mu_\cX P - \mu_\cX Q\|_{\cH_\cX}\\
    &= \left( \int k_\cX(x,y)\dd (P-Q)(x)\dd (P-Q)(y) \right)^{1/2}.
\end{align*}
\end{definition}
When the kernel $k_\cX$ is characteristic, as defined next, $\mathsf{MMD}_\cX$ metrizes the space of distributions $\cP(\cX)$.
\begin{definition}[Characteristic kernel]
The kernel $k_\cX$ of an RKHS $\cH_\cX$ is called characteristic if the mean embedding $\mu_\cX:\cP(\cX)\to\cH_\cX$ is injective.
\end{definition}

Also recall that characteristic kernels enable defining a metric on the $\cX$ space~\citep{sejdinovic2013equivalence}. Namely, defining $\rho_{k_\cX}(x,x'):= k_\cX(x,x)+k_\cX(x',x')-2k_\cX(x,x')$, for $x,x'\in\cX$, we have that $\big(\cX,\sqrt{\rho_{k_\cX}}\big)$ is a metric space. To simplify notation, we henceforth denote $\rho_{k_\cX}$ by $\rho_\cX$. 

\subsection{Generalized MMD}\label{subsec:gmmd}

The GMMD is defined as follows. Throughout we assume that $\cX$ and $\cY$ are compact with diameters bounded by $K$, and specialize to the case of $p=1$.
%As an example, we give the following new definition of the generalized MMD. Suppose $k_\cX$ and $k_\cY$ are fixed characteristic kernels that are associated with the spaces $\cX,\cY$. Throughout we assume $\cX,\cY$ are compact and their diameters are both bounded by $K$.
\begin{definition}[Generalized MMD between Metric Measure Spaces]
Let $k_\cX$ and $k_\cY$ be characteristic kernels on $\cX$ and $\cY$, respectively. The GMMD between $(\cX,d_{\cX},P)$ and $(\cY,d_{\cY},Q)$ is 
\begin{align*}
   &\mathsf{UD}^{\cF,
    \cG}_{k_\cX,k_\cY}(P\|Q):= \inf_{\substack{f\in\cF\\g\in\cG}}   \Delta_1(f,g;P,Q)
    + \lambda_x \mathsf{MMD}_{\cX}( P, g_{\sharp }Q)+ \lambda_y \mathsf{MMD}_{\cY}(f_\sharp  P, Q).
\end{align*}
\end{definition}
% For simplicity, we henceforth abbreviate GMMD as $\mathsf{UD}^{\cF,     \cG}_{k_\cX,k_\cY}$, whenever no ambiguity arises. \zg{it doesn't seem we use this shorthand much in this section, just to state Theorem 1. I suggest omitting it here, stating the theorem in full notation (the function classes appear in the bound anyway), and defining it when it is actually used (Section 5? proofs?)}
% \zg{ALL THIS TEXT WAS PART OF THE DEFINITION, BUT IT DOESN'T BELONG THERE:
% The $\inf$ is over mapping classes $\cF,\cG$ between $\cX,\cY$, both of which are supposed to contain more than 1 element, and we will always take subclasses of continuous mappings. $\mathsf{MMD}_{\cX}$ and $\mathsf{MMD}_{\cY}$ use kernels $k_\cX$ and $k_\cY$ respectively. \\
% MOST OF THIS WAS ALREADY DEFINED AND NEEDS NO REPETITION. RE CONTINUITY OF FUNCTION, WHEN YOU INTRODUCE $\cF$ AND $\cG$ FOR THE FIRST TIME, YOU SHOULD SAY THEY ARE ASSUMED TO BE SUBSETS OF $C^0$ THROUGHOUT. 
% }
\begin{remark}[GMMD is a divergence]
% %\zg{We don't need to say all of this since we already have Prop2. Just say that in accordance to the framework of Prop2, MMDs (with appropriate assumptions on kernels) are metrics, in particualr divergences, and so GMMD is a divergence. Same goes for continuity.}
% Suppose kernels $k_\cX$ and $k_\cY$ are bounded, continuous and characteristic, and $\cF,\cG$ are dense around the isometric bijections (when $\cX$ and $\cY$ are isometric). By Proposition \ref{prop:divergence}, we have that $\mathsf{UD}^{\cF,     \cG}_{k_\cX,k_\cY}$ is a (symmetric) divergence. In particular, $\mathsf{UD}^{\cF,     \cG}_{k_\cX,k_\cY}(P\|Q)\geq 0$ with equality iff there exists isometries $f,g$, such that $f_\sharp  P =Q$, $g_\sharp  Q=P$, and $f=g^{-1}$. Weak continuity of $\mathsf{UD}^{\cF,     \cG}_{k_\cX,k_\cY}$ also follows from the proposition.
MMDs w.r.t.~characteristic kernels are metrics (in particular, divergences) on the space of Borel probability measures. Proposition \ref{prop:divergence} thus implies that GMMD is a divergence. Further, under the mild condition that $k_\cX,k_\cY$ are bounded from above, MMDs are weakly continuous. Hence, weak continuity of GMMD also follows from the proposition, given that $\cF,\cG$ are compact.
\end{remark}

% \youssef{you should delete the remark then? you need to say kernel continuous charact and bounded not only bounded}

%\youssef{we should add a remark that we can use kernel induced metric or natural metrics of the metric space}
\begin{remark}[Kernels specify GMMD]
GMMD can be fully specified by the kernels if one defines the mm  spaces using the kernel induced metrics $\rho_\cX$ and $\rho_\cY$. %Hence 
%While we could choose whatever metric that is suitable for $\cX,\cY$, it is sometimes of interest if we replace the usual metric by the kernel induced metric $\rho_\cX,\rho_\cY$.% The applicability of this kind of non-Euclidean (e.g. $\rho_\cX$, $L^p$ norms) metrics is one of the major difference between our approach and procrustes methods. As illustrated in our synthetic experiments, such metrics offer as good results as Euclidean norm, and the utility of such flexibility is left for future work.
\end{remark}

\subsection{GMMD Empirical Estimation Rates}\label{subsec:statistical}

We now study the convergence rate of the two-sample plugin estimator of GMMD. Let $(X_i)_{i=1}^n$ and $(Y_i)_{i=1}^m$ be i.i.d. samples from $P$ and $Q$, respectively. Denote by $P_n:=n^{-1}\sum_{i=1}^n\delta_{X_i}$ and $Q_m:=m^{-1}\sum_{i=1}^m\delta_{Y_i}$ the empirical measures associated with these samples. Since GMMD is weakly continuous (for compact $\cF$ and $\cG$), we immediately have $\mathsf{UD}^{\cF,     \cG}_{k_\cX,k_\cY}(P_n\|Q_m)\to \mathsf{UD}^{\cF,     \cG}_{k_\cX,k_\cY}(P\|Q)$ as $n,m\to\infty$ a.s. The focus of this section is the rate at which this convergence happens.

\begin{theorem}\label{thm:convergence}
Suppose $k_\cX,k_\cY$ are uniformly bounded by a constant $C$, and the diameters of $\cX$ and $\cY$ are bounded by $K$. Further suppose that $\cF$ and $\cG$ are compact in $d_\cF$ and $d_\cG$, respectively. Then
\begin{align*}
    \EE\left[\left|\mathsf{UD}^{\cF,
    \cG}_{k_\cX,k_\cY}(P\|Q)-\mathsf{UD}^{\cF,
    \cG}_{k_\cX,k_\cY}(P_n\|Q_m)\right|\right]
  &\lesssim \lambda_y \delta_n(\cF_{k_{\cY}} ) + \lambda_x\delta_m(\cG_{k_{\cX}} ) + \delta_{n,m}(\cF,\cG)\\
    &\qquad\quad+ \lambda_x C^{\frac 12}n^{-\frac 12}+\lambda_yC^{\frac 12}m^{-\frac 12} + K(n\wedge m)^{-1}
\end{align*}
where $\cF_{k_{\cY}}:=\{k_\cY\circ(f,f):\,f\in\cF\}$ and
\begin{align*}
    \delta_n(\cF_{k_{\cY}} )\mspace{-2mu}&:=\mspace{-2mu} \inf_{\alpha>0}\mspace{-2mu} \left(\mspace{-1mu}\alpha \mspace{-1mu}+\mspace{-1mu} \frac{1}{n} \int_\alpha^{2C}\mspace{-12mu} \log\big(N(\cF_{k_{\cY}} ,\|\mspace{-2mu}\cdot\mspace{-2mu}\|_\infty,\tau)\big) \dd \tau \right)^{\mspace{-4mu}\frac 12}\mspace{-8mu},
\end{align*}
with $\cG_{k_{\cX}}$ and $\delta_m(\cG_{k_{\cX}})$ defined analogously, and 
\begin{align*}
   \delta_{n,m}(\cF,\cG)\mspace{-2mu}:= \mspace{-2mu} \inf_{\alpha>0}\mspace{-2mu} \Bigg(&\mspace{-2mu}\alpha\mspace{-2mu}+\mspace{-2mu} \frac{1}{\sqrt{n\wedge m}}\mspace{-3mu}\int_\alpha^K \mspace{-5mu}\Big(\mspace{-3mu}\log\big(N(\cF,d_\cF,\tau)\big)+\log\big(N(\cG,d_\cG,\tau)\big)\Big)^{\frac 12} \dd \tau \Bigg).
\end{align*}

% , $\cG_{k_{\cX}}:=\{k_\cX\circ(g,g):\,g\in\cG\}$ and 

% $\delta_n(\cF_{k_{\cY}} ) = \inf_{\alpha>0} \left(\alpha + \frac{1}{n} \int_\alpha^{2C} \log(N(\cF_{k_{\cY}} ,\|\cdot\|_\infty,\tau)) \dd \tau \right)^{1/2}$, and similar for $\delta_m(\cG_{k_{\cX}} )$.
% % where $N(\cF_{k_{\cY}} ,\|\cdot\|_\infty,\tau)$ is the covering number of function class $\cF_{k_{\cY}} $ with respect to the infinity norm with radius $\tau>0$.
% $\delta_{n,m}(\cF,\cG) = \inf_{\alpha>0} \Big(\alpha + \frac{1}{\sqrt{\min\{n,m\}}}\int_\alpha^K \sqrt{\log(N(\cF,d_\cF,\tau))+\log(N(\cG,d_\cG,\tau))} \dd \tau \Big)$.
% % \begin{align*}
% %     &\delta_n(\cF_{k_{\cY}} ) = \inf_{\alpha>0} \left(\alpha + \frac{1}{n} \int_\alpha^{2C} \log(N(\cF_{k_{\cY}} ,\|\cdot\|_\infty,\tau)) \dd \tau \right)^{1/2}\\
% %     &\delta_m(\cG_{k_{\cX}} ) = \inf_{\alpha>0} \left(\alpha + \frac{1}{m} \int_\alpha^{2C} \log(N(\cG_{k_{\cX}} ,\|\cdot\|_\infty,\tau)) \dd \tau \right)^{1/2}\\
% %     &\delta_{n,m}(\cF,\cG) = \inf_{\alpha>0} \Big(\alpha + \frac{1}{\sqrt{\min\{n,m\}}}\int_\alpha^K \sqrt{\log(N(\cF,d_\cF,\tau))+\log(N(\cG,d_\cG,\tau))} \dd \tau \Big)
% % \end{align*}
\end{theorem}

Theorem \ref{thm:convergence} bounds the estimation error for general function classes $\cF$ and $\cG$ in terms of the appropriate entropy integrals. The proof is given in Appendix \ref{app:thm_convergence} and relies on standard chaining arguments and bounds on Rademacher chaos complexity (cf. e.g, \citet{sriperumbudur2016optimal}). In general, these entropy integrals cannot be further simplified due to the dependence on the arbitrary classes $\cF$ and $\cG$. Nevertheless, as discussed next, Theorem \ref{thm:convergence} can be instantiated to obtain explicit rates for particular function classes of interest.%\zg{better phrasing?}

\begin{remark}[Special cases]
Corollary \ref{cor:convergence} in Appendix \ref{app:cor} instantiates the result for two important cases: when the function classes are (i) Lipschitz and (ii) parametric. We obtain a convergence rate of $O\big((1/n)^{1/2d_x}+(1/m)^{1/2d_y}\big)$ %$O\left(\left(\frac{d_y}{n}\right)^{\frac{1}{2d_x}}+\left(\frac{d_x}{m}\right)^{\frac{1}{2d_y}}\right)$
for the Lipschitz case when $d_x,d_y>2$, %\zg{add comment that rate can be improved via a direct argument through $\mathsf{W}_1$---please verify this first!}\zz{the rate still has a 1/2 exponent using W1}, 
and $O\big((n\wedge m)^{-1/2}\big)$ %$\left(\frac{1}{n\wedge m}\right)^{\frac{1}{2}})$ 
for the parametric case. The latter is of particular practical interest, as NNs offer a convenient and trainable model for the bidirectional maps (see next section).     
\end{remark}

\section{NUMERICAL EXPERIMENTS}\label{sec:exp}
 %  \youssef{and unaligned word embeddings matching between different natural languages. we will see if things work } 
We present applications of GMMD in shape matching. We parametrize the bidirectional maps $f,g$ as neural networks with parameters $\theta$ and $\psi$ respectively. Algorithm \ref{alg:cyclegan} (Appendix \ref{app:exp}) summarizes the optimization of the GMMD objective as function of $\theta$ and $\psi$. All experiments are run on the same machine with 4 core CPUs and a Tesla T4 GPU. The examples below highlight the qualitative and quantitative behavior of GMMD, illustrating the fact that GMMD is applicable to the same tasks as classical methods such as GW and UGW for finding correspondences. Further, GMMD amortizes the computational cost, as it results in continuous mappings that generalize to unseen datapoints drawn from the same distributions.
\begin{figure*}[ht!]
\hspace{-3.5mm}
\centering
 \begin{subfigure}[t]{0.23\textwidth}
 \centering
  \includegraphics[width=0.8\linewidth]{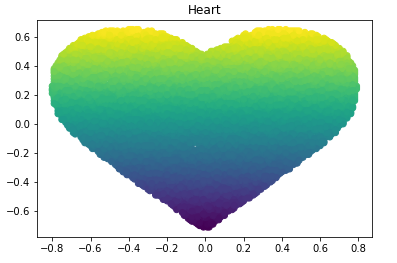}
   \vspace{-2mm}
   \caption{Heart shape ($P$)}
      \label{fig:heart}
   \vspace{-1mm}
  \end{subfigure}
 \ \hspace{0.3mm}
 \begin{subfigure}[t]{0.23\textwidth}
 \centering
 \includegraphics[width=0.8\linewidth]{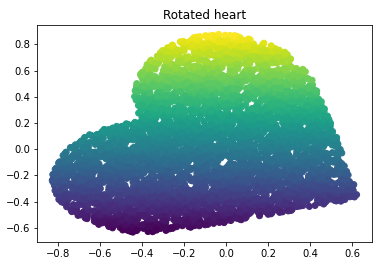}
    \vspace{-2mm}
   \caption{Rotation ($Q_b$)}
   \label{fig:rotation}
 \end{subfigure}
 \ \hspace{0.3mm}
 \begin{subfigure}[t]{0.23\textwidth}
 \centering
  \includegraphics[width=0.8\linewidth]{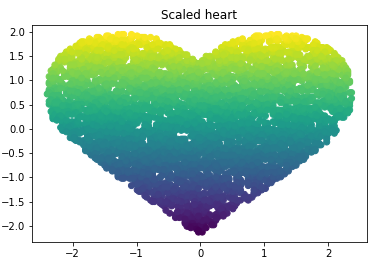}
   \vspace{-2mm}
   \caption{Scaling ($Q_c$)} 
     \label{fig:scaling}
   \vspace{-1mm}
 \end{subfigure}
  \ \hspace{-0.2 mm}
 \begin{subfigure}[t]{0.23\textwidth}
 \centering
  \includegraphics[width=0.8\linewidth]{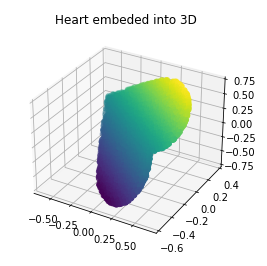}
   \vspace{-2mm}
   \caption{3D embedding ($Q_d$)}
      \label{fig:3D}
   \vspace{-1mm}
 \end{subfigure}
 \ 
  \caption{Heart shape and its transformations. }
  \label{fig:hearts}
\vspace{-2mm}
\end{figure*}

\begin{figure*}[!t]
\hspace{-3.5mm}
\centering
 \begin{subfigure}[t]{0.28\textwidth}
 \centering
  \includegraphics[width=1.1\linewidth]{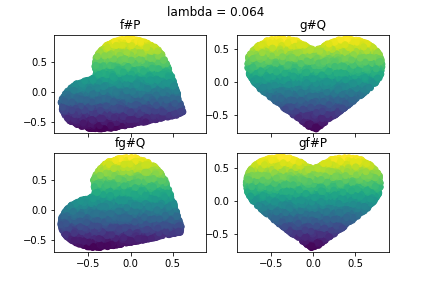}
   \vspace{-5mm}
   \caption{GMMD: $P$ vs. $Q_b$}%(\protect{\subref{fig:heart}}) \& $Q$ (\protect{\subref{fig:rotation})}.}
%   \vspace{5mm}
 \end{subfigure}
 \ \hspace{5mm}
 \begin{subfigure}[t]{0.28\textwidth}
 \centering
  \includegraphics[width=1.1\linewidth]{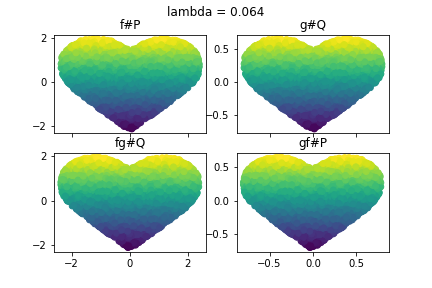}
   \vspace{-5mm}
 \caption{GMMD: $P$ vs. $Q_c$}%{GMMD $P$ (\protect{\subref{fig:heart}}) \& $Q$ (\protect{\subref{fig:scaling})}.}
%   \vspace{5mm}
 \end{subfigure}
 \ \hspace{5mm}
 \begin{subfigure}[t]{0.28\textwidth}
 \centering
  \includegraphics[width=1.1\linewidth]{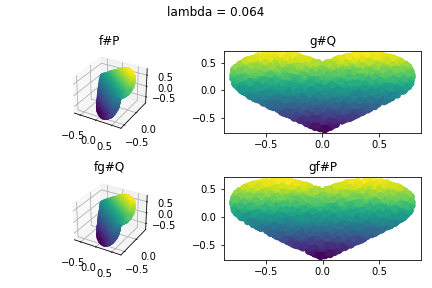}
   \vspace{-5mm}
   \caption{GMMD: $P$ vs. $Q_d$}%{GMMD maps: $P$ (\protect{\subref{fig:heart}}) \& $Q$ (\protect{\subref{fig:3D})}.}
%   \vspace{5mm}
 \end{subfigure}
 \ \hspace{5mm}
 %\begin{subfigure}[t]{0.23\textwidth}
 %\centering
 % \includegraphics[width=1.1\linewidth]{Aistats2022/figs/biplane_0.064.png}
 %  \vspace{-2mm}
%   \caption{}
 %  \vspace{-1mm}
 %\end{subfigure}
 \ \hspace{5mm}
 %\ \hspace{-1.8mm}
 \begin{subfigure}[t]{0.28\textwidth}
 \centering
 \includegraphics[width=1.1\linewidth]{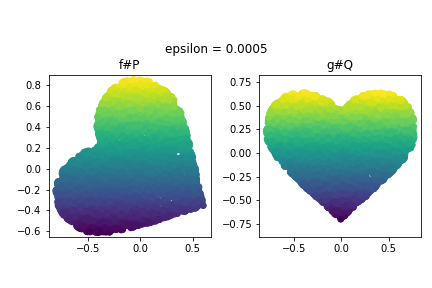}
     \vspace{-7mm}
  \caption{GW : $P$ vs. $Q_b$.}
 \end{subfigure}
 \ \hspace{5mm}
 \begin{subfigure}[t]{0.28\textwidth}
 \centering
 \includegraphics[width=1.1\linewidth]{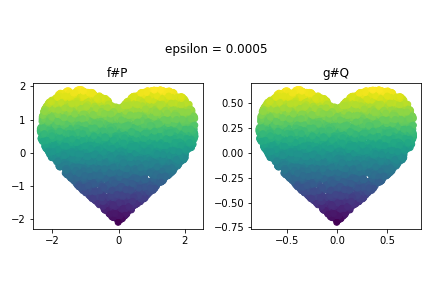}
    \vspace{-7mm}
    \caption{GW : $P$ vs. $Q_c$.}
 \end{subfigure}
 \ \hspace{5mm}
%%%%%%%%%%%%%%%%%%%%%%%%%%
\centering
 \begin{subfigure}[t]{0.28\textwidth}
 \centering
 \includegraphics[width=1.1\linewidth]{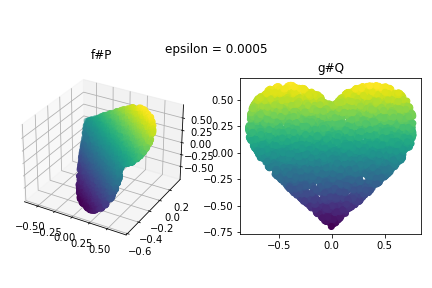}
    \vspace{-7mm}
    \caption{GW : $P$ vs. $Q_d$.}
 \end{subfigure} 
  \caption{First row: learned continuous GMMD Mappings  and their cycle consistency in shape matching. Second row: discrete entropic GW Barycentric Mappings. The color code in the heatmaps is coordinate based.}% this facilitates tracking the correspondence between the shapes. }
  \label{fig:heart_mapping}
\vspace{-2mm}
\end{figure*}

\subsection{GMMD For Shape Matching}
%To illustrate the applicability of $\GMMD$ comparing to $\GW$
We consider here  matching of synthetic shapes, specifically  a 2-dimensional heart shape given in Figure \ref{fig:hearts}(\subref{fig:heart})   and its transformations through rotation (\subref{fig:rotation}), scaling (\subref{fig:scaling}) and isometrically embedding into 3-dimensional space (\subref{fig:3D}). The data is generated via sampling $n=4000$ points for each shape. %The distributions for each matching experiment are denoted as $P=n^{-1}\sum_{i=1}^n \delta_{x_i}$ corresponding to Figure \ref{fig:hearts}(\subref{fig:heart})  and $Q=n^{-1}\sum_{i=1}^n \delta_{y_i}$ for  (\subref{fig:rotation}), (\subref{fig:scaling}), or (\subref{fig:3D}).\zg{ALTERNATIVE: 
The distributions for each matching experiment are the empirical measures induced by these samples, with $P$ corresponding to \ref{fig:hearts}(\subref{fig:heart}) and $Q_b$, $Q_c$, and $Q_d$ corresponding to subfigures (\subref{fig:rotation}), (\subref{fig:scaling}), and (\subref{fig:3D}), respectively (the subscript is suppressed when we simultaneously refer to several experiments).

For matching experiments, we compute GMMD using Algorithm \ref{alg:cyclegan} for $\lambda_x=\lambda_y=\lambda^{-1}$, where $ \lambda\in\{2^{i}\times10^{-3}:i=0,\cdots,9\}$. We use a uniform mixture of Gaussian kernels to define $\mathsf{MMD}_{\cX}$ and $\mathsf{MMD}_{\cY}$  and use kernel induced metrics $\rho_{\cX}$, and $\rho_{\cY}$ in the distortion $\Delta_1$. The bandwidths we used for the Gaussian kernels are median of the metric $\times\{.0001,.001,.01,.05,.25,1,4,20,100,1000\}$. The architecture of the bidirectional maps $f$ and $g$ is a 3-layer ReLU NN with 200 neurons each, and an output dimension matching the target distribution dimension. We use Adam optimizer \citep{kingma2014adam} for 3000 epochs with a learning rate $10^{-3}$. 

\begin{figure*}[ht!]
\captionsetup[subfigure]{justification=centering}
\hspace{-19mm}
\centering
 \begin{subfigure}[t]{0.3\textwidth}
 \centering
  \includegraphics[width=1.3\linewidth]{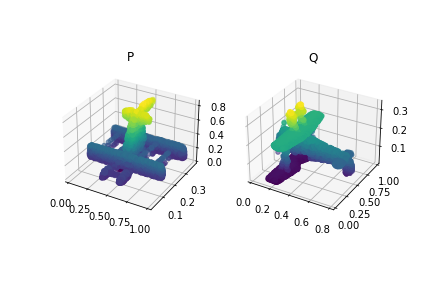}
   \vspace{-10mm}
   \caption{Original $P$ and $Q$.}
%   \vspace{-1mm}
 \end{subfigure}
 \ \hspace{2mm}
 \begin{subfigure}[t]{0.3\textwidth}
 \centering
  \includegraphics[width=1.3\linewidth]{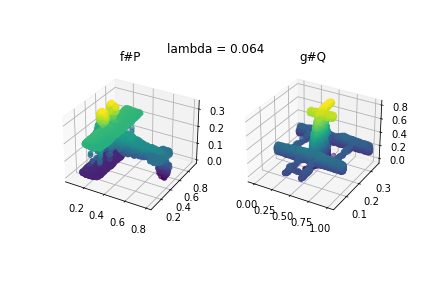}
   \vspace{-10mm}
 \caption{GMMD maps.}
%   \vspace{-1mm}
 \end{subfigure}
 \ \hspace{2mm}
 \begin{subfigure}[t]{0.3\textwidth}
 \centering
  \includegraphics[width=1.3\linewidth]{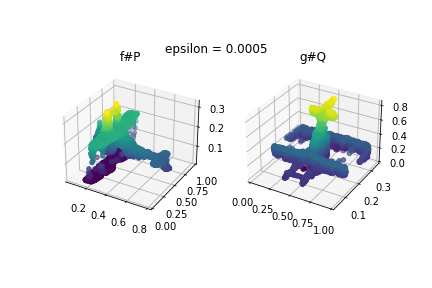}
   \vspace{-10mm}
   \caption{GW barycentric maps.}
%   \vspace{-1mm}
 \end{subfigure}
%  \ \hspace{0.3mm}
  \caption{Matching 3D shapes with GMMD and entropic GW. }
  \label{fig:biplanes_mapping}
\vspace{-2mm}
\end{figure*}
In Figure \ref{fig:heart_mapping}, the first row corresponds to GMMD matching for $\lambda=0.064$. For each case, we see that the learned bidirectional  maps of  GMMD successfully perform the matching, i.e., $f_{\sharp }P \approx Q$ and $g_{\sharp }Q \approx P$. We also confirm that they satisfy the cycle consistency property, i.e.,  $f\circ g\approx \mathrm{id}_{\cY}$ and $g\circ f\approx \mathrm{id}_{\cX}$. The second row shows entropic GW matchings \citep{pmlr-v48-peyre16}. We use the POT library \citep{flamary2021pot} to perform discrete entropic GW for an entropic regularization parameter $\varepsilon=5e^{-4}$. Note that entropic GW results in a coupling matrix $\pi$. To obtain discrete mappings of the points we employ barycentric maps \citep{barycentricmaps}, i.e., $\tilde{f}(x_i):=(\sum_{j=1}^n \pi_{ij})^{-1}\sum_{j=1}^n\pi_{ij}y_j$ and $\tilde{g}(y_j):=(\sum_{i=1}^n \pi_{ij})^{-1}\sum_{i=1}^n \pi_{ij}x_i$. 

We see from Figure \ref{fig:heart_mapping} that the GMMD continuous maps and the discrete barycentric maps induced by the GW coupling are on par qualitatively in these matching tasks. To confirm this quantitatively, we consider the matching of the heart shape and its rotation (Figure \ref{fig:hearts}(\subref{fig:rotation})) since for this case an isometry exists, i.e., there are $f^\star,g^\star$ with $\Delta_1(f^\star,g^\star;P,Q_b)=0$. Tables \ref{tab:gmmd_rotate} and \ref{tab:gw_rotate1} state the values of $\mathsf{MMD}_{\cY}(f_{\sharp }P,Q_b)$, $\mathsf{MMD}_{\cX}(P,g_{\sharp }Q_b)$, and $\Delta(f,g;P,Q_b)$ across different regularization parameters for the GMMD and GW-based mappings, respectively. We see that GMMD and GW indeed result in small MMD and distortions values. GMMD yields a smaller distortion than GW. Note that we also have evaluated UGW \citep{sejourne2020unbalanced} with the code provided by the authors and found that it is sensitive to hyper-parameters choice, and did not result in an accurate matching on the considered tasks. We think that more tuning is needed for UGW. Additional results and ablation on regularization parameters and shapes are given in Appendix~\ref{app:exp}.

Figure \ref{fig:biplanes_mapping} presents a more complex matching of 3D shapes that consist in two different biplanes models from the Princeton Shape benchmark \citep{shilane2004princeton} (for $n=8000$). We see that GMMD and GW are also on par and that the GMMD bidirectional maps result in less outliers than barycentric GW-based maps. This robustness of GMMD is due to the use of kernel induced metrics. Quantitative evaluation is  presented in Appendix \ref{app:exp}.

\begin{table}[htbp]
\caption{Evaluating GMMD's mappings for $P$ vs. $Q_b$.}
\begin{center}
\begin{tabular}{|l|l|l|l|l|}
\hline
\textbf{$\lambda$}  &\textbf{$\GMMD$}&\textbf{$\MMD_\cX$}&\textbf{$\MMD_\cY$}&\textbf{$\Delta$} \\
%\hline
%0.512  &1.33   &0.0524   &0.00207    &2.50 \\
\hline
0.256  &0.0794     &0.0294    &0.0294    &0.0801\\
\hline
0.128  &0.00825   &4.94e-4   &4.05e-4  &0.0574  \\
\hline
0.064  &0.00776   &0.00227   &0.00190   &0.0560\\
\hline
0.032  &0.0924    &2.90e-4  &0.00386   &2.76\\
%\hline
%0.016  &0.0579    &0.00137   &3.33e-4  &3.51\\
%\hline
%0.008  &0.0136      &0.00188  &0.00181   &1.24\\
%\hline
%0.004  &8.04e-4  &2.85e-4  &2.55e-4 &0.0661\\
%\hline
%0.002  &0.00624 &0.00149    &0.00147    &1.64\\
%\hline
%0.001  &0.00378   &0.00118   &0.00114   &1.47\\
\hline
\end{tabular}
\end{center}
\label{tab:gmmd_rotate}
\vskip -0.2in
\end{table}

\begin{table}[htbp]
\caption{GW barycentric maps for $P$ vs. $Q_b$.}
\begin{center}
\begin{tabular}{|l|l|l|l|l|}
\hline
\textbf{$\epsilon$}  &\textbf{$\GW$}&\textbf{$\MMD_\cX$}&\textbf{$\MMD_\cY$}&\textbf{$\Delta$} \\
%\hline
%0       &0.00113 &3.55e-15  &1.78e-15   &0.803\\
\hline
0.0005  &0.00134   &0.00420 &0.00299    &0.696\\
\hline
0.005   &0.00660  &0.127      &0.116      &1.73\\
\hline
0.05    &0.0424   &0.615 &0.613      &6.69\\
\hline
0.5     &0.0686    &3.99       &4.12      &22.9\\
%\hline
%5       &0.0699    &4.86       &4.89      &26.2\\
\hline
\end{tabular}
\end{center}
\label{tab:gw_rotate1}
\end{table}

\iffalse

%\subsubsection{Comparison With $\GW$ And $\UGW$}
%The popular $\GW$ and $\UGW$ methods also gives matching between shapes. We present here the baseline using these methods, demonstrating their qualitative and quantitative performance. It is standard to use entropic regularized versions of these methods, and this has both positive effect on computational efficiency, and negative effect on the information it carries about the underlying distributions, see Fig \ref{fig:gw_hearts} and Table \ref{tab:gw_rotate}. We choose entropic regularizer multiplier $\epsilon\in\{5\times10^{-i}:i=0,\cdots,4\}\cup\{0\}$. To obtain a deterministic mapping, we use the barycentric mapping.
\begin{table}[htbp]
\caption{$\mathsf{UGW}$ Between Heart And Rotated Heart}
\begin{center}
\begin{tabular}{|l|l|l|l|l|}
\hline
\textbf{$\epsilon$}  &\textbf{$\mathsf{UGW}$}&\textbf{$\MMD_\cX$}&\textbf{$\MMD_\cY$}&\textbf{$\Delta$} \\
\hline
    10 &    0.277856  &5.00562  &5.00562 & 25.8038 \\
\hline
     1  &   0.199544  &4.96044  &4.96038 & 25.6224\\
\hline
     0.1 &  0.189629  &4.57678  &4.57691 & 24.0855 \\
\hline
     0.01 & 0.178746 & 3.34716  &3.34713 & 19.1421 \\
\hline
\end{tabular}
\end{center}
\label{tab:ugw_Rotated}
\end{table}
\fi

%\youssef{stopped here, will continue}

\subsection{GMMD Amortization and Generalization} 
For the biplane matching experiment with sample size $n=8000$ from each distribution, Table~\ref{tab:runtime} reports the training time for GMMD and the runtime for GW and UGW computings. The computational complexity of training GMMD maps amounts to the complexity of gradient descent in NN training for $3000$ epochs, which is $O(n)$. For entropic GW and UGW, however, the implementations are variants of the Sinkhorn algorithm, whose complexity scales as $O(n^3)$. The longer training time for GMMD is due to the large number of epochs used in gradient descent (namely, $3000$), but at inference time this cost is amortized since we obtain continuous maps that generalize to unseen datapoints (see Appendix \ref{app:exp} for quantitative evaluation of the generalization). For instance, matching 8000 new datapoints sampled from $P$ and $Q$ each using the learned mapping requires 63 ms, while with GW one would incur the cost of recomputing the coupling (26 minutes)---a three order of magnitude speedup.   

%One remarkable property of our method through NN is that we obtain a mapping that takes not only the training data, but also any additional amount of data. Once the NN is trained to certain precision, which is guaranteed by the statistical convergence theorems, for any reasonable amount of extra data we can simply push through the data through the trained NN. This amortization property is the main advantage of $\GMMD$ method comparing to $\GW$ and its variants which use couplings and obtain discrete mappings.  We first compare their run time in Table \ref{tab:runtime} where the tests use 8000 points. 

\begin{table}[htbp]
\caption{Training Time (in seconds) comparison using 8000 samples from the biplanes data.}
\begin{center}
\begin{tabular}{|l|l||l|l|}
\hline
\textbf{$\epsilon$}  &\textbf{$\GW$}&\textbf{$\lambda$}&\textbf{$\GMMD$}\\
%\hline
%   0       &736.702   &0.001    &5039.63\\
\hline
   0.0005  &1566.86&0.002     &5048.11\\

%\hline
%   0.005   &3886.84   &0.004     &5035.43\\

%\hline
%   0.05    &1755.91   &0.008     &5039.79\\

%\hline
%   0.5     &335.886   &0.016    &5022.21\\

%\hline
%   5       &336.46   &0.032      &5041.16\\
\hline
    \textbf{$\epsilon$} & \textbf{$\UGW$}  &0.256  &5026.5\\
\hline
%    0.01   &34.5931 &0.128  &5025.67\\
% \hline
    0.1     &28.7508 &0.064  &5052.89\\
%\hline
%    1       &28.0992 &0.512   &5132.95\\
\hline
\end{tabular}
\end{center}
\label{tab:runtime}
\end{table}

%\paragraph{Word embedding Alignment} ?

\iffalse
\youssef{Do you have this with cluster time? @Zhengxin add time only for Biplanes and lambda= 0.064.}

\begin{table}[htbp]
\caption{GMMD amortization Time}
\begin{center}
\begin{tabular}{|l|l|l|l|l|}
\hline
\textbf{$\lambda$}  &\textbf{ab}&\textbf{ac}&\textbf{ad}&\textbf{Biplanes} \\
\hline
0.512      &0.153 &0.0531 &0.0552   &0.0470\\
\hline
0.256      &0.0729 &0.0526 &0.0561   &0.0504\\
\hline
0.128      &0.0690 &0.0558  &0.0517 &0.0546\\
\hline
0.064      &0.0541 &0.0537 &0.0562 &0.0638\\
\hline
0.032      &0.0615 &0.0727 &0.0648 &0.0602\\
\hline
0.016      &0.0522 &0.0517  &0.0500 &0.0555\\
\hline
0.008      &0.0512  &0.0494 &0.0480 &0.0513\\
\hline
0.004      &0.0483 &0.0502 &0.0524 &0.0485\\
\hline
0.002      &0.0497 &0.0572 &0.0476 &0.0572\\
\hline
0.001      &0.0484 &0.0503 &0.0502 &0.0511\\
\hline
\end{tabular}
\end{center}
\label{tab:amortization_time}
\end{table}
\fi
\section{CONCLUSION}
This paper introduced the UBGM divergence---a novel discrepancy measure between distributions across heterogeneous spaces, which employs bidirectional and cycle-consistent mappings. We established structural properties of the UBGM divergence and highlighted its intimate connection to the so-called cycle GAN. We also presented a kernelized variant of this divergence, termed GMMD, and analysed its statistical estimation from samples. Numerical experiments demonstrated the promise of this new divergence and compared it to other known metrics, such as the GW and UGW distances. Appealing future directions include extending the GMMD to allow optimization over kernels, sharper statistical bounds, as well as connections between the UBGM divergence and the UGW distance (in particular, under what conditions they coincide).
%theoretical questions on the BGM metric and when it coincides with GW are left for future investigation.\zg{COPIED FROM SLACK:This sentence is not really saying much. I would either omit it altogether or expand into 2-3 sentences to discuss future directions for real: sharpening the statistical analysis, revealing connections with GW, optimizing over choice of kernels (the GAN setup), etc.}

% \subsubsection*{Acknowledgements}
% All acknowledgments go at the end of the paper, including thanks to reviewers who gave useful comments, to colleagues who contributed to the ideas, and to funding agencies and corporate sponsors that provided financial support. 
% To preserve the anonymity, please include acknowledgments \emph{only} in the camera-ready papers.

% \begin{thebibliography}{}
% \setlength{\itemindent}{-\leftmargin}
% \makeatletter\renewcommand{\@biblabel}[1]{}\makeatother
% \bibitem{} J.~Alspector, B.~Gupta, and R.~B.~Allen (1989).
%     \newblock Performance of a stochastic learning microchip.
%     \newblock In D. S. Touretzky (ed.),
%     \textit{Advances in Neural Information Processing Systems 1}, 748--760.
%     San Mateo, Calif.: Morgan Kaufmann.

% \bibitem{} F.~Rosenblatt (1962).
%     \newblock \textit{Principles of Neurodynamics.}
%     \newblock Washington, D.C.: Spartan Books.

% \bibitem{} G.~Tesauro (1989).
%     \newblock Neurogammon wins computer Olympiad.
%     \newblock \textit{Neural Computation} \textbf{1}(3):321--323.
% \end{thebibliography}

%\bibliographystyle{plain}
\bibliographystyle{abbrvnat}
\bibliography{references}
\onecolumn
\appendix
%\begin{center}
%%%%\textbf{Appendixary Material:\\ Cycle Consistent Probability Divergence Across Different Spaces}
%\end{center}

% \zg{I think we can omit the table of content for the appendix.}
% The Appendix is organized as follows: 
% \begin{itemize}
% \item \ref{app:proof} Proofs:
%     \begin{itemize}
%         \item \ref{app:Dmetric_proof} Proof of Proposition \ref{pro:Dmetric}
%         \item
%         \ref{app:existence_proof} Proof of Lemma \ref{lem:existence}
%         \item
%         \ref{app:divergence_proof} Proof of Proposition \ref{prop:divergence}
%         \item
%         \ref{app:thm_convergence} Proof of Theorem \ref{thm:convergence}
%         \item
%         \ref{app:cor} Special cases of Theorem \ref{thm:convergence}
%         \item
%         \ref{app:continuity} Continuity of the functional $\cL$
%     \end{itemize}
    
%     \item
%     \ref{app:exp} Experiments:
%     \begin{itemize}
%         \item \ref{app:algo} Algorithm
%         \item
%         \ref{app:word} Additional High Dimensional Experiments on Unaligned Word Embeddings
%         \item
%         \ref{app:comparison_gw_ugw} Comparison to GW and UGW
%         \item
%         \ref{app:amortization} Amortization
%     \end{itemize}
    
% \end{itemize}

% \tableofcontents
%\section*{Additional High Dimensional Experiments on Unaligned Word Embeddings}

\section{Proofs}\label{app:proof}
To simplify notation we denote $\cL_{P,Q}(f,g): = \lambda_x\mathsf{MMD}_{\cX}( P, g_{\sharp }Q)+\lambda_y\mathsf{MMD}_{\cY}(f_\sharp  P, Q) +  \Delta_1(f,g;P,Q)$, which is the functional that is optimized in definition of $\mathsf{UD}^{\cF,     \cG}_{k_\cX,k_\cY}$.

\subsection{Proof of Proposition \ref{pro:Dmetric}}\label{app:Dmetric_proof}
The symmetry and positivity follows directly from definition and Lemma \ref{lem:existence}, which is proven below. For the triangle inequality, fix 3 mm spaces $(\cX,d_\cX,P)$, $(\cY,d_\cY,Q)$, $(\cZ,d_\cZ,R)$ and  functions $f_1,f_2,g_1,g_2$ (over the appropriate domains) with $(f_1)_\sharp P=Q$, $(f_2)_\sharp Q=R$, $(g_1)_\sharp Q=P$, $(g_2)_\sharp R=Q$. We only show the derivation for $\Delta^{(p)}_{\cX,\cY}$; a similar argument applies to $\Delta^{(p)}_{\cX},\Delta^{(p)}_{\cY}$. For $\Delta^{(p)}_{\cX,\cY}$, we have
\begin{align*}
    \Delta^{(p)}_{\cX,\cY}(f_1,g_1;P,Q)&+\Delta^{(p)}_{\cX,\cY}(f_2,g_2;Q,R)\\
    &=\Big(\int |d_\cX(x,g_1(y))-d_\cY(f_1(x),y)|^p\dd P(x)\dd Q(y)\Big)^{\frac 1p}\\
    &\qquad\qquad\qquad\qquad+\Big(\int |d_\cY(y,g_2(z))-d_\cZ(f_2(y),z)|^p\dd Q(y)\dd R(z)\Big)^{\frac 1p}\\
    &= \Big(\int |d_\cX(x,g_1(g_2(z)))-d_\cY(f_1(x),g_2(z))|^p\dd P(x)\dd R(z)\Big)^{\frac 1p}\\
    &\qquad\qquad\qquad\qquad+\Big(\int |d_\cY(f_1(x),g_2(z))-d_\cZ(f_2(f_1(x)),z)|^p\dd P(x)\dd R(z)\Big)^{\frac 1p}\\
    &\geq \Big(\int|d_\cX(x,g_1(g_2(z)))-d_\cZ(f_2(f_1(x)),z)|^p\dd P(x)\dd R(z)\Big)^{\frac 1p}\\
    &=\Delta^{(p)}_{\cX,\cY}(f_2\circ f_1, g_1\circ g_2;P,R).
\end{align*}
%So $\Delta^{(p)}_{\cX,\cY}(f_1,g_1;P,Q)+\Delta^{(p)}_{\cX,\cY}(f_2,g_2;Q,R)\geq \Delta^{(p)}_{\cX,\cY}(f_2\circ f_1, g_1\circ g_2;P,R)$. 
Hence $\mathsf{D}_p$ is a metric on $\mM$. 

\subsection{Proof of Lemma \ref{lem:existence}}\label{app:existence_proof}

Suppose $\{f_n\}_{n\in\NN}$ and $\{g_n\}_{n\in\NN}$ are sequence such that $\Delta_p(f_n,g_n;P,Q)\rightarrow0$. We will show that up to extracting subsequences, these sequences converge $P\otimes Q$ a.s. to isometrics, $f$ and $g$, respectively, such that $f=g^{-1}$. The argument first shows that there is a countable dense $S\subseteq\cX$ such that the distortion function $\phi_n$ (defined below) converges on $S\times S$ to 0. Then we take a subsequence of $f_n$ that converges on $S$, and show that this subsequence also converges $P$-a.s. on $\cX$, and the limit is an isometry. After applying the same to $\{g_n\}_{n\in\NN}$, we conclude the desired convergence and demonstrate that the limits $f$ and $g$ satisfy $f=g^{-1}$.
% \zg{say what you are going to show; give a general flavor of the proof. `We will show that these sequences converge a.s. to isometrics, $f$ and $g$, respectively, such that $f=g^{-1}$. The argument first shows... then... START WITH AN OVERVIEW OF THE PROOF STRATEGY. IT WILL BE MUCH MORE ENJOYABLE TO READ AND EASIER TO PARSE AFTER HAVING THE BROAD IDEA IN MIND. } 

We first consider the term $\Delta^{(p)}_{\cX}(f_n;P)$. Since 
\[\int \big|d_\cX(x,x')-d_\cY\big(f_n(x),f_n(x')\big)\big|^p\dd P(x)\dd P(x')\rightarrow0,\]
we may assume that, up to extraction of subsequences, we have 
\[\phi_n(x,x'):=\big|d_\cX(x,x')-d_\cY\big(f_n(x),f_n(x')\big)\big|\rightarrow0,\quad P^{\otimes 2}-a.s.\]
Set $\Omega=\{(x,x'):\phi_n(x,x')\rightarrow0\}$ as the set of pairs for which the convergence occurs, and let $\Omega_x=\{x':(x,x')\in\Omega\}$ be the slice at $x\in\cX$ in the first coordinate. Then $P^{\otimes 2}((\cX\times\cX)-\Omega\big)=0$, and by Fubini's theorem 
\[\int_\cX P(\Omega_x)\dd P(x)=P^{\otimes 2}(\Omega)=1. \]
Denoting $A=\{x: P(
\Omega_x) =1 \}$, we thus have $P(A)=1$, and hence $A$ is dense. We next construct $S$ as a countable dense subset of $A$.
% \zg{ZHENGXIN, LIKE BEFORE, YOU HAVE TO ADD SOME TEXT TO EXPLAIN WHAT YOU ARE DOING. SAY WHY YOU NEED $P(A)=1$, ETC. ALSO, THIS IS VERY DENSE... MAKE SOME EQUATIONS DISPLAY SO THAT IT IS EASIER ON THE EYE.}

\underline{Step 1 -- Separability of convergence points:} 
% \zg{RELATE TO PREVIOUS STEP AND APPLY SIMILAR PRESENTATION MODIFICATION. THIS PROOF NEEDS WORK AND I WILL LOOK AT IT AGAIN AFTER YOU MAKE A FEW PASSES.} 
We present an inductive construction of a countable dense subset $S\subset \cX$ such that $\phi_n$ converges to 0 on $S\times S$. First take any $x_0\in A$ and define \[S_0:=\big\{x':\phi_n(x_0,x')\text{ does not converge to 0}\big\}=\cX-\Omega_{x_0},\]
then $P(S_0)=0$ since $x_0\in A$. Suppose we have points $x_0,...,x_k\in A$, such that $\phi(x_i,x_j)\rightarrow0$ for $i,j=0,\ldots,k$, and define $S_i = \cX-\Omega_{x_i}$ for $i=0,\ldots,k$. Define function $$\psi_k(x):=\min_{0\leq i\leq k}\{d_\cX(x,x_i)\},$$ and set $w_k=\argmax \psi_k(x)$. Suppose $\psi_k(w_k)>0$, otherwise $x_0,...,x_k$ is already dense. Since $\psi_k(x)$ is continuous, $B_k=\{\psi_k(x)> \psi_k(w_k)/2\}$ is a nonempty open set on $\cX$. Notice that set $C_k=A-\cup_{0\leq i\leq k} S_i$ still have probability 1, and any point $x'\in C_k$ satisfies that $\phi_n(x',x_i)$ converges for all $i=0,...,k$. Since $P$ has full support, $B_k\cap C_k$ is not empty, hence we pick $x_{k+1}\in B_k\cap C_k$. Inductively we have sequence $\{x_k\}_{k\in\NN}$ such that $\phi_n(x_i,x_j)$ converges for any $i,j\in\NN$. Denote $S=\{x_k\}_{k\in\NN}$.

Now we prove that $S$ is dense in $\cX$. Suppose it is not, then there is an $\epsilon>0$ and an $\tilde{x}\in\cX$ such that $d_\cX(\tilde{x},x_k)>\epsilon$, for all $k\in\NN$. So $\psi_k(w_k)\geq \psi_k(\tilde{x})>\epsilon$. By construction, $d_\cX(x_{k+1},x_i)\geq\psi_k(w_k)/2\geq \epsilon/2$, for all $i\leq k$, so $d_\cX(x_i,x_j)\geq \epsilon/2$ for any $i\neq j$. This is a contradiction since $\cX$ is compact.

\underline{Step 2 -- Convergence to isometry:} Next we find a subsequence of $\{f_n\}_{n\in\NN}$ such that it converges on $S$ to an isometry $f$, and extend this convergence to a.s. on $\cX$. Now we have a countable set $S\subseteq A$ that is dense in $\cX$ such that $\phi_n(s,t)\rightarrow0$, $\forall s,t\in S$. We can thus take a subsequence of $\{f_n\}_{n\in\NN}$ such that it converges on $S$ pointwise to a mapping $f$. Without loss of generality (WLOG) we assume $f_n$ converges, for any $s\in S$, as any subsequence still approaches infinum and the subsequent $\phi_n$ still converges to 0 on $S\times S$. Since $\lim_{n\to\infty}\big|d_\cX(s,t)-d_\cY\big(f_n(s),f_n(t)\big)\big|=0$, by continuity we have $d_\cX(s,t) = d_\cY(f(s),f(t))$. For any $x\not\in S$, fix a sequence 
% \zg{you'll need to switch notation here since $s_k$ is take; maybe use tilde.}\zz{re this and previous comment I prefer to have $s_\ell$ for the sequence here; the previous $x_k$ are just extracted subset from $\cX$ so using $s_n$ looks not proper. Also if using tilde for $s_\ell$, so should be applied to $t_\ell$, which is out of context here}\zg{the limit point should correspond to the notation of the sequence; $s_\ell\to x$ doesn't look good. You can keep the $x_k$ from before, but then here take $s\in\cS$ and $s_\ell\to s$. Similarly for $t$.}
$\{s_\ell\}_{\ell\in\NN}\subseteq S$ with $s_\ell\rightarrow x$, and define $$f(x):=\lim_{\ell\to\infty} f(s_\ell).$$ So $$d_\cY\big(f(x),f(x')\big)=\lim_{\ell\to\infty} d\big(f(s_\ell),f(t_\ell)\big) = \lim_{\ell\to\infty} d_\cX(s_\ell,t_\ell)=d_\cX(x,x')$$ for $x,x'\in \cX$, and $s_\ell\rightarrow x$, $t_\ell\rightarrow x'$. So $f$ is extended to an isometry on $\cX$. 

Now consider any $x\in C=\cap_{s\in S} \Omega_{s}$, where $P(C)=1$. Clearly for all $s\in S$, 
$$\lim_{n\to\infty} d_\cY\big(f_n(x),f_n(s)\big)=d_\cX(x,s).$$ We have a sequence $\{s_\ell\}_{\ell\in\NN}$ in $S$ such that $s_\ell\rightarrow x$, and
\begin{align*}
    d_\cY\big(f_n(x),f(x)\big) &\leq d_\cY\big(f_n(x),f_n(s_\ell)\big) + d_\cY\big(f_n(s_\ell),f(s_\ell)\big) + d_\cY\big(f(s_\ell),f(x)\big),
\end{align*}
which is true for all $\ell$. Fix $\ell$, and take upper limit in $n$, we have
\begin{align*}
    \limsup_n d_\cY\big(f_n(x),f(x)\big)\leq 2d_\cX(s_\ell,x),
\end{align*}
which holds for all $\ell$. Then we can take $\ell\to \infty$ which shows that $\lim_{n\to\infty} f_n(x) = f(x)$, i.e. $f_n$ converges on $C$. So $f_n$ converges to $f$ $P$-a.s. Similarly, via subsequence extraction, we can find $g_n$ that also converges $Q$-a.s. to an isometry $g$. As $\cX,\cY$ are compact, the limits $f$ and $g$ are both surjective and have inverses.

Now consider the third term, i.e., $\int \big|d_\cX\big(x,g_n(y)\big)-d_\cY\big(f_n(x),y\big)\big|^p\dd P(x)\dd Q(y)\rightarrow0$. Since $\cX,\cY$ are bounded, by dominated convergence theorem we have
\begin{align*}
    &\int \big|d_\cX\big(x,g(y)\big)-d_\cY\big(f(x),y\big)\big|^p\dd P(x)\dd Q(y) \\
    &= \lim_n \int \big|d_\cX\big(x,g_n(y)\big)-d_\cY\big(f_n(x),y\big)\big|^p\dd P(x)\dd Q(y)
    =0.
\end{align*}
Thus $d_\cX(x,x')=d_\cX\big(g\circ f(x),x'\big)$ holds $P\otimes g_\sharp Q$-a.s., hence holds densely on $\cX\times\cX$. By continuity this holds for all $\cX\times\cX$. So $g\circ f (x)=x$, i.e. $f=g^{-1}$.

\subsection{Proof of Proposition \ref{prop:divergence}}\label{app:divergence_proof}
% First we prove positivity. Non-negativity is straightforward. Now consider the case when $\mathsf{UD}^{\cF,\cG}_p(P\|Q)=0$. Since $$\mathsf{UD}^{\cF,\cG}_p(P\|Q)\geq \inf_{\substack{f\in\cF\\ g\in\cG }} \Delta_p(f,g;P,Q),$$  \zg{this sentence in not clear---please rewrite. Say when one argument ends (e.g., non-negativity) and when another begins (identity of indiscernibles).}by Lemma \ref{lem:existence}, there exists a sequence of $\{(f_n,g_n)\}_{n\in\NN}$ such that $(f_n,g_n)$ converges  $P\otimes Q$-a.s. to a pair of isometries $(f,g)$. 
% \zg{expand: separate $f_n$ from $g_n$ and you can talk about appropriate a.s. convergence for each; also, pls add a mention to the limits $f$ and $g$ in this sentence.}, hence $(f_n)_\sharp P$ and $(g_n)_\sharp Q$ converges weakly to $f_\sharp P,g_\sharp Q$ by dominated convergence theorem.
% \zg{explain why a.s. convergence of functions implies weak convergence of pushforward measures.}. As $\sD_{\cX}$ and $\sD_{\cY}$ are weakly continuous on their arguments, we know that $\sD_{\cX}(g_\sharp Q,P)=\sD_{\cY}(f_\sharp P,Q)=0$, hence $f_\sharp P=Q$ and $g_\sharp Q=P$. Thus the two mm spaces are equivalent.
% % \zg{The notation you are using here is wrong; the main text now uses $\mathsf{UD}^{\cF,\cG}_p$; please adopt here as well.}

% \zg{use double bar notation between divergence arguments as well; apply throughout.}

% \zg{use $f_\sharp P$ instead of $f_\sharp  P$ throughout.}

% % \zg{EXAMPLE OF HOW THE ABOVE PARAGRAPH SHOULD BE WRITTEN:

Non-negativity of $\mathsf{UD}^{\cF,\cG}_p(P\|Q)$ is immediate. The fact that it nullifies when the mm spaces are equivalent, as specified in Definition \ref{DEF:equivalence}, is also straightforward. For the opposite implication, let $(\cX,d_{\cX},P)$ and $(\cY,d_{\cY},Q)$ be mm spaces such that $\mathsf{UD}^{\cF,\cG}_p(P\|Q)=0$. Since all summands in the definition of $\mathsf{UD}^{\cF,\cG}_p(P\|Q)$ are non-negative we have that
\[\inf_{f\in\cF,g\in\cG} \Delta_p(f,g;P,Q)=0.\]
By Lemma \ref{lem:existence}, there exist infimizing sequences $\{f_n\}_{n\in\NN}\subset\cF$ and $\{g_n\}_{n\in\NN}\subset\cG$ that converge $P,Q$-a.s. to isometries $f:\cX\to\cY$ and $g:\cY\to\cX$, respectively. This further implies weak convergence of the pushforward measures, i.e., $(f_n)_\sharp P\stackrel{w}{\to} f_\sharp  P$ and $(g_n)_\sharp Q\stackrel{w}{\to} g_\sharp  Q$.
% Indeed, EXPLAIN HERE WHY THE FORMER IMPLIES THE LATTER.
In fact, for any bounded continuous function $\phi$ on $\cY$, $\phi\circ f_n$ converges to $\phi\circ f$ $P$-a.s. Consequently, we have 
\[\int\phi(f_n(x))\dd P(x) \to \int\phi(f(x))\dd P(x),\] 
and hence $(f_n)_\sharp P\stackrel{w}{\to} f_\sharp  P$ (the argument for $g_n$ is analogous). Since $\sD_{\cX}$ and $\sD_{\cY}$ are weakly continuous in their arguments, we have
\begin{align*}
    0&=\lim_{n\to\infty} \Delta_p(f_n,g_n;P,Q)+\sD_{\cX}\big((g_n)_\sharp Q\big\|P\big)+\sD_{\cY}\big((f_n)_\sharp  P\big\|Q\big)\\
    &=\sD_{\cX}(g_\sharp Q\|P\big)+\sD_{\cY}(f_\sharp  P\|Q),
\end{align*}
which further implies that $\sD_{\cX}(g_\sharp Q\|P)=\sD_{\cY}(f_\sharp P\|Q)=0$. We conclude that  $f_\sharp P=Q$ and $g_\sharp Q=P$ for the isometries $f$ and $g$ that are inverses of each other, which establishes equivalence of the mm space.
% \ \\ 
% PLEASE WORK IN THE PRESENTATION OF YOUR PROOFS AND TRY TO ADHERE TO THE ABOVE STANDARDS.}

We next prove continuity. Suppose $P_n\stackrel{w}{\to}P$ and $Q_n\stackrel{w}{\to}Q$.
% \zg{add weak convergence notation and definition to notation section.}.
For any fixed $f\in\cF$ and $g\in\cG$,
% \zg{specify spaces when you fix an object} 
% \zg{stay consistent with ordering of summands; in definition and the above proof you have $\Delta$ term first and the the divergences. No need/reason to change.} 
\[\mathsf{UD}^{\cF,\cG}_p(P_n\|Q_n)\leq  \Delta_p(f,g;P_n,Q_n)+ \lambda_x\sD_\cX( g_{\sharp }Q_n\| P_n)+\lambda_y\sD_\cY(f_\sharp  P_n\| Q_n), \]
and by infimizing over $\cF,\cG$, we have
% \zg{presentation is not clear here; why did you fix $f,g$ and how is that related to the next step. Please improve.}
\begin{align*}
    \limsup_{n\to\infty} \mathsf{UD}^{\cF,\cG}_p(P_n\|Q_n) &\leq  \inf_{f\in\cF,g\in\cG}\lim_{n\to\infty} \Delta_p(f,g;P_n,Q_n) +  \lambda_x\sD_\cX( g_{\sharp }Q_n\|P_n )+\lambda_y\sD_\cY(f_\sharp  P_n\| Q_n) \\
    &= \mathsf{UD}^{\cF,\cG}_p(P\|Q).
\end{align*}
Thus $\mathsf{UD}^{\cF,\cG}$ is upper semi-continuous. If further $\cF,\cG$ are both compact, let $f^\star_{n},g^\star_{n}$ 
% \zg{NOTATION! You have to differentiate an arbitrary sequence of functions (like in the previous proof) from optimizers; usually a star superscript does the trick.} 
be minimizers for $\mathsf{UD}^{\cF,\cG}_p(P_n\|Q_n)$. Suppose $\{k_n\}_{n\in\NN}$ is the index sequence of a $\liminf$ subsequence of the sequence $\{\mathsf{UD}^{\cF,\cG}_p(P_n\|Q_n)\}_{n\in\NN}$. Since $\cF,\cG$ are both compact, we may also assume that $\{f^\star_{k_n}\}_{n\in\NN}$ converges in $\cF$, and $\{g^\star_{k_n}\}_{n\in\NN}$ converges in $\cG$. Denote by $f^\star$ and $g^\star$ the limits of $\{f^\star_{k_n}\}_{n\in\NN}$ and $\{g^\star_{k_n}\}_{n\in\NN}$, respectively. Also by Prokhorov's theorem, WLOG we can suppose that $(f^\star_{k_n})_\sharp P_{k_n}$ and $(g^\star_{k_n})_\sharp Q_{k_n}$ both converges weakly. Now we identify their limits. Since $\cF,\cG$ are assumed to be compact in sup-metrics, $\{(f^\star_{k_n},g^\star_{k_n})\}_{n\in\NN}$ converges uniformly, hence for any bounded continuous Lipschitz function $\phi$ on $\cY$, $\phi(f^\star_{k_n}(x))$ converges uniformly to $\phi(f^\star(x))$, hence $$\int \phi(f^\star_{k_n}(x))\dd P_{k_n}(x) \to \int \phi(f^\star(x))\dd P(x).$$ So $(f^\star_{k_n})_\sharp P_{k_n}\stackrel{w}{\to} f^\star_\sharp P$, and similarly $(g^\star_{k_n})_\sharp Q_{k_n}\stackrel{w}{\to} g^\star_\sharp Q$.
% \zg{This sentence is also not clear. Stick to linear presentation! Extract a subsequence with the desired property, denote their limits, state you claims about them. Avoid `Clearly'!! Your job at the moment is to prove each step. Afterwards we can omit ones that are too simple/straightforward}.
So 
% \zg{the 2nd equality below was split into two lines, but it fits in one (as you can see). No need to split if it fits. Also, please stick to a single (!!) step per line. If you have a chain of inequalities and you present them as a multi line equation---it is one step per line. I fixed it below but you have the old one commented out for reference. Please check for that and adapt throughout.}
\begin{align*}
    \mathsf{UD}^{\cF,\cG}_p(P\|Q)&\leq \Delta_p(f^\star,g^\star;P,Q) + \lambda_x\sD_\cX(  g^\star_{\sharp }Q\|P)+\lambda_y\sD_\cY(f^\star_\sharp  P\| Q) \\
    &=\lim_{n} \Delta_p(f^\star_{k_n},g^\star_{k_n};P_{k_n},Q_{k_n}) + \lambda_x\sD_\cX( (g^\star_{k_n})_{\sharp }Q_{k_n}\| P_{k_n})+\lambda_y\sD_\cY((f^\star_{k_n})_\sharp  P_{k_n}\| Q_{k_n})\\
    &=\liminf_{n}\mathsf{UD}^{\cF,\cG}_p(P_n\|Q_n)\\
    &\leq \limsup_{n}\mathsf{UD}^{\cF,\cG}_p(P_n\|Q_n)\\
    &\leq \mathsf{UD}^{\cF,\cG}_p(P\|Q),
\end{align*}
% \begin{align*}
%     \mathsf{UD}^{\cF,\cG}_p(P\|Q)&\leq \lambda_x\sD_\cX( P, \hat{g}_{\sharp }Q)+\lambda_y\sD_\cY(\hat{f}_\sharp  P, Q) +  \Delta_p(\hat{f},\hat{g};P,Q)\\
%     &=\lim_{n}\lambda_x\sD_\cX( P_{k_n}, (g_{k_n})_{\sharp }Q_{k_n})+\lambda_y\sD_\cY((f_{k_n})_\sharp  P_{k_n}, Q_{k_n})+\Delta_p(f_{k_n},g_{k_n};P_{k_n},Q_{k_n})\\
%     &=\liminf_{n}\mathsf{UD}^{\cF,\cG}_p(P_n\|Q_n)\leq \limsup_{n}\mathsf{UD}^{\cF,\cG}_p(P_n\|Q_n)\leq \mathsf{UD}^{\cF,\cG}_p(P\|Q),
% \end{align*}
hence $\lim_{n}\mathsf{UD}^{\cF,\cG}_p(P_n\|Q_n)= \mathsf{UD}^{\cF,\cG}_p(P\|Q)$, as desired.\qed

% \begin{proof}[Proof of Proposition \ref{prop:continuity}]
% Suppose $(P_n,Q_m)$ weakly converges to $(P,Q)$. Fix a pair $f,g$. So $\mathsf{GMMD}(P_n,Q_m)\leq \mathsf{MMD}_\cX( P_n, g_{\sharp }Q_m)+\mathsf{MMD}_\cY(f_\sharp  P_n, Q_m) +\lambda  \Delta(f,g;P_n,Q_m)$. Also the kernel is bounded and continuous and $\cX,\cY$ are compact, so $\limsup_{m,n} \mathsf{GMMD}(P_n,Q_m) \leq  \inf_{f,g}\lim_{m,n} \mathsf{MMD}_\cX( P_n, g_{\sharp }Q_m)+\mathsf{MMD}_\cY(f_\sharp  P_n, Q_m) +\lambda  \Delta(f,g;P_n,Q_m) = \mathsf{GMMD}(P,Q)$. If $\cF,\cG$ are both compact, then denote $f_{n,m},g_{n,m}$ as the minimizer for $P_n,Q_m$, and we suppose $\hat{f},\hat{g}$ is the limit of a $\liminf$ subsequence $f_{k_n,k_m},g_{k_n,k_m}$. Since the integrands are composition of $f_{k_n,k_m},g_{k_n,k_m}$ with continuous functions, they all converges in the $\sup$ norm, hence $|\mathsf{MMD}_\cX( P_{k_n}, \hat{g}_{\sharp }Q_{k_m})+\mathsf{MMD}_\cY(\hat{f}_\sharp  P_{k_n}, Q_{k_m}) +\lambda  \Delta(\hat{f},\hat{g};P_{k_n},Q_{k_m}) - \mathsf{GMMD}(P_{k_n},Q_{k_m})|\rightarrow0$ as $n,m\rightarrow\infty$. So
% \begin{align*}
%     \mathsf{GMMD}(P,Q)&\leq \lim_{n,m}\mathsf{MMD}_\cX( P_{k_n}, \hat{g}_{\sharp }Q_{k_m})+\mathsf{MMD}_\cY(\hat{f}_\sharp  P_{k_n}, Q_{k_m}) +\lambda  \Delta(\hat{f},\hat{g};P_{k_n},Q_{k_m})\\
%     &=\liminf_{n,m}\mathsf{GMMD}(P_n,Q_m)\leq \limsup_{n,m}\mathsf{GMMD}(P_n,Q_m)\leq \mathsf{GMMD}(P,Q),
% \end{align*}
% hence $\lim_{n,m} \mathsf{GMMD}(P_n,Q_m) = \mathsf{GMMD}(P,Q)$.
% \end{proof}

\subsection{Proof of Theorem \ref{thm:convergence}}\label{app:thm_convergence}

% zg{Start from presenting the decomposition of GMMD as at the start from the proof of Thm 1. Once you have the upper bound, say that we control each of the terms via the following technical lemmas (whose proof is deferred to the next appendix). Then present the lemma (IMPORTANT: don't just throw those lemmas at the beginning of the proof without any context! Present them before they are actually used). It is much harder to write a proof that is not just `correct' but also `makes sense' to the reader. You need to learn to do the latter.}
% \zg{After the lemmas are presented, use them to conclude the proof of the theorem.}
% We break the proof of Theorem \ref{thm:convergence} into the following lemmas.
To prove Theorem \ref{thm:convergence} it suffices to upper bound $\EE\big[\sup_{f,g}\big|\cL_{P,Q}(f,g)-\cL_{P_n,Q_m}(f,g)\big|\big]$. We have
\begin{align*}
    \sup_{f,g} \big|\cL_{P,Q}(f,g)-&\cL_{P_n,Q_m}(f,g)\big|\\
    &=  \sup_{f,g}  \Big|  \lambda_x\|\mu_\cX P-\mu_\cX g_\sharp Q\|_{\cH_\cX}-\lambda_x\|\mu_\cX P_n-\mu_\cX g_\sharp Q_m\|_{\cH_\cX}\\
    &\qquad\qquad+ \lambda_y\|\mu_\cY Q - \mu_\cY f_\sharp  P \|_{\cH_\cY} - \lambda_y\|\mu_\cY Q_m - \mu_\cY f_\sharp  P_n \|_{\cH_\cY}\\
    &\qquad\qquad\qquad\qquad+ \Delta_1(f,g;P_n,Q_m)-\Delta_1(f,g;P,Q) \Big|\\
    &\leq \sup_{g}  \lambda_x\Big|  \|\mu_\cX P-\mu_\cX g_\sharp Q\|_{\cH_\cX}-\|\mu_\cX P_n-\mu_\cX g_\sharp Q_m\|_{\cH_\cX}\Big|\\
    &\qquad\qquad+ \sup_{f} \lambda_y\Big|\|\mu_\cY Q - \mu_\cY f_\sharp  P \|_{\cH_\cY} - \|\mu_\cY Q_m - \mu_\cY f_\sharp  P_n \|_{\cH_\cY}\Big|\\
    &\qquad\qquad\qquad\qquad+ \sup_{f,g} \Big|\Delta_1(f,g;P_n,Q_m)-\Delta_1(f,g;P,Q) \Big|\\
    &\leq \sup_{g}   \lambda_x \|\mu_\cX g_\sharp Q-\mu_\cX g_\sharp Q_m\|_{\cH_\cX}+\lambda_x\|\mu_\cX P-\mu_\cX P_n\|_{\cH_\cX} \\
    &\qquad\qquad+ \sup_{f} \lambda_y\| \mu_\cY f_\sharp  P- \mu_\cY f_\sharp  P_n  \|_{\cH_\cY} + \lambda_y\|\mu_\cY Q -\mu_\cY Q_m \|_{\cH_\cY} \\
    &\qquad\qquad\qquad\qquad+ \sup_{f,g} \Big|\Delta_1(f,g;P_n,Q_m)-\Delta_1(f,g;P,Q) \Big|.\numberthis\label{EQ:sup_UB}
\end{align*}
We control each of the terms in the last line via the following technical lemmas (whose proof is deferred to the Appendix \ref{app:Complementary_thm_1}).

\begin{lemma}[Convergence of $\MMD$]\label{lem:mmd_convergence}
For mapping class $\cF$, recall that $\cF_{k_{\cY}}:=\{k_\cY\circ(f,f):\,f\in\cF\}$.
% \zg{this is not the notation we are using in the main text (for composition)! be consistent.}.
Under the same condition of Theorem \ref{thm:convergence}, we have
% \zg{this is a sloppy statement. You need to specify all the conditions for the lemma (e.g., boundedness of the kernel and anything else!)}
\begin{align*}
\EE\left[\sup_f \big\|\mu_\cY f_\sharp P_n - \mu_\cY f_\sharp  P \big\|_{\cH_\cY}\right] \lesssim & \inf_{\alpha>0} \left(\alpha + \frac{1}{n} \int_\alpha^{2C} \log\big(N(\cF_{k_{\cY}} ,\|\cdot\|_\infty,\tau)\big) \dd \tau \right)^{1/2} + \sqrt{\frac{C}{n}}.
\end{align*}
\end{lemma}

\begin{lemma}[Convergence of $\Delta_1$]\label{lem:distortion_convergence}
Under the same condition of Theorem \ref{thm:convergence}, we have
\begin{align*}
    &\EE\left[\sup_{f,g} \Big|\Delta_1(f,g;P_n,Q_m)-\Delta_1(f,g;P,Q) \Big| \right]\\
    &\qquad\lesssim \inf_{\alpha>0} \left(\alpha + \frac{1}{\sqrt{n\wedge m}}\int_\alpha^K \sqrt{\log\big(N(\cF,d_\cF,\tau)\big)+\log\big(N(\cG,d_\cG,\tau)\big)} \dd \tau \right) + \frac{K}{n\wedge m}.
\end{align*}
%\zg{are you hiding only an absolute constant here? If not, specify what the constant depends on.}\zz{yes only universal constant}
\end{lemma}

Proceeding from \eqref{EQ:sup_UB} and using the lemmas, we obtain the desired bound:
\begin{align*}
    \EE&\left[\left|\mathsf{UD}^{\cF,     \cG}_{k_\cX,k_\cY}(P\|Q)-\mathsf{UD}^{\cF,\cG}_{k_\cX,k_\cY}(P_n\|Q_m)\right|\right]\\
    &\qquad\lesssim\frac{K}{n\wedge m} + \lambda_y\inf_{\alpha>0} \left(\alpha + \frac{1}{n} \int_\alpha^{2C} \log\big(N(\cF_{k_{\cY}} ,\|\cdot\|_\infty,\tau)\big) \dd \tau \right)^{1/2} +  \lambda_y \sqrt{\frac{C}{m}} \\
    &\qquad\qquad+ \lambda_x\inf_{\alpha>0} \left(\alpha + \frac{1}{m} \int_\alpha^{2C} \log\big(N(\cG_{k_{\cX}} ,\|\cdot\|_\infty,\tau)\big) \dd \tau \right)^{1/2} +  \lambda_x\sqrt{\frac{C}{n}}\\
    &\qquad\qquad\qquad+ \inf_{\alpha>0} \left(\alpha + \frac{1}{\sqrt{n\wedge m}}\int_\alpha^K \sqrt{\log\big(N(\cF,d_\cF,\tau)\big)+\log\big(N(\cG,d_\cG,\tau\big))} \dd \tau \right) .
\end{align*}
\qed

%\zg{STOPPING HERE UNTIL WE HAVE ALL THE ABOVE IN SHAPE. I THINK YOU GET THE GENERAL IDEA SO YOU CAN PROCEED ADAPTION AND OPTIMIZING THE REST OF THE PROOFS ACCORDINGLY.}

\subsection{Complementary Proofs for Theorem \ref{thm:convergence}}\label{app:Complementary_thm_1}

\subsubsection{Proof of Lemma \ref{lem:mmd_convergence}}
First observe that $\big\|\mu_\cY f_\sharp P_n - \mu_\cY f_\sharp  P \big\|_{\cH_\cY} = \big\|n^{-1}\sum_{i=1}^n k_\cY\big(\cdot,f(X_i)\big)-\mu_\cY f_\sharp  P \big\|_{\cH_\cY}$. Let $\{\epsilon_i\}_{i\in\NN}$ be a sequence of i.i.d. Rademacher random variables and consider the following symmetrization. Suppose $X'_1,\cdots,X_n'$ are another i.i.d sequence from $P$ that is independent of $X_1,\cdots,X_n$. By Jensen's inequality we have
\begin{align*}
    &\EE\left[\sup_f\big\|\frac{1}{n}\sum_{i=1}^n k_\cY\big(\cdot,f(X_i)\big)-\mu_\cY f_\sharp  P \big\|_{\cH_\cY}\right] \\
    &\qquad\qquad= \EE\left[ \sup_f \Bigg\|\EE\left[ \frac{1}{n} \sum_{i=1}^n k_\cY\big(\cdot,f(X_i)\big) -\frac{1}{n} \sum_{i=1}^n k_\cY\big(\cdot,f(X_i')\big) \middle| X_1,\ldots,X_n  \right]
    \Bigg\|_{\cH_\cY}  \right]\\
    &\qquad\qquad\leq \EE\left[ \sup_f \Bigg\| \frac{1}{n} \sum_{i=1}^n k_\cY\big(\cdot,f(X_i)\big) -\frac{1}{n} \sum_{i=1}^n k_\cY\big(\cdot,f(X_i')\big) 
    \Bigg\|_{\cH_\cY}  \right]\\
    &\qquad\qquad = \EE\left[ \sup_f \Bigg\| \frac{1}{n} \sum_{i=1}^n \epsilon_i\Big(k_\cY\big(\cdot,f(X_i)\big) - k_\cY\big(\cdot,f(X_i')\big) \Big)
    \Bigg\|_{\cH_\cY}  \right]\\
    &\qquad\qquad\leq \frac{2}{n}\EE\left[\EE\left[\sup_f \Bigg\| \sum_{i=1}^n \epsilon_i k_\cY\big(\cdot,f(X_i)\big) \Bigg\|_{\cH_\cY} \middle| X_1,\ldots,X_n  \right]\right].\numberthis\label{EQ:E_sup_norm}
\end{align*}
The RKHS norm inside the conditional expectation can be further bounded as
% \zg{you are overflowing in the next equation; fix spacing or break to lines. Also, those huge $\sqrt{\cdot}$ look ugly. Let's (throughout) use $(\cdot)^{1/2}$ instead. Also, write out full indexing for the sums. GENERAL: WHEN A. W. van der Vaart WRITES A PAPER HE CAN ALLOW HIMSELF TO BE SLOPPY AND USE SHORTHAND NOTATION. UNTIL WE ARE AT HIS LEVEL, WE WRITE THINGS IN FULL. LIKE I MENTIONED BEFORE, OMITTING REDUNDANCIES IS EASY, BUT WE START FROM FULLY DETAILED PRESENTATION.}
\begin{align*}
    \left\|\sum_{i=1}^n \epsilon_i k_\cY\big(\cdot,f(X_i)\big) \right\|_{\cH_\cY} &=  \Bigg(\sum_{i,j=1}^n \epsilon_i \epsilon_j k_\cY\big(f(X_i),f(X_j)\big)\Bigg)^{1/2}  \\
    &\leq\Bigg( 2\Bigg|\sum_{i<j}^n \epsilon_i \epsilon_j k_\cY\big(f(X_i),f(X_j)\big)\Bigg| \Bigg)^{1/2} + \sqrt{nC},
\end{align*}
% \zg{Use equation numbers when appropriate to improve readability! E.g.:}
Inserting this back into \eqref{EQ:E_sup_norm}, we obtain
% \zg{square roots... and equation overflowing again. Also, please change notation for conditioning on $X_i$'s as I did below (throughout).}
\begin{align*}
    &\EE\left[\sup_f \big\|\mu_\cY f_\sharp P_n - \mu_\cY f_\sharp  P \big\|_{\cH_\cY}\right]\\
    &\qquad\leq \frac{2}{n}\EE\left[ \left(2\EE\Bigg[\sup_f\bigg|\sum_{i<j}^n \epsilon_i \epsilon_j k_\cY\big(f(X_i),f(X_j)\big)\bigg|\middle| X_1,\ldots,X_n\Bigg]\right)^{1/2}  \right] + 2\sqrt{\frac{C}{n}}.\numberthis\label{EQ:E_sup_norm_bound}
\end{align*}

% The first term contains \zg{what do you mean contains?}\zz{it's inside the square root in the expectation }the Rademacher chaos complexity of a class of kernels.
% % \zg{this notation is already taken! see statement of lemma. You cannot just redefine it as something else. Add a superscript 0 or a tile.}. 
% Formally for any set of $n$ points $\{x_i\}_i$, the complexity is defined as \zg{AGAIN, PRESENTATION SHOULD BE LINEAR. AFTER THE ABOVE EQUATION SAY: `Recall that the Rademacher chaos complexity (PROVIDE REFERENCE) of a function class $\cF$ is define as (DEFINE IN DISPLAY EQUATION), and Lemma A.2 bounds it as PROVIDE THE GENERAL BOUND. Evidently, the inner expectation of the right-hand side (RHS) of \ref{EQ:E_sup_norm_bound} corresponds to the Rademacher chaos complexity of the class $\tilde{cF}_{k_\cY}:=...$, and using the above we have...}

Recall that the Rademacher chaos complexity \citep{sriperumbudur2016optimal} of a kernel class $\cG$ is define as
$$ U(\cG,x_1,\dots,x_n):= \EE\left[\sup_{g\in\cG}\Bigg|\sum_{i<j}^n \epsilon_i \epsilon_j g(x_i,x_j)\Bigg|\right]. $$
% and Lemma A.2 from \citet{sriperumbudur2016optimal} bounds it as$$U(\cG,x_1,\dots,x_n)\lesssim $$
Evidently, the inner expectation of the right-hand side (RHS) of \ref{EQ:E_sup_norm_bound} corresponds to the Rademacher chaos complexity of the class $\cF_{k_\cY}$, and using Lemma A.2 from \citet{sriperumbudur2016optimal} we have
% \begin{align*}
%     U(\cF_{k_{\cY}} ,x_1,\ldots,x_n) := \EE\left[\sup_f\Bigg|\sum_{i<j} \epsilon_i \epsilon_j k_\cY\big(f(x_i),f(x_j)\big)\Bigg|\right].
% \end{align*}
% To bound this term we apply Lemma A.2 from \citet{sriperumbudur2016optimal}:
\begin{align*}
    U(\cF_{k_{\cY}} ,X_1,\ldots,X_n) \lesssim n^2\inf_{\alpha>0} \left(\alpha + \frac{1}{n} \int_\alpha^{2C} \log\big(N(\cF_{k_{\cY}} ,\|\cdot\|_\infty,\tau)\big) \dd \tau \right) + nC.
\end{align*}
Combining all previous bounds we have that
\begin{align*}
    \sup_f \|\mu_\cY f_\sharp  P_n - \mu_\cY f_\sharp  P \|_{\cH_\cY}
    \leq&  \frac{2}{n}\EE\left[\left(2\EE\Big[U(\cF_{k_{\cY}} ,X_1,\ldots,X_n)\Big|X_1,\ldots,X_n\Big] \right)^{1/2}  \right] + 2\sqrt{\frac{C}{n}}  \\
    \lesssim & \inf_{\alpha>0} \left(\alpha + \frac{1}{n} \int_\alpha^{2C} \log\big(N(\cF_{k_{\cY}} ,\|\cdot\|_\infty,\tau)\big) \dd \tau \right)^{1/2} + \sqrt{\frac{C}{n}}.
\end{align*}
 \qed

% \begin{remark}[High probability bound]
% \zg{Are you using this anywhere? where?}\zz{only here, as a supplementary remark. You can omit it if you think it's unnecessary}\zg{if not use then indeed unnecessary; please omit. }Applying the McDiarmid’s inequality to $\sup_f \|\mu_\cY P_n - \mu_\cY f_\sharp  P \|_{\cH_\cY}$, we have that with probability at least $1-\delta$, 
% $$\sup_f \|\mu_\cY P_n - \mu_\cY f_\sharp  P \|_{\cH_\cY} \leq \EE\left[\sup_f \|\mu_\cY P_n - \mu_\cY f_\sharp  P \|_{\cH_\cY}\right] + \sqrt{\frac{2C}{n}\log(\frac{2}{\delta})}.$$ So
% \begin{align*}
% \sup_f \|\mu_\cY P_n - \mu_\cY f_\sharp  P \|_{\cH_\cY} \lesssim & \inf_{\alpha>0} \left(\alpha + \frac{1}{n} \int_\alpha^{2C} \log(N(\cF_{k_{\cY}} ,\|\cdot\|_\infty,\tau)) \dd \tau \right)^{1/2} \\
% &\qquad\qquad\qquad\qquad\qquad+ \sqrt{\frac{C}{n}}\left(1+\sqrt{\log\frac{2}{\delta}}\right).
% \end{align*}
% \end{remark}

\subsubsection{Proof of Lemma \ref{lem:distortion_convergence}}
Recalling the definition of $\Delta_1$ from Definition \ref{def:metricD}, to prove the lemma we separately bound the terms $\sup_{f} \big|\Delta^{(1)}_{\cX}(f;P_n)-\Delta^{(1)}_{\cY}(f;P) \big|$, $\sup_{g} \big|\Delta^{(1)}_{\cY}(g;Q_m)-\Delta^{(1)}_{\cY}(g;Q) \big|$, and $\sup_{f,g} \big|\Delta^{(1)}_{\cX,\cY}(f,g;P_n,Q_m)-\Delta^{(1)}_{\cX,\cY}(f,g;P,Q) \big|$. For the first, we have
\begin{align*}
    \EE&\left[\sup_{f} \Big|\Delta^{(1)}_{\cX}(f;P_n)-\Delta^{(1)}_{\cX}(f;P) \Big| \right] \\
    &\leq \frac{2K}{n} + \EE\left[ \sup_{f} \left|\frac{1}{n(n-1)}\sum_{i\neq j}^n\Big| d_\cX(X_i,X_j)-d_\cY\big(f(X_i),f(X_j)\big)\Big|-\Delta(f;P) \right| \right],
\end{align*}
which follows because the summands with $i=j$ are all 0. Also recall that $K$ is the bound of diameters of $\cX,\cY$.
% \zg{introduce new notation with $:=$ and not $=$. I made the change below; please apply throughout.}
Denote 
\[h_f(x,x'):= \big| d_\cX(x,x')-d_\cY\big(f(x),f(x')\big)\big| - \Delta^{(1)}_{\cX}(f;P),\]
and note that it is a bounded, symmetric, and centered (w.r.t. $P$) kernel. By Theorem 3.5.3 in \citep{de2012decoupling}, we have
% \zg{Is this just a consequence of the Theorem 3.5.3? If so, not need to add `by symmetrization' as it will be part of the proof of that theorem}
\begin{align*}
    \EE\left[ \sup_{f} \bigg|\frac{1}{n(n-1)}\sum_{i\neq j}^n h_f(X_i,X_j) \bigg| \right]\lesssim \EE\left[ \sup_f\bigg|\frac{1}{n(n-1)}\sum_{i\neq j}^n\epsilon_i h_f(X_i,X_j)\bigg| \right]
\end{align*}
% \zg{again, absolute constants?}
where $\{\epsilon_i\}_{i\in\NN}$ is a sequence of i.i.d. Rademacher variables, independent of the samples $X_1,\ldots,X_n$. %Conditioning on points $\{X_i\}_{i=1}^n$, 
% \zg{Explain what is your next step... `To control the RHS above, we shall apply Dudley's entropy integral bound [Theorem... CITE]. To that end we need a handle on the covering number of the function class... w.r.t. the sup-norm. Specifically, we next bound this covering number in terms of that of the original class $\cF$. Define (THE DEFINITION HAS NOTHING TO DO WITH CONDITIONING ON THE SAMPLE)... Observe that, conditioned on the samples $X_1,\ldots,X_n$, $A_f$ is sub-Gaussian... Further note that... Hence we see that COVERING NUMBER BOUND NUMBER, and this [THEOREM DUDLEY CITE] implies... }
To control the RHS above, we shall apply Dudley's entropy integral bound to sub-Gaussian processes (see, for example, Theorem 5.22 from \citet{wainwright2019high}). To that end we need a handle on the covering number of the function class $\{h_f:f\in\cF\}$ w.r.t. the sup-norm. Specifically, we next bound this covering number in terms of that of the original class $\cF$. Define \[A_f:= \frac{1}{\sqrt{n}(n-1)}\sum_{i\neq j}^n \epsilon_i h_f(X_i,X_j),\]
Observe that, conditioned on the samples $X_1,\ldots,X_n$, $A_f$ is sub-Gaussian in $L^2(\tilde{P})$ norm where $\tilde{P} := \frac{1}{n(n-1)} \sum_{i\neq j}^n\delta_{X_i,X_j}$, since for any function $h$,
\begin{align*}
    \sum_{i=1}^n \left(\frac{1}{\sqrt{n}(n-1)}\sum_{j\neq i}^n h(X_i,X_j)\right)^2 &= \sum_{i=1}^n \frac{1}{n(n-1)^2}\left(\sum_{j\neq i}^n h(X_i,X_j)\right)^2\\
    &\leq  \frac{1}{n(n-1)}\sum_{i\neq j}^n h(X_i,X_j)^2.
\end{align*} 
Also 
\begin{align*}
    |A_f - A_{f'}|&\leq \frac{1}{n-1} \left( \sum_{i=1}^n \left(\sum_{j\neq i}^n (h_f-h_{f'})(X_i,X_j)\right)^2  \right)^{1/2} \\
    &\leq \frac{1}{\sqrt{n-1}} \Bigg(    \sum_{i\neq j}^n  (h_f-h_{f'})(X_i,X_j)^2  \Bigg)^{1/2} \\
    &= \sqrt{n}\|h_f-h_{f'}\|_{L^2(\tilde{P})}.
\end{align*}
Further note that $\|h_f-h_{f'}\|_{\infty}\leq 4d_\cF(f,f')$, hence we see that the covering number of $\{h_f:\cF\}$ is bounded by that of $\cF$ in $d_\cF$: $N\big(\{h_f:f\in\cF\},\|\cdot\|_\infty,\tau\big)\leq N(\cF,d_\cF,\tau/4)$. By Dudley's entropy integral bound we have
\begin{align*}
    &\EE\left[ \sup_f\bigg|\frac{1}{n(n-1)}\sum_{i\neq j}^n\epsilon_i h_f(X_i,X_j)\bigg| \middle| X_1,\ldots,X_n \right] \\
    &\qquad\qquad= \EE\left[ \sup_f\bigg|\frac{A_f}{\sqrt{n}}\bigg| \middle| X_1,\ldots,X_n \right]\\
    &\qquad\qquad\lesssim  \inf_{\alpha>0} \left(\alpha + \frac{1}{\sqrt{n}}\int_\alpha^{4K} \sqrt{\log\big(N\big(\{h_f:f\in\cF\},\|\cdot\|_\infty,\tau\big)\big)} \dd \tau \right).
\end{align*}
So
\begin{align*}
    &\EE\left[ \sup_{f} \bigg|\frac{1}{n(n-1)}\sum_{i\neq j}^n h_f(X_i,X_j) \bigg| \right]\\
    &\qquad\qquad\lesssim \EE\left[ \sup_f\bigg|\frac{1}{n(n-1)}\sum_{i\neq j}^n\epsilon_i h_f(X_i,X_j)\bigg| \right]\\
    &\qquad\qquad\lesssim  \inf_{\alpha>0} \left(\alpha + \frac{1}{\sqrt{n}}\int_\alpha^{4K} \sqrt{\log\big(N\big(\{h_f:f\in\cF\},\|\cdot\|_\infty,\tau\big)\big)} \dd \tau \right)\\
    &\qquad\qquad\lesssim  \inf_{\alpha>0} \left(\alpha + \frac{1}{\sqrt{n}}\int_\alpha^K \sqrt{\log\big(N(\cF,d_\cF,\tau)\big)} \dd \tau \right).
\end{align*}
% The same holds for $\Delta^{(1)}_{\cY}(g;Q_m)$ \zg{it is not the same, but the argument of course is similar. Instead say: `By a similar argument, we also have' and then write out the appropriate bound for the $\Delta_\cY$ term.}.
By a similar argument, we also have
\begin{align*}
    \EE&\left[\sup_{g} \Big|\Delta^{(1)}_{\cY}(g;Q_m)-\Delta^{(1)}_{\cY}(g;Q) \Big| \right] \\
    &\qquad\qquad\lesssim \frac{2K}{m} +\inf_{\alpha>0} \left(\alpha + \frac{1}{\sqrt{m}}\int_\alpha^K \sqrt{\log\big(N(\cG,d_\cG,\tau)\big)} \dd \tau \right).
\end{align*}

For the third term, we decouple the samples into several stacks that have the same distribution, and within each stack the points are i.i.d. samples from $P\otimes Q$. This allows us to apply again the entropy integral bound to each stack of samples.
% \zg{say why this is needed. Generally, apply comments from above.}
Suppose $n\leq m$, and consider the samples sets $\{(X_i,Y_{i+j-1})\}_{i=1}^n$, for $j=1,\ldots,m$, where the index of $Y_{i+j-1}$ is modulo $m$. Denote $Z_i^j:=(X_i,Y_{i+j-1})$, for $i=1,\ldots,n$ and $j=1,\ldots,m$, and further set $Z^j:=\{Z_i^j\}_{i=1}^n$. Note that for each $j=1,\ldots,m$, the $Z^j$ comprises $n$ i.i.d. samples from $P\otimes Q$. Denoting 
\[h_{f,g}(x,y):= \big|d_\cX\big(x,g(y)\big)-d_\cY\big(f(x),y\big)\big| - \Delta^{(1)}_{\cX,\cY}(f,g;P,Q),\]
we now have
% \zg{full indexing}
% For each fixed $j=1,\ldots,m$, the samples in $\{(X_i,Y_{i+j-1})\}_{i=1}^n$ are pairwise i.i.d. according to $P\otimes Q$.
% then we have $m$ sets of $n$ iid \zg{i.i.d. not iid.}samples from $P\otimes Q$, with each set $\{(X_i,Y_{i+j-1})\}_{i=1}^n$ and $j=1,\ldots,m$ where index of $Y_{i+j-1}$ is mod $m$. Denote each of the $m$ sets as $Z^j = \{Z_i^j\}_{i=1}^n$, and $h_{f,g}(x,y) = \big|d_\cX\big(x,g(y)\big)-d_\cY\big(f(x),y\big)\big| - \Delta^{(1)}_{\cX,\cY}(f,g;P,Q)$, then
\begin{align*}
    \EE\left[\sup_{f,g} \Big|\Delta^{(1)}_{\cX,\cY}(f,g;P_n,Q_m)-\Delta^{(1)}_{\cX,\cY}(f,g;P,Q) \Big|\right]& = \EE\left[\sup_{f,g} \bigg|\frac{1}{nm}\sum_{i=1}^n\sum_{j=1}^m h_{f,g}(X_i,Y_j) \bigg|\right]\\
    &\leq \EE\left[\sup_{f,g} \bigg|\frac{1}{m}\sum_{j=1}^m \frac{1}{n}\sum_{i=1}^n h_{f,g}(Z_i^j) \bigg|\right]\\
    &\leq \frac{1}{m}\sum_{j=1}^m\EE\left[\sup_{f,g} \bigg| \frac{1}{n}\sum_{i=1}^n h_{f,g}(Z_i^j) \bigg|\right]\\
    &= \EE\left[\sup_{f,g} \bigg| \frac{1}{n}\sum_{i=1}^n h_{f,g}(Z_i^1) \bigg|\right].
\end{align*}
% where the last term is a standard empirical process so we have chaining.\zg{there is no standard empirical process and `we have chaining' in vague. State precisely what is the empirical process (variables, indexing set, etc.) and be more precise in your statement.} 
Notice that up to a factor of $\sqrt{n}$, the quantity within the absolute value is an empirical process of $n$ i.i.d. samples $\{Z_i^1\}_{i=1}^n$ from $P\otimes Q$, that is indexed by function class $\{h_{f,g}:f\in\cF,g\in\cG\}$. Further note that $\|h_{f,g}-h_{f',g'}\|_\infty \leq 2d_\cF(f,f')+2d_\cG(g,g')$, hence the covering number of this function class is bounded as
% \zg{use consistent notation! the class above is defined with a colon, then when you mention it next you use a comma. Please fix this a check throughout.}
$N\big(\{h_{f,g}:f\in\cF,g\in\cG\},\|\cdot\|_\infty,\tau\big)\leq N(\cF,d_\cF,\tau/4)N(\cG,d_\cG,\tau/4)$.
% \zg{Here again, you should mention entropy integral and that like before you want a handle on packing number of the composed class in terms of the original ones.}
Applying the entropy integral bound (see Lemma 2.14.3 in \citet{van1996weak}), we have
\begin{align*}
    &\EE\left[\sup_{f,g} \Big|\Delta^{(1)}_{\cX,\cY}(f,g;P_n,Q_m)-\Delta^{(1)}_{\cX,\cY}(f,g;P,Q) \Big|\right]\\
    &\qquad\qquad\qquad\qquad\leq \EE\left[\sup_{f,g} \bigg| \frac{1}{n}\sum_{i=1}^n h_{f,g}(Z_i^1) \bigg|\right]\\
    &\qquad\qquad\qquad\qquad\lesssim \inf_{\alpha>0} \left(\alpha + \frac{1}{\sqrt{n}}\int_\alpha^{4K} \sqrt{\log\big(N\big(\{h_{f,g}:f\in\cF,g\in\cG\},\|\cdot\|_\infty,\tau\big)\big)} \dd \tau \right)\\
    &\qquad\qquad\qquad\qquad\lesssim \inf_{\alpha>0} \left(\alpha + \frac{1}{\sqrt{n}}\int_\alpha^K \sqrt{\log\big(N(\cF,d_\cF,\tau)\big)+\log\big(N(\cG,d_\cG,\tau)\big)} \dd \tau \right).
\end{align*}
Combining all 3 terms we have
\begin{align*}
    &\EE\left[\sup_{f,g} \Big|\Delta_1(f,g;P_n,Q_m)-\Delta_1(f,g;P,Q) \Big| \right]\\
    &\qquad\qquad\lesssim  \frac{K}{n\wedge m}+ \inf_{\alpha>0} \left(\alpha + \frac{1}{\sqrt{n\wedge m}}\int_\alpha^K \sqrt{\log\big(N(\cF,d_\cF,\tau)\big)+\log\big(N(\cG,d_\cG,\tau)\big)} \dd \tau \right).
\end{align*}
\qed

\subsection{Special Cases of Theorem \ref{thm:convergence}}\label{app:cor}
\begin{corollary}[Special cases]\label{cor:convergence}
Under the same condition of Theorem \ref{thm:convergence}, further suppose that $\cX\subset \RR^{d_x},\cY\subset \RR^{d_y}$ are compact, and $k_\cX,k_\cY$ are $L$-Lipschitz in both slots\footnote{Namely, $|k_\cX(x_1,x_1')-k_\cX(x_2,x_2')|\leq L(d_\cX(x_1,x_2) + d_\cX(x_1',x_2'))$, and similarly for $k_\cY$}.
\begin{enumerate}[leftmargin=*,align=parleft]
    \item For $\cF = \mathsf{Lip}_{L_\cF}(\cX,\cY)$, $\cG=\mathsf{Lip}_{L_\cG}(\cY,\cX)$, and $d_x,d_y>2$, we have 
\begin{align*}
    \EE\left[\left|\mathsf{UD}^{\cF,
    \cG}_{k_\cX,k_\cY}(P\|Q)-\mathsf{UD}^{\cF,
    \cG}_{k_\cX,k_\cY}(P_n\|Q_m)\right|\right]\lesssim_{\lambda_x,\lambda_y,L,L_\cF,L_\cG,C,K}  \left(\frac{1}{n}\right)^{\frac{1}{2d_x}} + \left(\frac{1}{m}\right)^{\frac{1}{2d_y}}.
\end{align*}
    \item Consider parametrized classes $\cF,\cG$. Namely, let $\Theta\subset\RR^{k_1},\Phi\subset\RR^{k_2}$ be compact parameter sets with diameters bounded by $K'$. Take $\cF = \{ f_\theta: \theta\in\Theta \}$, with $d_\cF(f_{\theta_1},f_{\theta_2})\leq  L_\Theta\|\theta_1-\theta_2\|$, for some constant $L_\Theta$ and all $\theta_1,\theta_2\in\Theta$. Suppose analogously for $\cG= \{ g_\phi: \phi\in\Phi \}$ and $L_\Phi$. Then % with compact parameter space $\Phi\subset\RR^{k_2}$ and constant $L_\Phi$, and the diameters of $\Theta,\Phi$ are bounded by $K'$. Then
    \begin{align*}
    \EE\left[\left|\mathsf{UD}^{\cF,\cG}_{k_\cX,k_\cY}(P\|Q)-\mathsf{UD}^{\cF,
    \cG}_{k_\cX,k_\cY}(P_n\|Q_m)\right|\right]\lesssim_{\lambda_x,\lambda_y,L,L_\Theta,L_\Phi,C,K,K',k_1,k_2}  \left(\frac{1}{n\wedge m}\right)^{\frac{1}{2}}.
    \end{align*}
\end{enumerate}
\end{corollary}

The proof of Corollary \ref{cor:convergence} employs the 1-Wasserstein distance, as defined next.
\begin{definition}[1-Wasserstein distance]
The 1-Wasserstein distance between $P,Q\in\cP(\cX)$ is 
\[\mathsf{W}_1(P,Q):=\inf_{\pi\in\Pi(P,Q)}\int_{\cX\times\cX}\|x-y\|\dd \pi(x,y),\]
where $\Pi(P,Q)$ is the set of couplings of $P$ and $Q$.
\end{definition}

We also make use of the following technical lemma.

\begin{lemma}\label{lem:distortion_wassersyein}
Under the assumptions of Corollary \ref{cor:convergence}, and take $\cF = \mathsf{Lip}_{L_\cF}(\cX,\cY)$, $\cG=\mathsf{Lip}_{L_\cG}(\cY,\cX)$, we have
\begin{align*}
    \EE\mspace{-3mu}\left[\sup_{f,g} \Big|\Delta_1(f,g;P_n,Q_m)\mspace{-3mu}-\mspace{-3mu}\Delta_1(f,g;P,Q) \Big| \right]
    &\mspace{-3mu}\lesssim \mspace{-3mu} (L_\cF\mspace{-3mu}+\mspace{-3mu}1)\EE\big[\mathsf{W}_1(P,P_n)\big]\mspace{-3mu}+\mspace{-3mu}(L_\cG\mspace{-3mu}+\mspace{-3mu}1)\EE\big[\mathsf{W}_1(Q,Q_m)\big],\\
    \EE\left[\sup_{g} \|\mu_\cX g_\sharp Q-\mu_\cX g_\sharp Q_m\|_{\cH_\cX}\right]&\mspace{-3mu}\leq\mspace{-3mu}\sqrt{2LL_\cG\EE\big[\mathsf{W}_1(Q,Q_m)\big]},\\
    \EE\left[\sup_{f} \| \mu_\cY f_\sharp  P- \mu_\cY f_\sharp  P_n  \|_{\cH_\cY}\right]&\mspace{-3mu}\leq\mspace{-3mu} \sqrt{2LL_\cF \EE\big[\mathsf{W}_1(P,P_n)\big]}.
\end{align*}
\end{lemma}

\begin{proof}
We first give bounds for $\Delta^{(1)}_{\cX},\Delta^{(1)}_{\cY},\Delta^{(1)}_{\cX,\cY}$ using the coupling trick, i.e. we treat integration of different variables as marginals of a joint distribution, and optimize over all such choice with the chosen marginals. This leads to the 1-Wasserstein distance in the resulting bound. For any $\pi\in\Pi(P_n^{\otimes2},P^{\otimes2})$ we have 
\begin{align*}
    &\Delta^{(1)}_{\cX}(f;P_n)-\Delta^{(1)}_{\cX}(f;P)\\
    &\qquad= \int\big|d_\cX(x,x')-d_\cY\big(f(x),f(x')\big)\big| - \big|d_\cX(w,w')-d_\cY\big(f(w),f(w')\big)\big| \dd \pi(x,x', w,w').
\end{align*}
Consequently, we may infimize over $\pi\in\Pi(P_n^{\otimes2},P^{\otimes2})$ to obtain
\begin{align*}
    \Delta^{(1)}_{\cX}&(f;P_n)-\Delta^{(1)}_{\cX}(f;P) \\
    &=  \inf_\pi\int\big|d_\cX(x,x')-d_\cY\big(f(x),f(x')\big)\big| - \big|d_\cX(w,w')-d_\cY\big(f(w),f(w')\big)\big| \dd \pi(x,x', w,w') \\
    &\leq \inf_\pi\int|d_\cX(x,x')-d_\cX(w,w')| + \big|d_\cY\big(f(x),f(x')\big)-d_\cY\big(f(w),f(w')\big)\big|\dd\pi(x,x', w,w') \\
    &\leq \inf_\pi\int d_\cX(x,w)+d_\cX(x',w') + d_\cY\big(f(x),f(w)\big)+d_\cY\big(f(x'),f(w')\big)\dd\pi(x,x', w,w') \\
    &\leq \inf_\pi\int d_\cX(x,w)+d_\cX(x',w') + L_\cF d_\cX(x,w)+L_\cF d_\cX(x',w')\dd\pi(x,x', w,w')\\
    &= 2(L_\cF+1)\mathsf{W}_1(P,P_n),
\end{align*}
% \zg{you need to punctuate equations like they are part of the text. For instance, the above equation was missing a comma at the end. Please check equation punctuation throughout. }
where $\mathsf{W}_1$ is the 1-Wasserstein distance. The last line is because the minimum is achieved by $\pi^{\otimes2}$ where $\pi$ is the 1-Wasserstein optimal coupling between $P_n$ and $P$, hence \[\EE\big[\big|\Delta^{(1)}_{\cX}(f;P_n)-\Delta^{(1)}_{\cX}(f;P)\big|\big]\leq 2(L_\cF+1)\EE[\mathsf{W}_1(P,P_n)].\]
A similar derivation applies to $\Delta^{(1)}_{\cY}$. For $\Delta^{(1)}_{\cX,\cY}$, abbreviate the coupling set notation as $\Pi_{n,m}:=\Pi(P_n\otimes Q_m, P\otimes Q)$ and consider:% \zz{this is overflowing but not sure what to do}:
\begin{align*}
    &\Delta^{(1)}_{\cX,\cY}(f,g;P_n,Q_m)-\Delta^{(1)}_{\cX,\cY}(f,g;P,Q)\\
    &= \inf_{\pi\in\Pi_{n,m}}\int\big|d_\cX\big(x,g(y)\big)-d_\cY\big(f(x),y\big)\big| - \big|d_\cX\big(x',g(y')\big)-d_\cY\big(f(x'),g(y')\big)\big|\dd\pi(x,y,x'y') \\
    &\leq \inf_{\pi\in\Pi_{n,m}}\int d_\cX(x,x') + d_\cY(y,y') + d_\cX\big(g(y),g(y')\big)+d_\cY\big(f(x),f(x')\big)\dd\pi(x,y,x'y') \\
    &\leq \inf_{\pi\in\Pi_{n,m}}\int d_\cX(x,x') + d_\cY(y,y') + L_\cG d_\cY(y,y')+L_\cF d_\cX(x,x')\dd\pi(x,y,x'y')\\
    &=(L_\cF+1)\mathsf{W}_1(P,P_n)+
    (L_\cG+1)\mathsf{W}_1(Q,Q_m).
\end{align*}
So we have
\begin{align*}
    \EE\left[\sup_{f,g} \Big|\Delta_1(f,g;P_n,Q_m)-\Delta_1(f,g;P,Q) \Big| \right] \lesssim (L_\cF+1)\EE[\mathsf{W}_1(P,P_n)]+(L_\cG+1)\EE[\mathsf{W}_1(Q,Q_m)].
\end{align*}
For the MMD terms we use the same method:
\begin{align*}
    \|\mu_\cX g_\sharp Q-\mu_\cX g_\sharp Q_m\|_{\cH_\cX}^2 &= \int k_\cX(x,x')\dd(g_\sharp Q-g_\sharp Q_m)(x)\dd(g_\sharp Q-g_\sharp Q_m)(x')\\
    & = \int k_\cX\big(g(y),g(y')\big)\dd Q(y)\dd Q(y') \\
    &\qquad+ \int k_\cX\big(g(w),g(w')\big)\dd Q_m(w)\dd Q_m(w')\\
    &\qquad\qquad - \int 2k_\cX\big(g(y),g(w)\big)\dd Q(y)\dd Q_m(w)\\
    &\leq \inf_{\pi\in\Pi(Q,Q_m)}\int 2L d_\cX\big(g(y),g(w)\big) \dd\pi(y,w)\\
    &\leq 2LL_\cG\inf_{\pi\in\Pi(Q,Q_m)}\int d_\cY(y,w)\dd\pi(y,w)\\
    & = 2LL_\cG \mathsf{W}_1(Q,Q_m).
\end{align*}
Similarly we have $\sup_{f} \| \mu_\cY f_\sharp  P- \mu_\cY f_\sharp  P_n  \|_{\cH_\cY}^2\leq 2LL_\cF\mathsf{W}_1(P,P_n)$, which concludes the proof.
\end{proof}

\begin{proof}[Proof of Corollary \ref{cor:convergence}]
We still adopt the same decomposition in proof of Theorem \ref{thm:convergence}, but bound each term explicitly. We first prove part 1. 
% Following Lemma \ref{lem:mmd_convergence}, we first bound the entropy integral for the MMD terms. Clearly $\|k_\cY(f_1(\cdot),f_1(\cdot)) -k_\cY(f_2(\cdot),f_2(\cdot)) \|_\infty\leq 2Ld_\cF(f_1,f_2)$, so $N(\cF_{k_{\cY}} ,\|\cdot\|_\infty,\tau)\leq N\left(\cF,d_\cF,\frac{\tau}{2L}\right)$, $N(\cG_{k_{\cX}} ,\|\cdot\|_\infty,\tau)\leq N\left(\cG,d_\cG,\frac{\tau}{2L}\right)$. We have $\log(N(\cF,\|\cdot\|_\infty,\tau))=O\left(d_y\left(\frac{L_\cF}{\tau}\right)^{d_x}\right)$, $\log(N(\cG,\|\cdot\|_\infty,\tau))=O\left(d_x\left(\frac{L_\cG}{\tau}\right)^{d_y}\right)$, thus
% \begin{align*}
%     \inf_{\alpha>0} \left(\alpha + \frac{1}{n} \int_\alpha^D \log\big(N(\cF_{k_{\cY}} ,\|\cdot\|_\infty,\tau)\big) \dd \tau \right)\lesssim 2LL_\cF\left(\frac{d_y}{n}\right)^{1/d_x}.
% \end{align*}
% Similar for $\cG_{k_{\cX}} $.
For simplicity we apply Lemma \ref{lem:distortion_wassersyein} instead of the general claim in Theorem \ref{thm:convergence}. Combined with the fact that $\EE[\mathsf{W}_1(P,P_n)]\lesssim n^{-1/d_x}$ for $d_x>2$ (see, e.g.,  \citealt{weed2019sharp}), we have the overall rate of convergence is
% \zg{where writing inline exponents avoid fractions, i.e., write $n$, $n^{-1/2}\log(1+n)$, and $n^{-1/d}$ instead of $\frac{1}{n}$, $\frac{log(1+n)}{\sqrt{n}}$. and $\left(\frac{1}{n}\right)^{1/d}$. I found a few occurrences here and fixed them (this was also in the main text, which I fixed as well). Please fine ALL remaining occurrences and fix them too. }
\begin{align*}
    \EE&\left[\left|\mathsf{UD}^{\cF,     \cG}_{k_\cX,k_\cY}(P\|Q)-\mathsf{UD}^{\cF,     \cG}_{k_\cX,k_\cY}(P_n\|Q_m)\right|\right]\\
    &\lesssim  \lambda_y\sqrt{LL_\cF}\left(\frac{1}{n}\right)^{\frac{1}{2d_x}} + \lambda_x\sqrt{LL_\cG}\left(\frac{1}{m}\right)^{\frac{1}{2d_y}}+(L_\cF+1)\left(\frac{1}{n}\right)^\frac{1}{d_x}\\
    &\qquad\qquad\qquad\qquad\qquad+ (L_\cG+1)\left(\frac{1}{m}\right)^\frac{1}{d_y} +  \lambda_x\sqrt{\frac{C}{n}} +  \lambda_y\sqrt{\frac{C}{m}}.
\end{align*}
For part 2, notice that the entropy integral is finite when taking $\alpha=0$. Following Theorem \ref{thm:convergence}, we know that the rate in terms of $m,n$ is bounded by $(n\wedge m)^{-1/2}$, with a multiplicative constant depending on $\lambda_x,\lambda_y,L,L_\Theta,L_\Phi,C,K,K',k_1,k_2$.
\end{proof}
\begin{remark}
For part one in Corollary \ref{cor:convergence}, recall that \citealt{weed2019sharp} claims
\[\EE[\mathsf{W}_1(P,P_n)] \lesssim\begin{cases} n^{-1/2}, & d_x=1\\
n^{-1/2}\log(1+n), & d_x=2
\end{cases}.\]
%when $d_x=2$, $\EE[\mathsf{W}_1(P,P_n)] \lesssim \log(1+n)n^{-1/2}$, and when $d_x=1$, $\EE[\mathsf{W}_1(P,P_n)] \lesssim n^{-1/2}$. 
One can easily adapt the bound to the case where $d_x$ or $d_y$ is no larger than 2. 
The bounds obtained in Corollary \ref{cor:convergence} for Lipschitz classes, although they use different proof techniques, they essentially lead to similar rates in $m,n$  to the ones obtained by simply applying the  general bound in Theorem~\ref{thm:convergence}.
\end{remark}

\subsection{Continuity of the Functional $\cL$}\label{app:continuity}
Recall that $\cL_{P,Q}(f,g): = \lambda_x\mathsf{MMD}_{\cX}( P, g_{\sharp }Q)+\lambda_y\mathsf{MMD}_{\cY}(f_\sharp  P, Q) +  \Delta_1(f,g;P,Q)$.
\begin{proposition}
Suppose $\cF,\cG$ are Lipschitz subclasses with Lipschitz constants $L_\cF,L_\cG$ respectively, and kernels $k_\cX,k_\cY$ are Lipschitz on both slots with constant $L$. We have the continuity of $\cL$ in all of it's arguments:
\begin{align*}
    \big|\cL_{P,Q}(f,g)-\cL_{P',Q'}(f',g')\big|&\leq \lambda_x\mathsf{MMD}_{\cX}( P, P') +\lambda_y\mathsf{MMD}_{\cY}(Q,Q')\\
    &\qquad+\lambda_x\mathsf{MMD}_{\cX}( g_{\sharp }Q, g_{\sharp }Q')+\lambda_y\mathsf{MMD}_{\cY}( f_{\sharp }P, f_{\sharp }P')\\ &\qquad\qquad+3(1+L_\cF)\mathsf{W}_1(P,P')+3(1+L_\cG)\mathsf{W}_1(Q,Q') \\
    &\qquad\qquad\qquad
    + (2L\lambda_y+3)d_\cF(f,f') + (2L\lambda_x+3) d_\cG(g,g').
\end{align*} 
\label{prop:stability}
\end{proposition}

\begin{proof}[Proof of Proposition \ref{prop:stability}]
We prove separately for the MMD and $\Delta$.
\begin{align*}
    \big|\mathsf{MMD}_{\cX}( P, g_{\sharp }Q) -& \mathsf{MMD}_{\cX}( P', g'_{\sharp }Q')\big| \\
    &= \big|\|\mu_\cX P - \mu_\cX  g_{\sharp }Q\|_{\cH_\cX}-\|\mu_\cX  P' - \mu_\cX g'_{\sharp }Q'\|_{\cH_\cX}\big| \\
    &\leq \mathsf{MMD}_{\cX}( P, P') + \| \mu_\cX g'_{\sharp }Q'- \mu_\cX g_{\sharp }Q\|_{\cH_\cX}\\
    &\leq \mathsf{MMD}_{\cX}( P, P') + 2Ld_\cG(g,g') +   \mathsf{MMD}_{\cX}( g_{\sharp }Q, g_{\sharp }Q') .
\end{align*}
Also for any coupling $\pi(x,x',w,w')$ of $P\otimes P$ and $P'\otimes P'$
\begin{align*}
    &\big|\Delta^{(1)}_{\cX}(f;P) - \Delta^{(1)}_{\cX}(f';P')\big|\\
    &\qquad\leq \int d_\cX(x,w)+d_\cX(x',w')+d_\cY\big(f(x),f'(w)\big) + d_\cY\big(f(x'),f'(w')\big)\dd \pi(x,x',w,w')\\
    &\qquad\leq 2d_\cF(f,f') + \int (1+L_\cF)d_\cX(x,w)+(1+L_\cF)d_\cX(x',w')\dd \pi(x,x',w,w').
\end{align*}
And similarly for any coupling $\eta(x,y,x',y')$ of $P\otimes Q$ and $P'\otimes Q'$
\begin{align*}
    &|\Delta^{(1)}_{\cX,\cY}(f,g;P,Q)- \Delta^{(1)}_{\cX,\cY}(f',g';P',Q')|\\
    &\qquad\leq \int d_\cX(x,x')+d_\cY(y,y')+d_\cX\big(g(y),g'(y')\big) + d_\cY\big(f(x),f'(x')\big)\dd \eta(x,y,x',y')\\
    &\qquad\leq d_\cF(f,f') + d_\cG(g,g') + \int (1+L_\cF)d(x,x')+(1+L_\cG)d(y,y')\dd\eta(x,y,x',y').
\end{align*}
Apply the same coupling trick as in the previous section, we have
\begin{align*}
    &\big|\Delta_1(f,g;P,Q)- \Delta_1(f',g';P',Q')\big|\\
    &\qquad\qquad\leq  \big|\Delta^{(1)}_{\cX}(f;P) - \Delta^{(1)}_{\cX}(f';P')\big| + \big|\Delta^{(1)}_{\cY}(g;Q) - \Delta^{(1)}_{\cY}(g';Q')\big|\\
    &\qquad\qquad\qquad\qquad\qquad+ \big|\Delta^{(1)}_{\cX,\cY}(f,g;P,Q)- \Delta^{(1)}_{\cX,\cY}(f',g';P',Q')\big| \\
    &\qquad\qquad\leq  3d_\cF(f,f') + 3d_\cG(g,g') + 3(1+L_\cF)\mathsf{W}_1(P,P')+3(1+L_\cG)\mathsf{W}_1(Q,Q'),
\end{align*}
which concludes the proof.
\end{proof}

\newpage
\section{Experiments}\label{app:exp}
\subsection{Algorithm}\label{app:algo}
\begin{algorithm}[htbp]
\caption{Cycle consistent Monge map computation}\label{alg:cyclegan}
\begin{algorithmic}
\Require $X\in\RR^{n\times d_\cX}$, $Y\in\RR^{m\times d_\cY}$. $\alpha$, the learning rate. $b$, the batchsize.
\While{$\theta$ has not converged}
\State Sample $\{x_i\}_{i=1}^b$ a batch from the rows of $X$, forming $P_b$.
\State Sample $\{y_i\}_{i=1}^b$ a batch from the rows of $Y$, forming $Q_b$.
\State $v\gets \nabla_\theta\cL_{P_b,Q_b}(f_\theta,g_w)$
\State $u\gets \nabla_w\cL_{P_b,Q_b}(f_\theta,g_w)$
\State $\theta\gets Adam(v,\theta,\alpha)$
\State $w\gets Adam(u,w,\alpha)$
\EndWhile
\State\Return $(f_\theta,g_w)$.
\end{algorithmic}
\end{algorithm}
% \youssef{Appendix for experiments needs to be cleaned and structured so we refer to clear subsections in it }

\subsection{Additional High Dimensional Experiments on Unaligned Word Embeddings}\label{app:word}
We present here an additional experiment for alignment of word embedding spaces (see \citet{alvarez2018gromov} for experiment details), which demonstrates the applicability of GMMD method to higher dimensional scenarios. Specifically, we consider words from English and French that are embedded into 300 dimensional spaces. We apply our GMMD method to obtain mappings between these two spaces, and obtain correspondence by searching for the nearest neighbor. We verify how well the learned mappings align these spaces by checking how many words in the English-French dictionary are correctly matched. The word embedding data sets are from \cite{bojanowski2016enriching}, and dictionaries are from  \cite{conneau2017word}.

For training we use the 20k most frequent words, and learning rate 0.01. The kernel is a single Gaussian kernel with bandwidth 1, and $\lambda=0.01$. Batchsize is 500, and we train the NNs for 1000 epochs. The NNs are both single linear layer without bias. See Table \ref{tab:word} for comparison with the GW method \citep{alvarez2018gromov} and the MUSE method \citep{conneau2017word}. The MUSE method outperforms GMMD and GW on this task. GMMD matching is not far behind the  GW method.  

Note that the MUSE method uses a linear orthonomal mapping that maps only in one direction as follows:
\[\min_{U , UU^{\top}=I_d} \sup_{f\in\mathcal{F}} \mathbb{E}_{P}f(Ux) -  \mathbb{E}_{Q}f(x).\]
(in fact MUSE uses a GAN objective to learn the witness function $f$ and not an IPM objective as we present it here).
 \citet{memoli2021distance} showed that this form of  the MUSE algorithm is related to the Gromov-Monge distance (Section A.3 in \citet{memoli2021distance}). As \citet{conneau2017word} pointed out, learning the kernel or the discriminator  is  advantageous for the word alignment task. We believe that GMMD will benefit from learning the kernels  similar to MUSE in order to further improve its performance. The min-max formulation of GMMD with learned kernels (Equation \eqref{eq:IPM_UD}) is left for future work in terms of analysis and practical implementations.

\begin{table}[ht!]
\caption{Word matching performance comparison.}
\begin{center}
\begin{tabular}{|c|c|c|}
\hline
&EN to FR&FR to EN\\
\hline
GMMD & 76.1\%& 74.5\%\\
\hline
\begin{tabular}[c]{@{}c@{}}GW ($\epsilon=10^{-4}$)\\\cite{alvarez2018gromov}\end{tabular}&79.3\%&78.3\%\\
\hline
GW ($\epsilon=10^{-5}$)& 81.3\%& 78.9\%\\
\hline
\begin{tabular}[c]{@{}c@{}}MUSE \\ \cite{conneau2017word}\end{tabular}& 82.3\%& 82.1\%\\
\hline
\end{tabular}
\end{center}
\label{tab:word}
\end{table}

\subsection{Comparison to GW and UGW}\label{app:comparison_gw_ugw}
We present full results of the comparison between continuous GMMD mappings and discrete GW Barycentric Mappings. In Figure \ref{fig:heart_mapping_full} we illustrate the results for different $\lambda$ and 4 cases: heart vs. rotated/scaled/embedded heart, and biplanes. For each test case we present both the image of the learned mappings and the cycle consistency of the mappings. Figure \ref{fig:gw_full} contains results for the same test cases, using  barycentric mapping from entropic GW.  The parameter $\epsilon$ in Figure \ref{fig:gw_full}  corresponds to  the entropic regularizer. 

We also provide full tables of the quantitative behavior of GMMD and GW on the test cases. In Table \ref{tab:GMMD_rotate} to \ref{tab:ugw_Embedded}, the marginal MMDs and $\Delta$ for GMMD, GW and UGW are computed across different parameters respectively. For GMMD we use $\lambda=10^{-3}\times2^{\{0,1,\cdots,9\}}$; for GW we use entropic regularizer $\epsilon=5\times10^{\{0,-1,-2,-3,-4\}}$; for UGW we use entropic regularizer $10^{\{-2,-1,0,1\}}$. For GW and UGW we only list results for hyperparameters  that don't fail using POT \cite{flamary2021pot}  and UGW's code of \cite{sejourne2020unbalanced}. 

\begin{figure*}[ht]
\hspace{-3.5mm}
\centering
 \begin{subfigure}[t]{0.18\textwidth}
 \centering
  \includegraphics[width=1.1\linewidth]{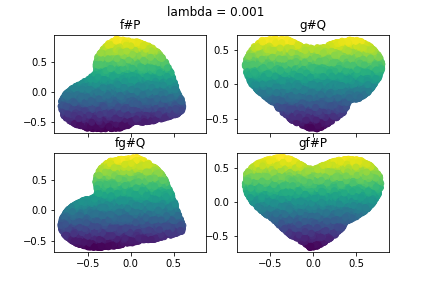}
   \vspace{-2mm}
%   \caption{}
   \vspace{-1mm}
 \end{subfigure}
 \ \hspace{0.3mm}
 \begin{subfigure}[t]{0.18\textwidth}
 \centering
 \includegraphics[width=1.1\linewidth]{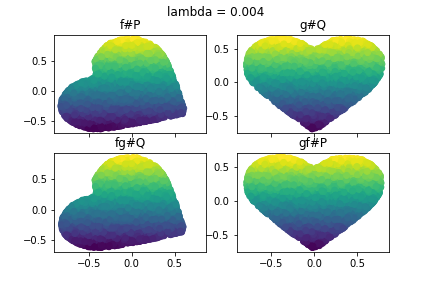}
    \vspace{-2mm}
%   \caption{}
 \end{subfigure}
 \ \hspace{0.3mm}
 \begin{subfigure}[t]{0.18\textwidth}
 \centering
  \includegraphics[width=1.1\linewidth]{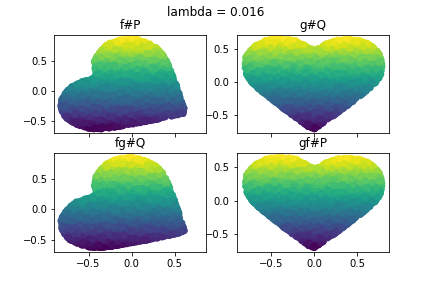}
   \vspace{-2mm}
%   \caption{} 
   \vspace{-1mm}
 \end{subfigure}
  \ \hspace{-0.2 mm}
 \begin{subfigure}[t]{0.18\textwidth}
 \centering
\includegraphics[width=1.1\linewidth]{Aistats2022/figs/rotate_0.064.png}
   \vspace{-2mm}
%   \caption{}
   \vspace{-1mm}
 \end{subfigure}
 \ \hspace{0.3mm}
 \begin{subfigure}[t]{0.18\textwidth}
 \centering
 \includegraphics[width=1.1\linewidth]{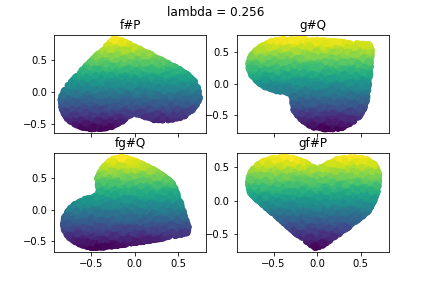}
    \vspace{-2mm}
%   \caption{}
 \end{subfigure}
%   \caption{Heart And Rotated Heart}
%   \label{fig:rotate}
% \vspace{-2mm}
% \end{figure*}
% \begin{figure*}[!t]
% \hspace{-3.5mm}
\centering
 \begin{subfigure}[t]{0.18\textwidth}
 \centering
  \includegraphics[width=1.1\linewidth]{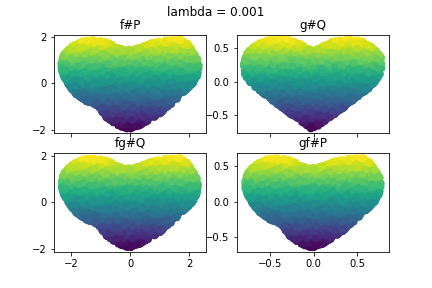}
   \vspace{-2mm}
%   \caption{}
   \vspace{-1mm}
 \end{subfigure}
 \ \hspace{0.3mm}
 \begin{subfigure}[t]{0.18\textwidth}
 \centering
 \includegraphics[width=1.1\linewidth]{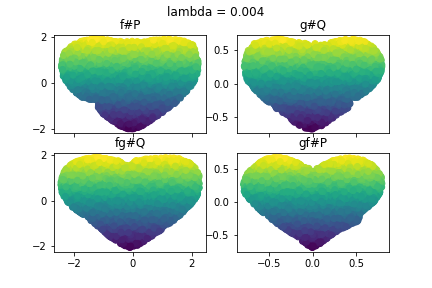}
    \vspace{-2mm}
%   \caption{}
 \end{subfigure}
 \ \hspace{0.3mm}
 \begin{subfigure}[t]{0.18\textwidth}
 \centering
  \includegraphics[width=1.1\linewidth]{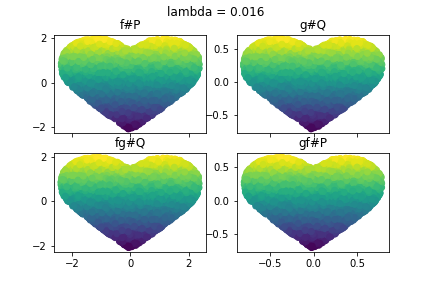}
   \vspace{-2mm}
%   \caption{} 
   \vspace{-1mm}
 \end{subfigure}
  \ \hspace{-0.2 mm}
 \begin{subfigure}[t]{0.18\textwidth}
 \centering
\includegraphics[width=1.1\linewidth]{Aistats2022/figs/scale_0.064.png}
   \vspace{-2mm}
%   \caption{}
   \vspace{-1mm}
 \end{subfigure}
 \ \hspace{0.3mm}
 \begin{subfigure}[t]{0.18\textwidth}
 \centering
 \includegraphics[width=1.1\linewidth]{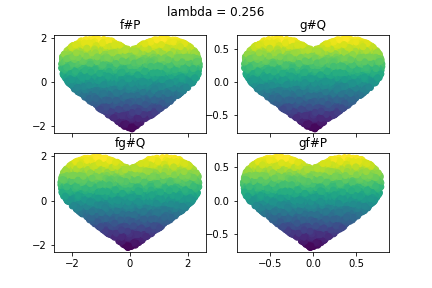}
    \vspace{-2mm}
%   \caption{}
 \end{subfigure}
%   \caption{Heart And Scaled Heart}
%   \label{fig:scale}
% \vspace{-2mm}
% \end{figure*}
% \begin{figure*}[!t]
% \hspace{-3.5mm}
\centering
 \begin{subfigure}[t]{0.18\textwidth}
 \centering
  \includegraphics[width=1.1\linewidth]{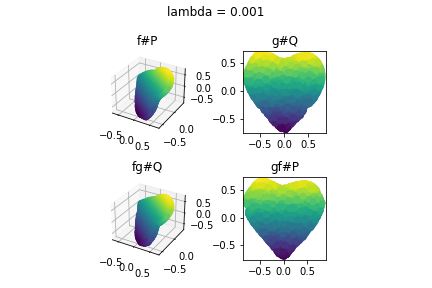}
   \vspace{-2mm}
%   \caption{}
   \vspace{-1mm}
 \end{subfigure}
 \ \hspace{0.3mm}
 \begin{subfigure}[t]{0.18\textwidth}
 \centering
 \includegraphics[width=1.1\linewidth]{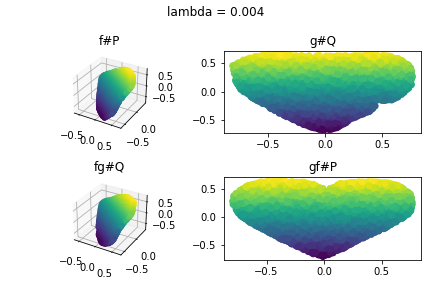}
    \vspace{-2mm}
%   \caption{}
 \end{subfigure}
 \ \hspace{0.3mm}
 \begin{subfigure}[t]{0.18\textwidth}
 \centering
  \includegraphics[width=1.1\linewidth]{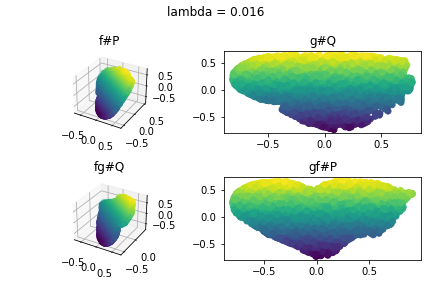}
   \vspace{-2mm}
%   \caption{} 
   \vspace{-1mm}
 \end{subfigure}
  \ \hspace{-0.2 mm}
 \begin{subfigure}[t]{0.18\textwidth}
 \centering
\includegraphics[width=1.1\linewidth]{Aistats2022/figs/embed_0.064.png}
   \vspace{-2mm}
%   \caption{}
   \vspace{-1mm}
 \end{subfigure}
 \ \hspace{0.3mm}
 \begin{subfigure}[t]{0.18\textwidth}
 \centering
 \includegraphics[width=1.1\linewidth]{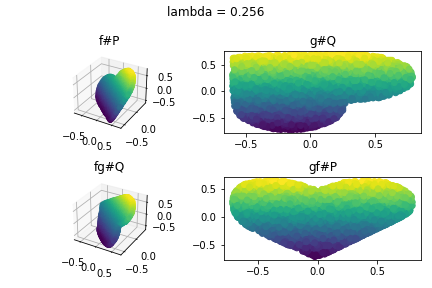}
    \vspace{-2mm}
%   \caption{}
 \end{subfigure}

 \centering
 \begin{subfigure}[t]{0.18\textwidth}
 \centering
  \includegraphics[width=1.1\linewidth]{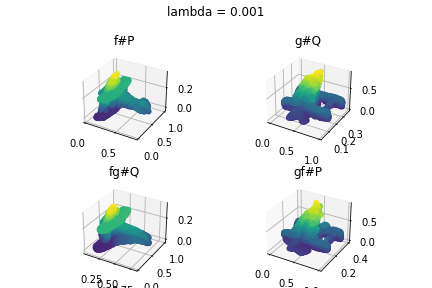}
   \vspace{-2mm}
%   \caption{}
   \vspace{-1mm}
 \end{subfigure}
 \ \hspace{0.3mm}
 \begin{subfigure}[t]{0.18\textwidth}
 \centering
 \includegraphics[width=1.1\linewidth]{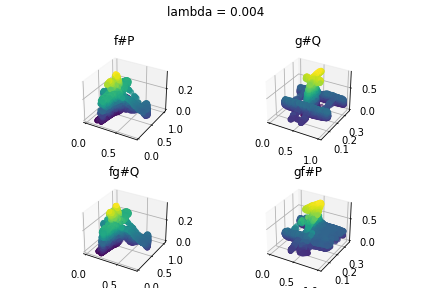}
    \vspace{-2mm}
%   \caption{}
 \end{subfigure}
 \ \hspace{0.3mm}
 \begin{subfigure}[t]{0.18\textwidth}
 \centering
  \includegraphics[width=1.1\linewidth]{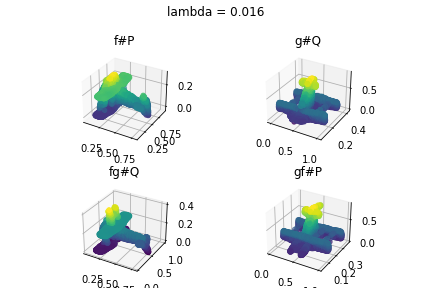}
   \vspace{-2mm}
%   \caption{} 
   \vspace{-1mm}
 \end{subfigure}
  \ \hspace{-0.2 mm}
 \begin{subfigure}[t]{0.18\textwidth}
 \centering
\includegraphics[width=1.1\linewidth]{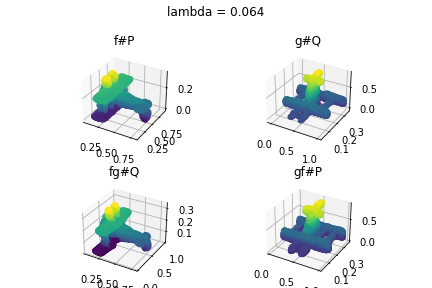}
   \vspace{-2mm}
%   \caption{}
   \vspace{-1mm}
 \end{subfigure}
 \ \hspace{0.3mm}
 \begin{subfigure}[t]{0.18\textwidth}
 \centering
 \includegraphics[width=1.1\linewidth]{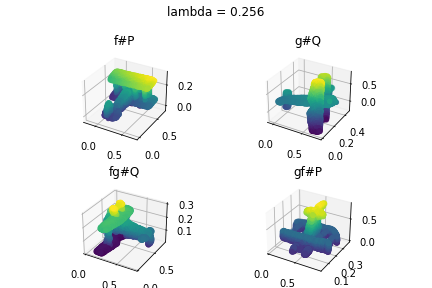}
    \vspace{-2mm}
%   \caption{}
 \end{subfigure}

  \caption{Learned continuous GMMD Mappings and their cycle consistency in shape matching. First row: heart $(P)$ and rotated heart $(Q_b)$. Second row: heart $(P)$ and scaled heart $(Q_c)$. Third row: heart $(P)$ and embedded heart $(Q_d)$. Last row: biplanes.}
  \label{fig:heart_mapping_full}
\vspace{-2mm}
\end{figure*}

\begin{figure*}[ht!]
\hspace{-3.5mm}
\centering
 \begin{subfigure}[t]{0.14\textwidth}
 \centering
  \includegraphics[width=1.1\linewidth]{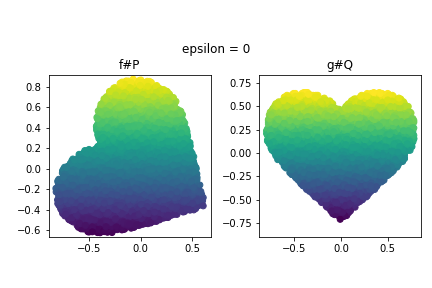}
   \vspace{-2mm}
   \vspace{-1mm}
 \end{subfigure}
 \ \hspace{0.3mm}
 \begin{subfigure}[t]{0.14\textwidth}
 \centering
 \includegraphics[width=1.1\linewidth]{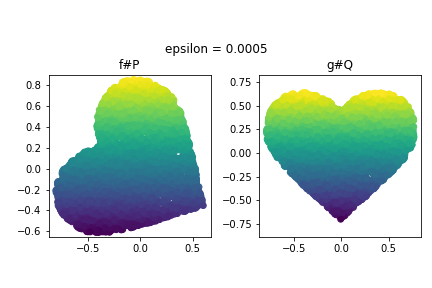}
    \vspace{-2mm}
 \end{subfigure}
 \ \hspace{0.3mm}
 \begin{subfigure}[t]{0.14\textwidth}
 \centering
 \includegraphics[width=1.1\linewidth]{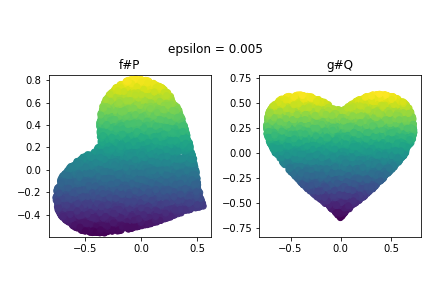}
    \vspace{-2mm}
 \end{subfigure}
 \ \hspace{0.3mm}
 \begin{subfigure}[t]{0.14\textwidth}
 \centering
 \includegraphics[width=1.1\linewidth]{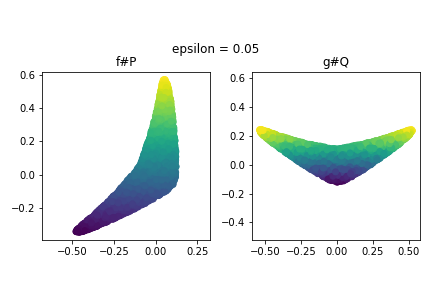}
    \vspace{-2mm}
 \end{subfigure}
 \ \hspace{0.3mm}
 \begin{subfigure}[t]{0.14\textwidth}
 \centering
 \includegraphics[width=1.1\linewidth]{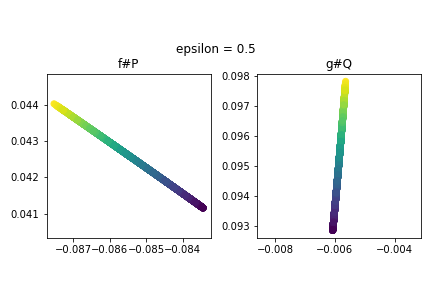}
    \vspace{-2mm}
 \end{subfigure}
 \ \hspace{0.3mm}
 \begin{subfigure}[t]{0.14\textwidth}
 \centering
 \includegraphics[width=1.1\linewidth]{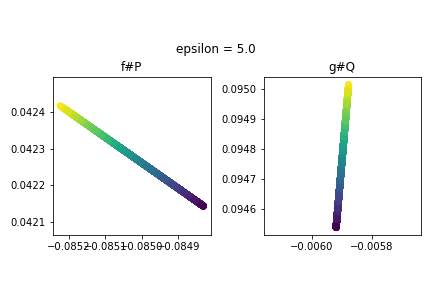}
    \vspace{-2mm}
 \end{subfigure}
 \  

%%%%%%%%%%%%%%%%%%%%%%%%%%
\centering
 \begin{subfigure}[t]{0.14\textwidth}
 \centering
  \includegraphics[width=1.1\linewidth]{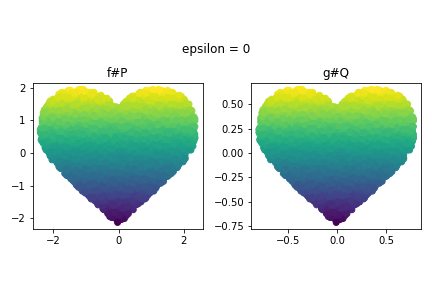}
   \vspace{-2mm}
   \vspace{-1mm}
 \end{subfigure}
 \ \hspace{0.3mm}
 \begin{subfigure}[t]{0.14\textwidth}
 \centering
 \includegraphics[width=1.1\linewidth]{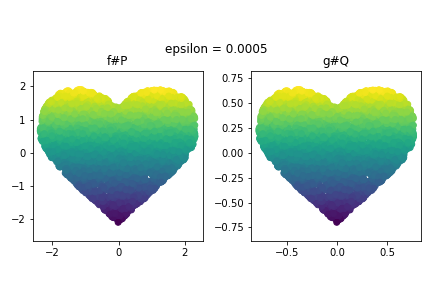}
    \vspace{-2mm}
 \end{subfigure}
 \ \hspace{0.3mm}
 \begin{subfigure}[t]{0.14\textwidth}
 \centering
 \includegraphics[width=1.1\linewidth]{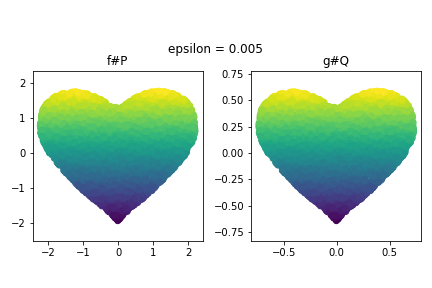}
    \vspace{-2mm}
 \end{subfigure}
 \ \hspace{0.3mm}
 \begin{subfigure}[t]{0.14\textwidth}
 \centering
 \includegraphics[width=1.1\linewidth]{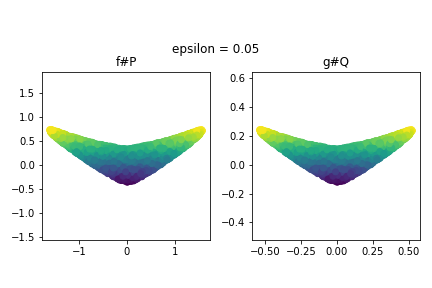}
    \vspace{-2mm}
 \end{subfigure}
 \ \hspace{0.3mm}
 \begin{subfigure}[t]{0.14\textwidth}
 \centering
 \includegraphics[width=1.1\linewidth]{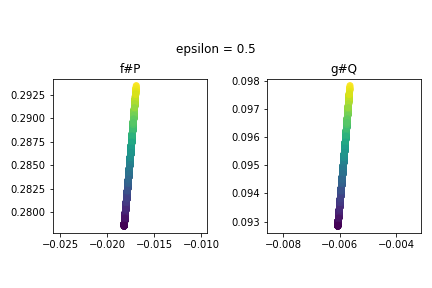}
    \vspace{-2mm}
 \end{subfigure}
 \ \hspace{0.3mm}
 \begin{subfigure}[t]{0.14\textwidth}
 \centering
 \includegraphics[width=1.1\linewidth]{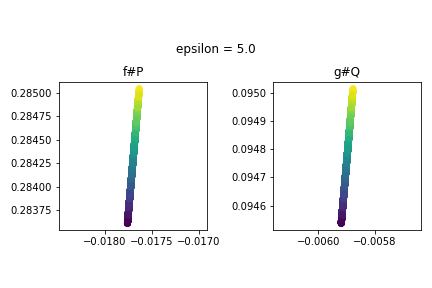}
    \vspace{-2mm}
 \end{subfigure}
 \  

%%%%%%%%%%%%%%%%%%%%%%%%%%%
\centering
 \begin{subfigure}[t]{0.14\textwidth}
 \centering
  \includegraphics[width=1.1\linewidth]{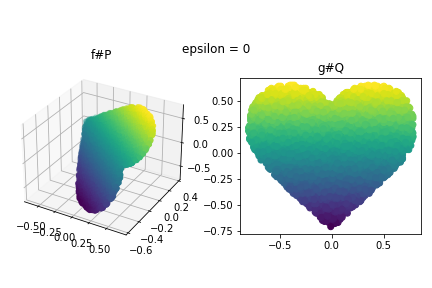}
   \vspace{-2mm}
   \vspace{-1mm}
 \end{subfigure}
 \ \hspace{0.3mm}
 \begin{subfigure}[t]{0.14\textwidth}
 \centering
 \includegraphics[width=1.1\linewidth]{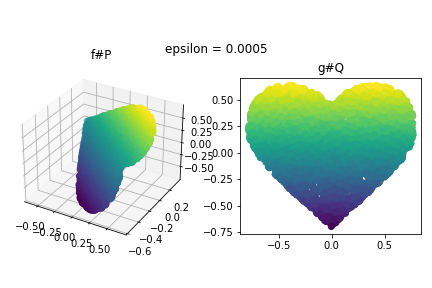}
    \vspace{-2mm}
 \end{subfigure}
 \ \hspace{0.3mm}
 \begin{subfigure}[t]{0.14\textwidth}
 \centering
 \includegraphics[width=1.1\linewidth]{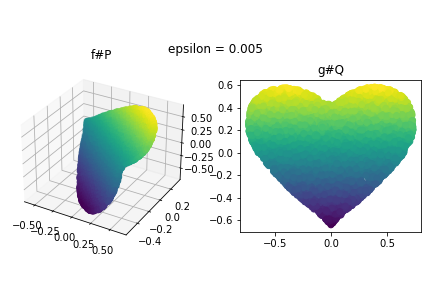}
    \vspace{-2mm}
 \end{subfigure}
 \ \hspace{0.3mm}
 \begin{subfigure}[t]{0.14\textwidth}
 \centering
 \includegraphics[width=1.1\linewidth]{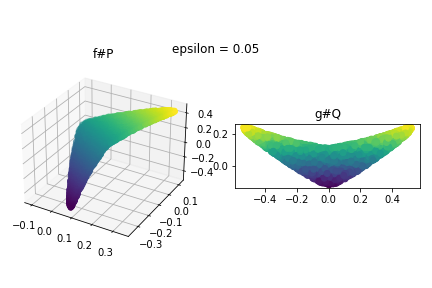}
    \vspace{-2mm}
 \end{subfigure}
 \ \hspace{0.3mm}
 \begin{subfigure}[t]{0.14\textwidth}
 \centering
 \includegraphics[width=1.1\linewidth]{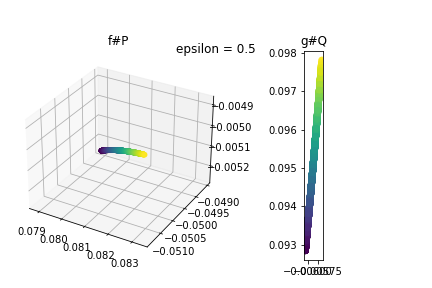}
    \vspace{-2mm}
 \end{subfigure}
 \ \hspace{0.3mm}
 \begin{subfigure}[t]{0.14\textwidth}
 \centering
 \includegraphics[width=1.1\linewidth]{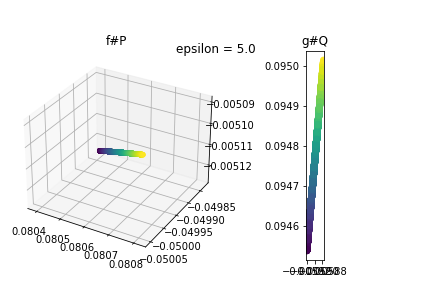}
    \vspace{-2mm}
 \end{subfigure}
 \  
 %%%%%%%%%%%%%%%%%%%%%%%%%%
\centering
 \begin{subfigure}[t]{0.14\textwidth}
 \centering
  \includegraphics[width=1.1\linewidth]{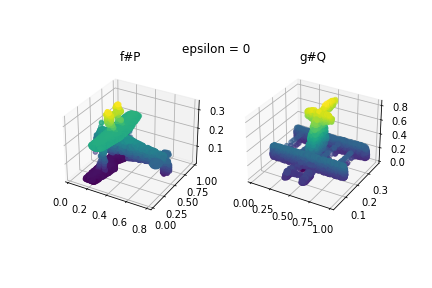}
   \vspace{-2mm}
   \vspace{-1mm}
 \end{subfigure}
 \ \hspace{0.3mm}
 \begin{subfigure}[t]{0.14\textwidth}
 \centering
 \includegraphics[width=1.1\linewidth]{Aistats2022/figs/gw/gw_biplane_0.0005.png}
    \vspace{-2mm}
 \end{subfigure}
 \ \hspace{0.3mm}
 \begin{subfigure}[t]{0.14\textwidth}
 \centering
 \includegraphics[width=1.1\linewidth]{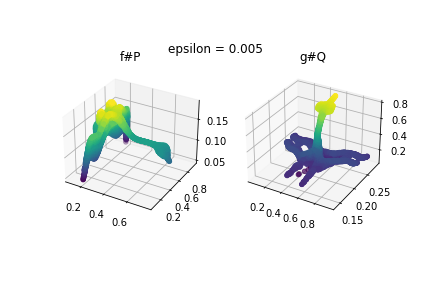}
    \vspace{-2mm}
 \end{subfigure}
 \ \hspace{0.3mm}
 \begin{subfigure}[t]{0.14\textwidth}
 \centering
 \includegraphics[width=1.1\linewidth]{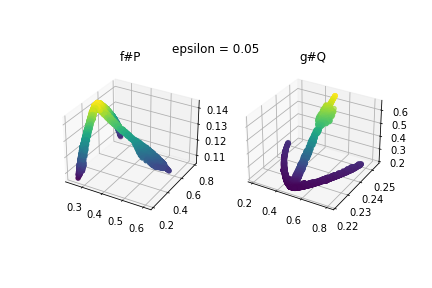}
    \vspace{-2mm}
 \end{subfigure}
 \ \hspace{0.3mm}
 \begin{subfigure}[t]{0.14\textwidth}
 \centering
 \includegraphics[width=1.1\linewidth]{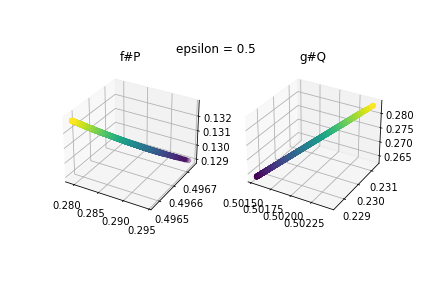}
    \vspace{-2mm}
 \end{subfigure}
 \ \hspace{0.3mm}
 \begin{subfigure}[t]{0.14\textwidth}
 \centering
 \includegraphics[width=1.1\linewidth]{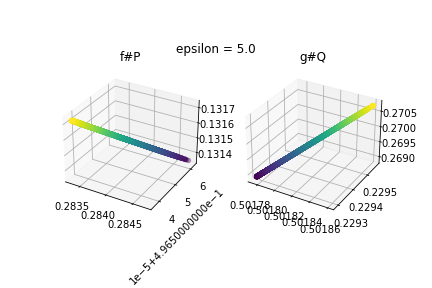}
    \vspace{-2mm}
 \end{subfigure}
 \  

  \caption{GW Barycentric Mappings. First row: heart $(P)$ and rotated heart $(Q_b)$. Second row: heart $(P)$ and scaled heart $(Q_b)$. Third row: heart $(P)$ and embedded heart $(Q_b)$. Last row: biplanes.}
  \label{fig:gw_full}
\vspace{-2mm}
\end{figure*}

\begin{table}[ht!]
\caption{Evaluation of GMMD mappings between heart $(P)$ and rotated heart $(Q_b)$.}
\begin{center}
\begin{tabular}{|l|l|l|l|l|}
\hline
\textbf{$\lambda$}  &\textbf{$\GMMD$}&\textbf{$\MMD_\cX$}&\textbf{$\MMD_\cY$}&\textbf{$\Delta$} \\
\hline
0.512  &1.33   &0.0524   &0.00207    &2.50 \\
\hline
0.256  &0.0794     &0.0294    &0.0294    &0.0801\\
\hline
0.128  &0.00825   &4.94e-4   &4.05e-4  &0.0574  \\
\hline
0.064  &0.00776   &0.00227   &0.00190   &0.0560\\
\hline
0.032  &0.0924    &2.90e-4  &0.00386   &2.76\\
\hline
0.016  &0.0579    &0.00137   &3.33e-4  &3.51\\
\hline
0.008  &0.0136      &0.00188  &0.00181   &1.24\\
\hline
0.004  &8.04e-4  &2.85e-4  &2.55e-4 &0.0661\\
\hline
0.002  &0.00624 &0.00149    &0.00147    &1.64\\
\hline
0.001  &0.00378   &0.00118   &0.00114   &1.47\\
\hline
\end{tabular}
\end{center}
\label{tab:GMMD_rotate}
\end{table}

%\subsubsection{Comparison With $\GW$ And $\UGW$}
%The popular $\GW$ and $\UGW$ methods also gives matching between shapes. We present here the baseline using these methods, demonstrating their qualitative and quantitative performance. It is standard to use entropic regularized versions of these methods, and this has both positive effect on computational efficiency, and negative effect on the information it carries about the underlying distributions, see Fig \ref{fig:gw_hearts} and Table \ref{tab:gw_rotate}. We choose entropic regularizer multiplier $\epsilon\in\{5\times10^{-i}:i=0,\cdots,4\}\cup\{0\}$. To obtain a deterministic mapping, we use the barycentric mapping.

\begin{table}[htbp]
\caption{GW and its induced MMDs and $\Delta$ between heart $(P)$ and rotated heart $(Q_b)$.}
\begin{center}
\begin{tabular}{|l|l|l|l|l|}
\hline
\textbf{$\epsilon$}  &\textbf{$\GW$}&\textbf{$\MMD_\cX$}&\textbf{$\MMD_\cY$}&\textbf{$\Delta$} \\
\hline
% 0       &0.00113 &3.55e-15  &1.78e-15   &0.803\\
% \hline
0.0005  &0.00134   &0.00420 &0.00299    &0.696\\
\hline
0.005   &0.00660  &0.127      &0.116      &1.73\\
\hline
0.05    &0.0424   &0.615 &0.613      &6.69\\
\hline
0.5     &0.0686    &3.99       &4.12      &22.9\\
\hline
5       &0.0699    &4.86       &4.89      &26.2\\
\hline
\end{tabular}
\end{center}
\label{tab:gw_rotate}
\end{table}

\begin{table}[htbp]
\caption{UGW and its induced MMDs and $\Delta$ between heart $(P)$ and rotated heart $(Q_b)$.}
\begin{center}
\begin{tabular}{|l|l|l|l|l|}
\hline
\textbf{$\epsilon$}  &\textbf{$\mathsf{UGW}$}&\textbf{$\MMD_\cX$}&\textbf{$\MMD_\cY$}&\textbf{$\Delta$} \\
\hline
    10 &    0.277856  &5.00562  &5.00562 & 25.8038 \\
\hline
     1  &   0.199544  &4.96044  &4.96038 & 25.6224\\
\hline
     0.1 &  0.189629  &4.57678  &4.57691 & 24.0855 \\
\hline
     0.01 & 0.178746 & 3.34716  &3.34713 & 19.1421 \\
\hline
\end{tabular}
\end{center}
\label{tab:ugw_Rotated}
\end{table}

\begin{table}[htbp]
\caption{Evaluation of GMMD mappings between biplanes.}
\begin{center}
\begin{tabular}{|l|l|l|l|l|}
\hline
\textbf{$\lambda$}  &\textbf{$\GMMD$}&\textbf{$\MMD_\cX$}&\textbf{$\MMD_\cY$}&\textbf{$\Delta$} \\
\hline
   0.512  &0.315   &0.124    &0.124     &0.131 \\
\hline
   0.256  &0.284  &0.125  &0.125    &0.134\\
\hline
   0.128  &0.0270  &0.00704  &0.00669  &0.104 \\
\hline
   0.064  &0.0166  &0.00594  &0.00619  &0.0691\\
\hline
   0.032  &0.127   &0.00891  &0.00947  &3.40 \\
\hline
   0.016  &0.0267  &0.00892  &0.00812  &0.602 \\
\hline
   0.008  &0.0357  &0.00637  &0.00548  &2.98 \\
\hline
   0.004  &0.0243  &0.00305  &0.00532  &3.98  \\
\hline
   0.002  &0.0115  &0.00343  &0.00252  &2.79  \\
\hline
   0.001  &0.0116  &0.00397  &0.00409  &3.58\\
\hline
\end{tabular}
\end{center}
\label{tab:GMMD_Biplanes}
\end{table}

\begin{table}[htbp]
\caption{GW and its induced MMDs and $\Delta$ between biplanes.}
\begin{center}
\begin{tabular}{|l|l|l|l|l|}
\hline
\textbf{$\epsilon$}  &\textbf{$\GW$}&\textbf{$\MMD_\cX$}&\textbf{$\MMD_\cY$}&\textbf{$\Delta$} \\
\hline
   5       &0.0699    &4.86       &4.89      &26.2\\
\hline
   0.5     &0.0686    &3.99       &4.12      &22.9\\
\hline
   0.05    &0.0424    &0.615      &0.613      &6.68\\
\hline
   0.005   &0.00660   &0.127      &0.116      &1.73\\
\hline
   0.0005  &0.00134   &0.00420    &0.00299    &0.696\\
% \hline
%   0       &0.00113 &3.55e-15  &1.78e-15   &0.803\\
\hline
\end{tabular}
\end{center}
\label{tab:gw_Biplanes}
\end{table}

\begin{table}[htbp]
\caption{UGW and its induced MMDs and $\Delta$ between biplanes.}
\begin{center}
\begin{tabular}{|l|l|l|l|l|}
\hline
\textbf{$\epsilon$}  &\textbf{$\mathsf{UGW}$}&\textbf{$\MMD_\cX$}&\textbf{$\MMD_\cY$}&\textbf{$\Delta$} \\
\hline
    10 &    0.277856  &5.00562  &5.00562 & 25.8038 \\
\hline
     1  &   0.199544  &4.96044  &4.96038 & 25.6224\\
\hline
     0.1 &  0.189629  &4.57678  &4.57691 & 24.0855 \\
\hline
     0.01 & 0.178746 & 3.34716  &3.34713 & 19.1421 \\
\hline
\end{tabular}
\end{center}
\label{tab:ugw_Biplanes}
\end{table}

\begin{table}[htbp]
\caption{Evaluation of GMMD mappings between heart $(P)$ and scaled heart $(Q_c)$.}
\begin{center}
\begin{tabular}{|l|l|l|l|l|}
\hline
\textbf{$\lambda$}  &\textbf{$\GMMD$}&\textbf{$\MMD_\cX$}&\textbf{$\MMD_\cY$}&\textbf{$\Delta$} \\
\hline

   0.512  &0.0362   & 0.00145   &0.00150  & 0.0649  \\
\hline
   0.256  &0.0148  & 0.00144   & 0.00147 & 0.0464  \\
\hline
   0.128 & 0.0675  & 0.0247   & 0.0251   & 0.139 \\
\hline
   0.064 & 0.00640 & 0.00139  & 0.00145   & 0.0556 \\
\hline
   0.032 & 0.0961  & 0.000462  & 0.00444   & 2.85 \\
\hline
   0.016 & 0.00201  & 0.000612  & 0.000598  & 0.0500 \\
\hline
   0.008 & 0.0281   & 0.00137   & 0.00307   & 2.96  \\
\hline
   0.004 & 0.00869   & 0.00108   & 0.00215   & 1.36    \\
\hline
   0.002 & 0.0104     & 0.00106    &  0.00128  &  4.03 \\
\hline
   0.001 & 0.00506   & 0.000304   &  0.00159    & 3.17   \\
\hline
\end{tabular}
\end{center}
\label{tab:GMMD_Scaled}
\end{table}

\begin{table}[htbp]
\caption{GW and its induced MMDs and $\Delta$ between heart $(P)$ and scaled heart $(Q_c)$.}
\begin{center}
\begin{tabular}{|l|l|l|l|l|}
\hline
\textbf{$\epsilon$}  &\textbf{$\GW$}&\textbf{$\MMD_\cX$}&\textbf{$\MMD_\cY$}&\textbf{$\Delta$} \\
\hline
   5       &0.0776     &5.00     &5.00      &25.8\\
\hline
   0.5     &0.0770    &4.93      &4.93      &25.5\\
\hline
   0.05    &0.0498    &0.483     &0.483      &5.89\\
\hline
   0.005   &0.00483   &0.00307   &0.00307    &0.378\\
\hline
   0.0005  &0.000470  &0.000227  &0.000227   &0.0833\\
% \hline
%   0       &2.62e-17  &0            &0             &7.80e-16\\
\hline
\end{tabular}
\end{center}
\label{tab:gw_Scaled}
\end{table}

\begin{table}[htbp]
\caption{UGW and its induced MMDs and $\Delta$ between heart $(P)$ and scaled heart $(Q_c)$.}
\begin{center}
\begin{tabular}{|l|l|l|l|l|}
\hline
\textbf{$\epsilon$}  &\textbf{$\mathsf{UGW}$}&\textbf{$\MMD_\cX$}&\textbf{$\MMD_\cY$}&\textbf{$\Delta$} \\
\hline
     10    &2.47  &4.09  &4.84   &23.6\\
\hline
      1    &1.79  &3.02  &4.28   &20.4\\
\hline
      0.1  &1.50  &2.90  &4.22   &20.0 \\
\hline
\end{tabular}
\end{center}
\label{tab:ugw_Scaled}
\end{table}

\begin{table}[htbp]
\caption{Evaluation of GMMD mappings between heart $(P)$ and embedded heart $(Q_d)$.}
\begin{center}
\begin{tabular}{|l|l|l|l|l|}
\hline
\textbf{$\lambda$}  &\textbf{$\GMMD$}&\textbf{$\MMD_\cX$}&\textbf{$\MMD_\cY$}&\textbf{$\Delta$} \\
\hline
0.512 & 0.0532  & 0.00157   & 0.00151  & 0.0979 \\
\hline
0.256 & 0.0625   & 0.0256   & 0.0255   & 0.0446 \\
\hline
0.128 & 0.0146  & 0.00177   & 0.00154  & 0.0881 \\
\hline
0.064 & 0.00612 & 0.00145  & 0.00148 & 0.0500  \\
\hline
0.032 & 0.0942   & 0.00179   & 0.00338  & 2.78    \\
\hline
0.016 & 0.0204  & 0.00222  & 0.00213  & 1.00   \\
\hline
0.008 & 0.0345  & 0.00152  & 0.00194 & 3.88   \\
\hline
0.004 & 0.0186  & 0.00152   & 0.00117  & 3.97   \\
\hline
0.002 & 0.00229 & 0.000984 & 0.000964 & 0.172  \\
\hline
0.001 & 0.00683 & 0.00143  & 0.00131   & 4.09   \\
\hline
\end{tabular}
\end{center}
\label{tab:GMMD_Embedded}
\end{table}

\begin{table}[htbp]
\caption{GW and its induced MMDs and $\Delta$ between heart $(P)$ and embedded heart $(Q_d)$.}
\begin{center}
\begin{tabular}{|l|l|l|l|l|}
\hline
\textbf{$\epsilon$}  &\textbf{$\GW$}&\textbf{$\MMD_\cX$}&\textbf{$\MMD_\cY$}&\textbf{$\Delta$} \\
\hline
5      & 0.0776   & 5.00     & 5.00     & 25.8     \\
\hline
0.5    & 0.0770   & 4.93     & 4.93     & 25.5     \\
\hline
0.05   & 0.0498   & 0.483    & 0.483    & 5.89     \\
\hline
0.005  & 0.00483  & 0.00307  & 0.00307  & 0.378    \\
\hline
0.0005 & 0.000470 & 0.000227 & 0.000227 & 0.0833   \\
% \hline
% 0      & 2.22e-17 & 0           & 0           & 6.68e-16\\
\hline
\end{tabular}
\end{center}
\label{tab:gw_Embedded}
\end{table}

\begin{table}[htbp]
\caption{UGW and its induced MMDs and $\Delta$ between heart $(P)$ and embedded heart $(Q_d)$.}
\begin{center}
\begin{tabular}{|l|l|l|l|l|}
\hline
\textbf{$\epsilon$}  &\textbf{$\mathsf{UGW}$}&\textbf{$\MMD_\cX$}&\textbf{$\MMD_\cY$}&\textbf{$\Delta$} \\
\hline
10   & 0.278 & 5.01 & 5.01 & 25.8 \\
\hline
1    & 0.200 & 4.96 & 4.96 & 25.6 \\
\hline
0.1  & 0.190 & 4.58 & 4.58 & 24.1 \\
\hline
0.01 & 0.179 & 3.35 & 3.35 & 19.1 \\
\hline
\end{tabular}
\end{center}
\label{tab:ugw_Embedded}
\end{table}

\begin{figure*}[htbp]
\hspace{-3.5mm}
\centering
 \begin{subfigure}[t]{0.23\textwidth}
 \centering
  \includegraphics[width=1.1\linewidth]{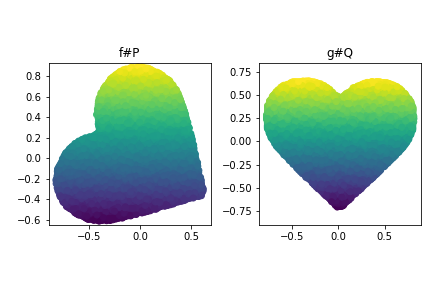}
   \vspace{-2mm}
   \vspace{-1mm}
 \end{subfigure}
 \ \hspace{0.3mm}
 \begin{subfigure}[t]{0.23\textwidth}
 \centering
 \includegraphics[width=1.1\linewidth]{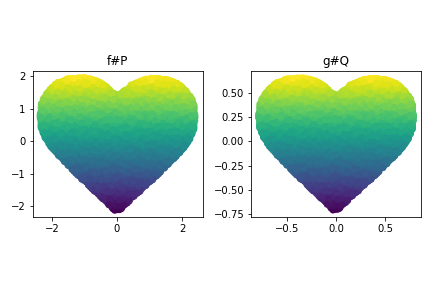}
    \vspace{-2mm}
 \end{subfigure}
 \ \hspace{0.3mm}
 \begin{subfigure}[t]{0.23\textwidth}
 \centering
 \includegraphics[width=1.1\linewidth]{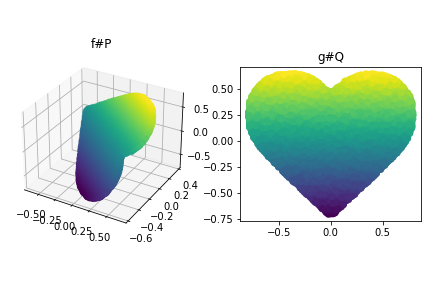}
    \vspace{-2mm}
 \end{subfigure}
 \ \hspace{0.3mm}
 \begin{subfigure}[t]{0.23\textwidth}
 \centering
 \includegraphics[width=1.1\linewidth]{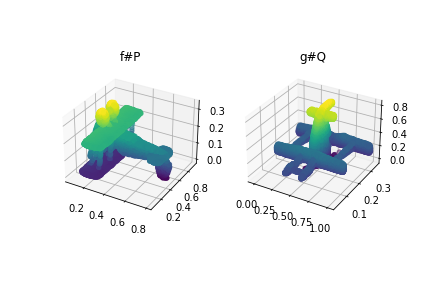}
    \vspace{-2mm}
 \end{subfigure}
 \  
  \caption{GMMD amortization. Each pair shows the image through the learned GMMD mapping. The pairs from left to right: heart $(P)$ vs. rotated/scaled/embedded heart $(Q_b/Q_c/Q_d)$, and biplanes. All 4 cases here are trained with $\lambda=0.064$.}
  \label{fig:gmmd_amortization}
\vspace{-2mm}
\end{figure*}

\subsection{Amortization}\label{app:amortization}
To illustrate the performance of the trained GMMD maps on unseen data, we push 8000 new datapoints  through the learned networks.  Note that GMMD was  trained on only 4000 points.  The output of the pushforward maps on unseen data points during training is  shown in Figure \ref{fig:gmmd_amortization}. We see that GMMD maps successfully generalizes to unseen data.  We also quantitatively demonstrate the amortization in Table \ref{tab:GMMD_amortization_Rotated} to \ref{tab:GMMD_amortization_Biplanes}, where for the same set of parameters $\lambda$ as previously, we push the 8000 new  points through the NNs ( trained with 4000 points ), and compute the resulting marginal MMDs and $\Delta$. As is shown in the tables, the MMDs remain small which means the marginals are well matched.

\begin{table}[htbp]
\caption{GMMD amortization between heart $(P)$ and rotated heart $(Q_b)$. }
\begin{center}
\begin{tabular}{|l|l|l|l|l|}
\hline
\textbf{$\lambda$}  &\textbf{$\GMMD$}&\textbf{$\MMD_\cX$}&\textbf{$\MMD_\cY$}&\textbf{$\Delta$} \\
\hline

0.512 &1.33 &5.43e-02 &1.67e-03
  &2.49\\
\hline
0.256    &  0.0731 &0.0262 &0.0262 &0.0809\\
\hline
0.128    &  0.00790 &0.000276 &0.000215 & 0.0579\\
\hline
0.064     & 0.00711 &0.00199  &0.00152  &0.0564\\
\hline
0.032 &9.34e-02 &1.46e-04 &4.86e-03
  &2.76\\
\hline
0.016  &5.72e-02 &7.75e-04 &1.82e-04
  &3.51\\
\hline
0.008   &   0.0148  &0.00242  &0.00243 &1.24\\
\hline
0.004    &  0.000568  &0.000162   &0.000141   &0.0660\\
\hline
0.002     & 0.00675  &0.00174  &0.00173 &1.64\\
\hline
0.001  &4.91e-03  &1.70e-03 &1.73e-03
  &1.47\\
\hline
\end{tabular}
\end{center}
\label{tab:GMMD_amortization_Rotated}
\end{table}

\begin{table}[htbp]
\caption{GMMD amortization between heart $(P)$ and scaled heart $(Q_c)$.}
\begin{center}
\begin{tabular}{|l|l|l|l|l|}
\hline
\textbf{$\lambda$}  &\textbf{$\GMMD$}&\textbf{$\MMD_\cX$}&\textbf{$\MMD_\cY$}&\textbf{$\Delta$} \\
\hline
0.512   &   0.0347 &0.000555  &0.000856 &0.0650\\
\hline
0.256    &  0.0133& 0.000599 &0.000787 &0.0466\\
\hline
0.128     & 0.0708  &0.0261 &0.0268 &0.139\\
\hline
0.064     & 0.00488 & 0.000577 &0.000716 &0.0561\\
\hline
0.032 &9.68e-02 &3.00e-04 &5.05e-03
  &2.86\\
\hline
0.016    &  0.00153   &0.000366 &0.000361  &0.0501\\
\hline
0.008  &2.83e-02 &1.60e-03  &3.23e-03
  &2.94\\
\hline
0.004     & 0.00952  &0.00210  &0.00196 &1.37\\
\hline
0.002  &1.19e-02 &1.76e-03 &2.01e-03
  &4.05\\
\hline
0.001  &4.89e-03 &1.75e-04 &1.59e-03
  &3.12\\
\hline
\end{tabular}
\end{center}
\label{tab:GMMD_amortization_Scaled}
\end{table}

\begin{table}[htbp]
\caption{GMMD amortization between heart $(P)$ and embedded heart $(Q_d)$.}
\begin{center}
\begin{tabular}{|l|l|l|l|l|}
\hline
\textbf{$\lambda$}  &\textbf{$\GMMD$}&\textbf{$\MMD_\cX$}&\textbf{$\MMD_\cY$}&\textbf{$\Delta$} \\
\hline
0.512      &0.0518  &0.000721   &0.000717   &0.0985\\
\hline
0.256      &0.0657  &0.0272  &0.0271  &0.0446\\
\hline
0.128      &0.0133  &0.00126  &0.000716  &0.0885\\
\hline
0.064      &0.00431  &0.000566 &0.000516  &0.0504\\
\hline
0.032  &9.38e-02 &1.52e-03 &3.98e-03
  &2.76\\
\hline
0.016      &0.0198  &0.00199   &0.00194  &0.990\\
\hline
0.008     &3.52e-02 &1.62e-03 &2.77e-03
  &3.85\\
\hline
0.004 &1.96e-02 &1.92e-03 &1.93e-03
  &3.95\\
\hline
0.002     & 0.00287  &0.00131   &0.00122   &0.173\\
\hline
0.001 & 6.89e-03 &1.48e-03 &1.34e-03  &4.07 \\
\hline
\end{tabular}
\end{center}
\label{tab:GMMD_amortization_Embedded}
\end{table}

\begin{table}[htbp]
\caption{GMMD amortization between biplanes.}
\begin{center}
\begin{tabular}{|l|l|l|l|l|}
\hline
\textbf{$\lambda$}  &\textbf{$\GMMD$}&\textbf{$\MMD_\cX$}&\textbf{$\MMD_\cY$}&\textbf{$\Delta$} \\
\hline
0.512     & 0.316 &0.123  &0.124  &0.134\\
\hline
0.256     & 0.283  &0.124  &0.124 &0.137\\
\hline
0.128     & 0.0274  &0.00664 &0.00660 &0.111\\
\hline
0.064    & 0.0165  &0.00596 &0.00636 &0.0658\\
\hline
0.032     & 0.127  & 0.00852  &0.00953   &3.40\\
\hline
0.016     & 0.0291    & 0.00995
&0.00953 &0.600\\
\hline
0.008     & 0.0367  &0.00666 &0.00627  &2.97\\
\hline
0.004   & 2.33e-02 &2.83e-03 &4.58e-03   &3.98\\
\hline
0.002  & 1.13e-02 &3.31e-03 &2.30e-03  &2.83\\
\hline
0.001 &1.21e-02 &3.96e-03 &4.51e-03   & 3.60 \\
\hline
\end{tabular}
\end{center}
\label{tab:GMMD_amortization_Biplanes}
\end{table}

\end{document}